\documentclass{article}

\usepackage[final]{corl_2021} 
\usepackage{times}
\usepackage{algorithmic}
\usepackage{algorithm}
\usepackage{amsmath}
\usepackage{amssymb}
\usepackage{graphics}
\usepackage{epsfig}
\usepackage[caption=false, font=footnotesize]{subfig}
\usepackage{multirow}
\usepackage{wrapfig}
\usepackage{amsthm}
\usepackage{mathtools}
\usepackage{makecell}
\usepackage{bbm}
\usepackage{booktabs}
\usepackage{graphicx}

\usepackage{amsmath}
\usepackage{amssymb}
\usepackage{amsthm}

\newtheorem{proposition}{Proposition}[section]

\newtheorem{lemma}{Lemma}[section]
\newtheorem{assumption}{Assumption}[section]

\newtheorem{innercustomprop}{Proposition}
\newenvironment{customprop}[1]
  {\renewcommand\theinnercustomprop{#1}\innercustomprop}
  {\endinnercustomprop}

\newcommand{\SE}{\mathrm{SE}}

\newcommand{\edit}[1]{#1}    

\title{Equivariant $Q$ Learning in Spatial Action Spaces}

\author{
    Dian Wang \quad \quad
    Robin Walters \quad \quad
    Xupeng Zhu \quad \quad
    Robert Platt\\
    Khoury College of Computer Sciences\\
    Northeastern University\\
    United States\\
    \texttt{\{wang.dian, r.walters, zhu.xup, r.platt\}@northeastern.edu}
}

\begin{document}
\maketitle



\begin{abstract}
Recently, a variety of new equivariant neural network model architectures have been proposed that generalize better over rotational and reflectional symmetries than standard models. These models are relevant to robotics because many robotics problems can be expressed \edit{in a rotationally symmetric way.} This paper focuses on equivariance over a visual state space and a spatial action space -- the setting where the robot action space includes a subset of $\SE(2)$. In this situation, we know a priori that rotations and translations in the state image should result in the same rotations and translations in the spatial action dimensions of the optimal policy. \edit{Therefore, we can use equivariant model architectures to make $Q$ learning more sample efficient. This paper identifies when the optimal $Q$ function is equivariant and proposes $Q$ network architectures for this setting. We show experimentally that this approach outperforms standard methods in a set of challenging manipulation problems. }
\end{abstract}


\keywords{Reinforcement Learning, Equivariance, Manipulation} 


\section{Introduction}

A key question in policy learning for robotics is how to leverage structure present in the robot and the world to improve learning. This paper focuses on a fundamental type of structure present in visuo-motor policy learning for most robotics problems: translational and rotational invariance with respect to camera viewpoint. Specifically, the reward and transition dynamics of most robotics problems can be expressed in a way that is invariant with respect to the camera viewpoint from which the agent observes the scene. 
In spite of the above, most visuo-motor policy learning agents do not leverage this invariance in camera viewpoint. The agent's value function or policy typically considers different perspectives on the same scene to be different world states. A popular way to combat this problem is through visual data augmentation, i.e., to create additional samples or experiences by randomly translating and rotating observed images~\cite{transporter} but keeping the same labels. This can be used in conjunction with a contrastive term in the loss function which helps the system learn an invariant latent representation~\cite{oord2018representation,curl}. While these methods can improve generalization, 
\edit{they require the neural network to learn translational and rotational invariance from the augmented data.}

Our key idea in this paper is to model rotational and translation invariance in policy learning using neural network model architectures that are equivariant over finite subgroups of $\SE(2)$. These equivariant model architectures reduce the number of free parameters using steerable convolutional layers~\cite{steerable_cnns}. 
\edit{Compared with traditional methods, this approach creates an inductive bias that can significantly improve the sample efficiency of the model, the number of environmental steps needed to learn a policy.}
Moreover, it enables us to generalize in a very precise way: everything learned with respect to one camera viewpoint is automatically also represented in other camera perspectives via selectively tied parameters in the model architecture. We focus our work on $Q$ learning in spatial action spaces, where the agent's action space spans $\SE(2)$ or $\SE(3)$. We make the following contributions. First, \edit{we identify the conditions under which the optimal $Q$ function is $\SE(2)$ equivariant.} Second, we propose neural network model architectures that encode $\SE(2)$ equivariance in the $Q$ function. Third, since most policy learning problems are only equivariant in \emph{some} of the state variables, we propose partially equivariant model architectures that can accommodate this. Finally, we compare equivariant models against non-equivariant counterparts in the context of several robotic manipulation problems. The results show that equivariant models are more sample efficient than non-equivariant models, often by a significant margin. Supplementary video and code are available at \url{https://pointw.github.io/equi_q_page}.

\section{Related Work}
\underline{Data Augmentation:} Data augmentation techniques have long been employed in computer vision to encode the invariance property of translation and reflection into neural networks~\cite{lenet, alexnet}. Recent work demonstrates the use of data augmentation improves the data efficiency and the policy's performance in reinforcement learning~\cite{rl_with_aug, kostrikov2020image, hansen2020generalization}. In the context of robotics, data augmentation is often used to generate additional samples~\cite{transporter, qtopt, lin2020invariant}.
In contrast to learning the equivariance property using data augmentation, our work utilizes the equivariant network to hard code the symmetries in the structure of the network to achieve better sample efficiency.

\underline{Contrastive Learning:} Another approach to learning a representation that is invariant to translation and rotation is to add a contrastive learning term to the loss function~\cite{oord2018representation}. This idea has been applied to reinforcement learning in general~\cite{curl} and robotic manipulation in particular~\cite{zhan2020framework}. While this approach can help the agent learn an invariant encoding of the data, it does not necessarily improve the sample efficiency of policy learning. 




\underline{Equivariant Learning:} Equivariant model architectures hard-code $\mathrm{E}(2)$ symmetries into the structure of the neural network and have been shown to be useful in computer vision~\cite{cohen2016group, steerable_cnns, e2cnn}. In reinforcement learning, some recent work applies equivariant models to structure-finding problems involving MDP homomorphisms~\cite{van2020plannable, van2020mdp}. In addition, \citet{mondal2020group} recently applied an $\rm{E}(2)$-equivariant model to $Q$ learning in an Atari game domain, but showed limited improvement. To our knowledge, equivariant model architectures have not been explored in the context of robotics applications.



\underline{Spatial Action Representations:} Several researchers have applied policy learning in spatial action spaces to robotic manipulation. A popular approach is to do $Q$ learning with a dense pixel action space using a fully convolutional neural network (this is the FCN approach we describe and extend in Section~\ref{sect:fcn})~\cite{morrison2018closing, fc_gq_cnn, zeng_picking, zeng_pushing}. Variations on this approach have been explored in~\cite{deictic,marcus_hier}. The FCN approach has been adapted to a variety of different manipulation tasks with different action primitives~\cite{ShiftingObjectsforGrasping, Slide-to-Wall, tossingbot, form2fit, berscheid2020self, transporter, asrse3, biza2021action,wu2020spatial}. In this paper, we extend the work above by proposing new equivariant architectures for the spatial action space setting.


\section{Problem Statement}
\label{sect:problem}

We are interested in solving complex robotic manipulation problems such as the packing and construction problems shown in Fig~\ref{fig:envs}. We focus on problems expressed in a spatial action space. \edit{This section identifies conditions under which the $Q$ function is $\SE(2)$-invariant. The next section describes how these invariance properties translate into equivariance properties in the neural network.}


\underline{Manipulation as an MDP in over a visual state space and a spatial action space:} We assume that the manipulation problem is formulated as a Markov decision process (MDP): $\mathcal{M} = (S,A,T,R,\gamma)$. We focus on MDPs in visual state spaces and spatial action spaces~\cite{asrse3,zeng_picking,wu2020spatial}. The state space is factored into the state of the objects in the world, expressed as an $n$-channel $h \times w$ image $I \in S_{\rm{world}} = \mathbb{R}^{n \times h \times w}$, and the state of the robot (including objects held by the robot) $s_{\rm{rbt}} \in S_{\rm{rbt}}$, expressed arbitrarily. The total state space is $S = S_{\rm{world}} \times S_{\rm{rbt}}$. The action space is expressed as a cross product of $\SE(2)$ (hence it is spatial) and a set of additional arbitrary action variables: $A = \SE(2) \times A_{\rm{arb}}$. The spatial component of action expresses where the robot hand is to move and the additional action variables express how it should move or what it should do. For example, in the pick/place domains shown in Fig~\ref{fig:envs}, $A_{\rm{arb}} = \{\textsc{pick}, \textsc{place}\}$, giving the agent the ability to move to a pose and close the fingers (pick) or move and open the fingers (place). We will sometimes decompose the spatial component of action $a_{\rm{sp}} \in \SE(2)$ into its translation and rotation components, $a_{\rm{sp}} = (x,\theta)$. The goal of manipulation is to achieve a desired configuration of objects in the world, as expressed by a reward function $R : S \times A \rightarrow \mathbb{R}$.

\underline{Translation and Rotation in $\SE(2)$:} \edit{We are interested in learning policies that are invariant to translation and rotation of the state and action.} To do that, we define rotation and translation of state and action as follows. Let $g \in \SE(2)$ be an arbitrary rotation and translation in the plane and let $s = (I,s_{\rm{rbt}}) \in S_{\rm{world}} \times S_{\rm{rbt}}$ be a state. $g$ operates on $s$ by rotating and translating the image $I$, but leaving $s_{\rm{rbt}}$ unchanged: $gs = (gI,s_{\rm{rbt}})$, where $gI$ denotes the image $I$ translated and rotated by $g$. For action $a = (a_{\rm{sp}}, a_{\rm{arb}})$, $g$ rotates and translates $a_{\rm{sp}}$ but not $a_{\rm{arb}}$: $ga = (g a_{\rm{sp}},a_{\rm{arb}})$. Notice that both $S$ and $A$ are closed under $g \in \SE(2)$, i.e. that $\forall g \in \SE(2)$, $a \in A \implies ga \in A$ and $s \in S \implies gs \in S$.


\underline{Assumptions:} We assume that the reward and transition dynamics of the system are invariant with respect to translation and rotation of state and action as defined above, and that the translation and rotation operations on state and action are invertible.

\begin{assumption}[Goal Invariance]
\label{assumption:goalinv}
The manipulation objective is to achieve a desired configuration of objects in the world without regard to the position and orientation of the scene. That is, $R(s,a) = R(gs,ga)$ for all $g \in \SE(2)$.
\end{assumption}

\begin{assumption}[Transition Invariance]
\label{assumption:transinv}
The outcome of robot actions is invariant to translations and rotations of both the scene and the action. Specifically, $T(s,a,s') = T(gs,ga,gs')$ for all $g \in \SE(2)$.
\end{assumption}


\begin{assumption}[Invertibility]
\label{assumption:invertability}
Translations and rotations in state and action are invertible. That is, $\forall g \in \SE(2)$, $g^{-1}(gs) = s$ and $g^{-1}(ga) = a$.
\end{assumption}

Assumptions~\ref{assumption:goalinv} and~\ref{assumption:transinv} are satisfied in problem settings where the objective and the transition dynamics can be expressed intrinsically to the world without reference to an external coordinate frame imposed by the system designer. These assumptions are satisfied in many manipulation domains including all those shown in Fig~\ref{fig:envs}. In House Building, for example, the reward and transition dynamics of the system are independent of the coordinate frame of the image or the action space. Assumption~\ref{assumption:invertability} is needed to guarantee the $Q$ function invariance described in the next section.

\begin{figure}[t]
\newlength{\env}
\newlength{\envcovid}
\setlength{\env}{0.135\linewidth}
\setlength{\envcovid}{0.09\linewidth}
\centering
\subfloat[Block Stacking]{
\label{fig:envs_4s}
\includegraphics[width=\env]{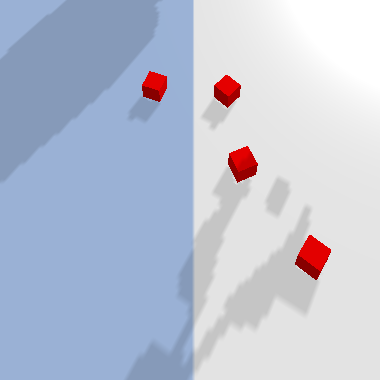}
\includegraphics[width=\env]{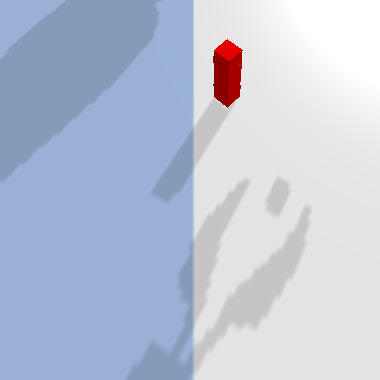}
}
\subfloat[Bottle Arrangement]{
\label{fig:envs_bt}
\includegraphics[width=\env]{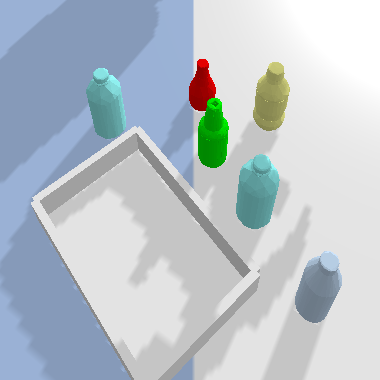}
\includegraphics[width=\env]{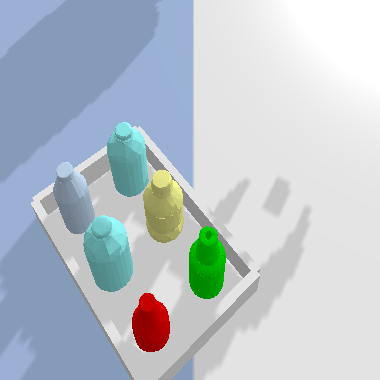}
}
\subfloat[House Building]{
\label{fig:envs_h4}
\includegraphics[width=\env]{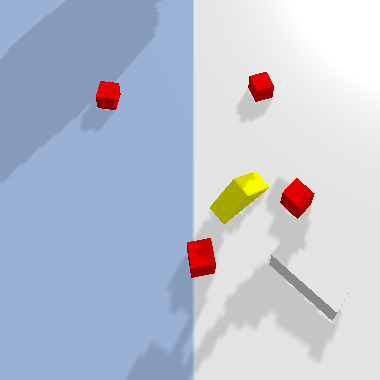}
\includegraphics[width=\env]{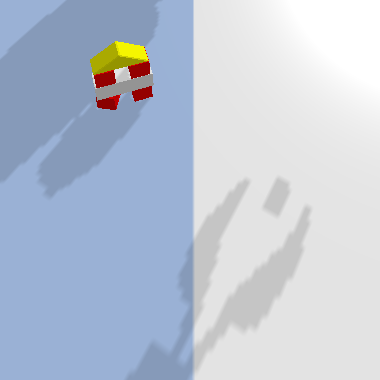}
}\\
\vspace{-0.3cm}
\subfloat[Covid Test]{
\label{fig:envs_covid}
\includegraphics[width=\envcovid]{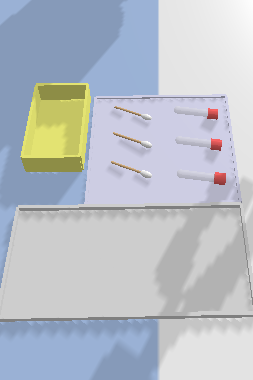}
\includegraphics[width=\envcovid]{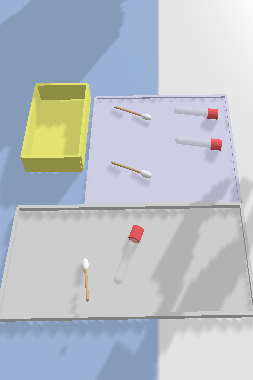}
\includegraphics[width=\envcovid]{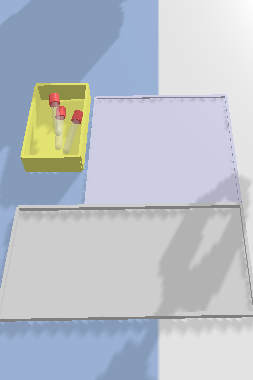}
}
\subfloat[Box Palletizing]{
\label{fig:envs_box18}
\includegraphics[width=\env]{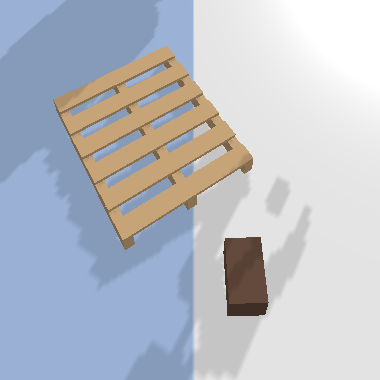}
\includegraphics[width=\env]{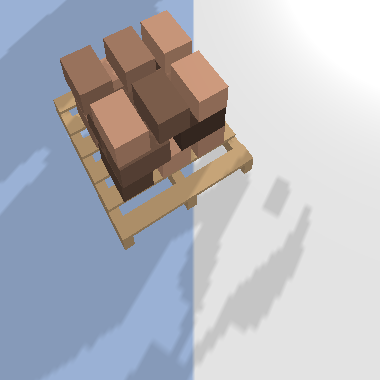}
}
\subfloat[Bin Packing]{
\label{fig:envs_bin}
\includegraphics[width=\env]{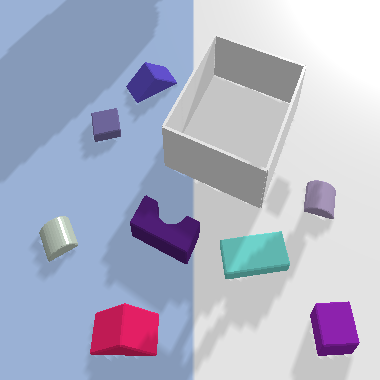}
\includegraphics[width=\env]{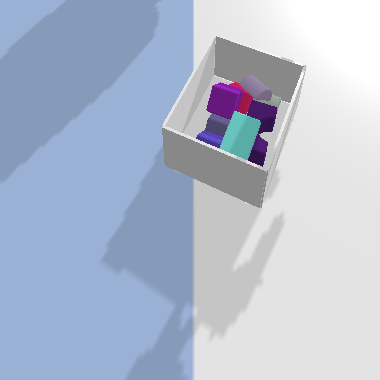}
}
\vspace{-0.2cm}
\caption{The experimental environments implemented in PyBullet~\cite{pybullet}. The left image in each sub figure shows an initial state of the environment; the right image shows the goal state.}
\vspace{-0.2cm}
\label{fig:envs}
\end{figure}



\section{Approach}


Assumptions~\ref{assumption:goalinv}, \ref{assumption:transinv}, and \ref{assumption:invertability} imply that the optimal $Q$ function is invariant to translations and rotations in $\SE(2)$.

\begin{proposition}
\label{prop:qstar}
Given an MDP $\mathcal{M} = (S,A,T,R,\gamma)$ for which Assumptions~\ref{assumption:goalinv}, \ref{assumption:transinv}, and \ref{assumption:invertability} are satisfied, the optimal $Q$ function is invariant to translation and rotation, i.e. $Q^*(s,a) = Q^*(gs,ga)$, for all $g \in \SE(2)$. (Proof in Appendix~\ref{appendix:proof}.)
\end{proposition}

Our key idea is to use the invariance property of Proposition~\ref{prop:qstar} to structure $Q$ learning (and make it more sample efficient) by defining a neural network that is hard-wired to encode only invariant $Q$ functions. However, in order to accomplish this in the context of DQN, we must allow for the fact that state is an input to the neural network while action values are an \emph{output}. This neural network is therefore a function $q : S \rightarrow \mathbb{R}^{A}$, where $\mathbb{R}^{A}$ denotes the space of functions $\lbrace A \to \mathbb{R} \rbrace$. The invariance property of Proposition~\ref{prop:qstar} now becomes an \emph{equivariance} property, 
\begin{equation}
q(gs)(a) = q(s)(g^{-1}a) ,
\label{eqn:eqiv_q_fn}
\end{equation}
where $q(s)(a)$ denotes the $Q$ value of action $a$ in state $s$. We implement this constraint using equivariant convolutional layers as described below.

\subsection{Equivariant Convolutions}
\label{sec:equi_conv}
\underline{Equivariance over a finite group:} In order to implement the equivariance constraint, it is standard in the literature to approximate $\SE(2)$ by a finite subgroup~\cite{steerable_cnns, e2cnn}. Recall that the spatial component of an action is $a_{\rm{sp}} = (x,\theta) \in \SE(2)$. We constrain position to be a discrete pair of positive integers $x \in \{1 \dots h\} \times \{1 \dots w\} \subset \mathbb{Z}^2$, corresponding to a pixel in the input image $I$. We constrain orientation to be a member of a finite cyclic group $\theta \in C_u$, i.e. one of $u$ discrete orientations. For example, if $u=8$, then $C_8 = \{0,\frac{\pi}{4}, \frac{2\pi}{4}, \frac{3\pi}{4}, \frac{4\pi}{4}, \frac{5\pi}{4}, \frac{6\pi}{4}, \frac{7\pi}{4}\}$. Our finite approximation of $a_{\rm{sp}} \in \SE(2)$ is $\hat{a}_{\rm{sp}}$ in the subgroup $\hat{\SE}(2)$ generated by translations $\mathbb{Z}^2$ and rotations $C_u$. 

\begin{wrapfigure}[18]{r}{0.4\textwidth}
\vspace{-0.5cm}
  \begin{center}
  \includegraphics[width=0.4\textwidth]{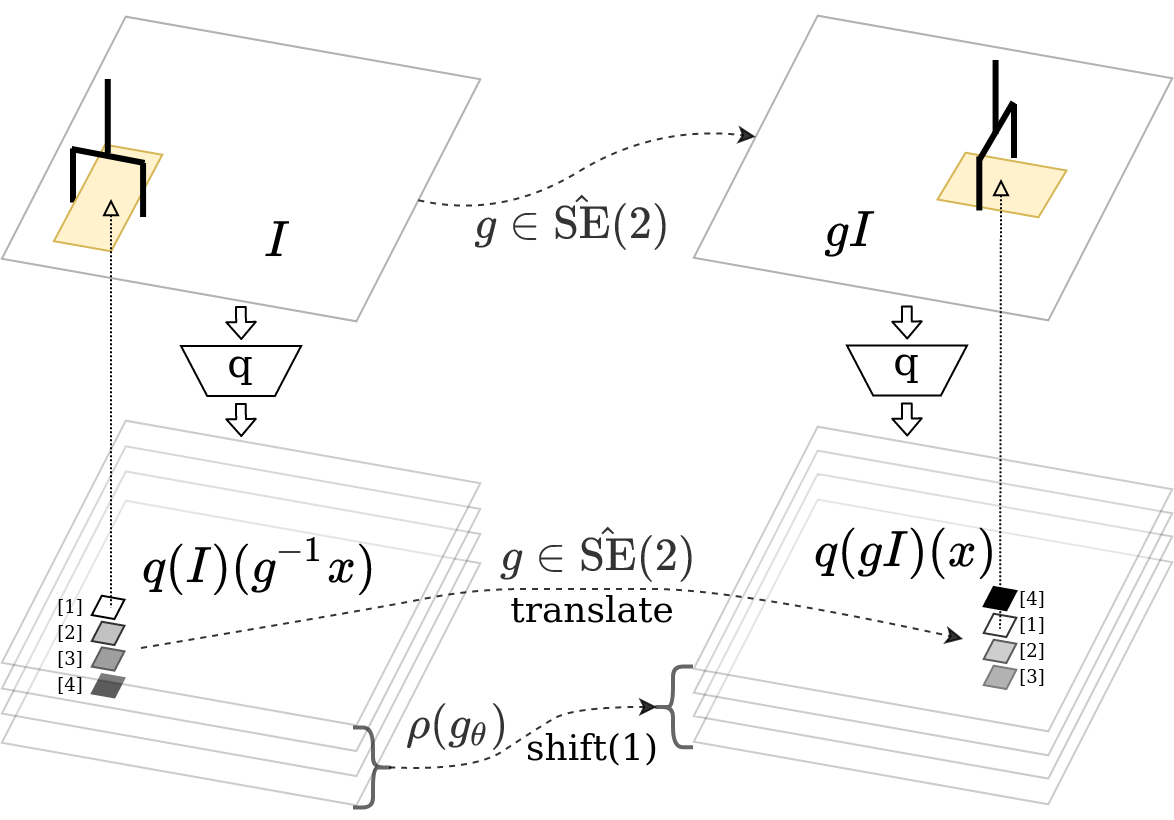}
  \end{center}
\vspace{-0.3cm}
  \caption{Illustration of $Q$-map equivariance when $C_u=C_4$. The output $Q$-map rotates and translates with the input image. The $4$-vector at each pixel does a circular shift, i.e., the optimal rotation changes from 0 (the 1st element of $C_4$) to $\frac{\pi}{2}$ (the 2nd element of $C_4$)}
\label{fig:q_theta_equivariant_property}
\end{wrapfigure}

\underline{Input and output of an equivariant convolutional layer:} A standard convolutional layer $h$ takes as input an $n$-channel feature map and produces an $m$-channel map as output, $h_{\mathrm{standard}} : \mathbb{R}^{n \times h \times w} \to \mathbb{R}^{m \times h \times w}$. We can construct an equivariant convolutional layer by adding an additional dimension to the feature map that encodes the values for each element of a group ($C_u$ in our case).\footnote{In the language of \cite{e2cnn}, this is a steerable convolution between regular representations of $C_u$.} The equivariant mapping therefore becomes $h_{\mathrm{equiv}} : \mathbb{R}^{u \times n \times h \times w} \to \mathbb{R}^{u \times m \times h \times w}$ for all layers except the first. The first layer of the network generally takes a ``flat'' image as input: $h_{\mathrm{equiv}}^{\mathrm{in}} : \mathbb{R}^{1 \times n \times h \times w} \to \mathbb{R}^{u \times m \times h \times w}$.\footnote{This is a steerable convolution between the trivial representation and regular representation of $C_u$.}


\underline{Equivariance constraint:} Let $h_i(I)(x)$ denote the output of convolutional layer $h$ at channel $i$ and pixel $x$ given input $I$. For an equivariant layer, $h_i(I)(x) \in \mathbb{R}^{C_u}$ describes feature values for each element of $C_u$.  
For an element $g \in \hat{\SE}(2)$, denote the 
rotational part by  
$g_\theta \in C_u$.
If we identify functions $\mathbb{R}^{C_u}$ with vectors $\mathbb{R}^u$, then the group action of $g_\theta \in C_u$ on $\mathbb{R}^{C_u}$ becomes left multiplication by a permutation matrix $\rho(g_\theta)$ that performs a circular shift on the vector in $\mathbb{R}^u$.
Then the group action of $\hat{\SE}(2)$ on a feature map $h(I) \in \mathbb{R}^{u \times m \times h \times w}$ can be expressed as $g(h_i(I))(x) = \rho(g_\theta) h_i(I)(g^{-1}x)$.
The mapping $h$ is equivariant if and only if
\begin{equation}
h_i(gI)(x) = g(h_i(I))(x) = \rho(g_\theta) h_i(I) (g^{-1}x),
\label{eqn:layer_equiv_constraint}
\end{equation}
for each $i \in \{1 \dots m\}$. 
This is illustrated in Fig~\ref{fig:q_theta_equivariant_property}. We can calculate the output feature map in the lower right corner by transforming the input by $g$ and then doing the convolution (left side of Eq.~\ref{eqn:layer_equiv_constraint}) or by doing the convolution first and then taking the value of $g^{-1}x$ and circular-shifting the output vector (right side of Eq.~\ref{eqn:layer_equiv_constraint}). In order to create a network that enforces the constraint of Eq.~\ref{eqn:eqiv_q_fn}, we can simply stack equivariant convolutions layers that each satisfy Eq.~\ref{eqn:layer_equiv_constraint}.

\underline{Kernel constraint:} The equivariance constraint of Eq.~\ref{eqn:layer_equiv_constraint} can be implemented by strategically tying weights together in the convolutional kernel~\cite{steerable_cnns}. Since the standard convolutional kernel is already translation equivariant~\cite{cohen2016group}, we must only enforce rotational ($C_u$) equivariance~\cite{cohen_equicnn_theory}:
\begin{equation}
\label{eqn:kernel_constraint}
K(g_\theta y) = \rho_{\rm{out}}(g_\theta) K(y) \rho_{\rm{in}}(g_\theta)^{-1},
\end{equation}
where $\rho_{\rm{in}}(g_\theta)$ and $\rho_{\rm{out}}(g_\theta)$ are the permutation matrix of the group element $g_\theta$ (note that for the first layer, $K(y)$ will be a $1\times u$ matrix, and $\rho_{\rm{in}}(g_\theta)$ will be 1). More details are in Appendix~\ref{appendix:equi_kernel_constraint}.

\subsection{Equivariant Fully Convolutional $Q$ Functions in $\SE(2)$}
\label{sect:fcn}

A baseline approach to encoding the $Q$ function over a spatial action space is to use a fully convolutional network (FCN) that stacks convolutional layers to produce an output $Q$ map with the same resolution as the input image. If we ignore the non-image state variables $s_{\rm{rbt}}$ and the non-spatial action variables $a_{\rm{arb}}$, then we have all the tools we need -- we simply replace all convolutional layers with equivariant convolutions and the $Q$ network becomes fully equivariant.



\underline{Partial Equivariance:} Unfortunately, in realistic robotics problems, $Q$ function is generally not equivariant with respect to all state and action variables. For example, the non-equivariant parts of state 
\begin{wrapfigure}[10]{r}{0.63\textwidth}
\vspace{-0.2cm}
\centering
\subfloat[Lift Expansion]{
\includegraphics[height=2cm]{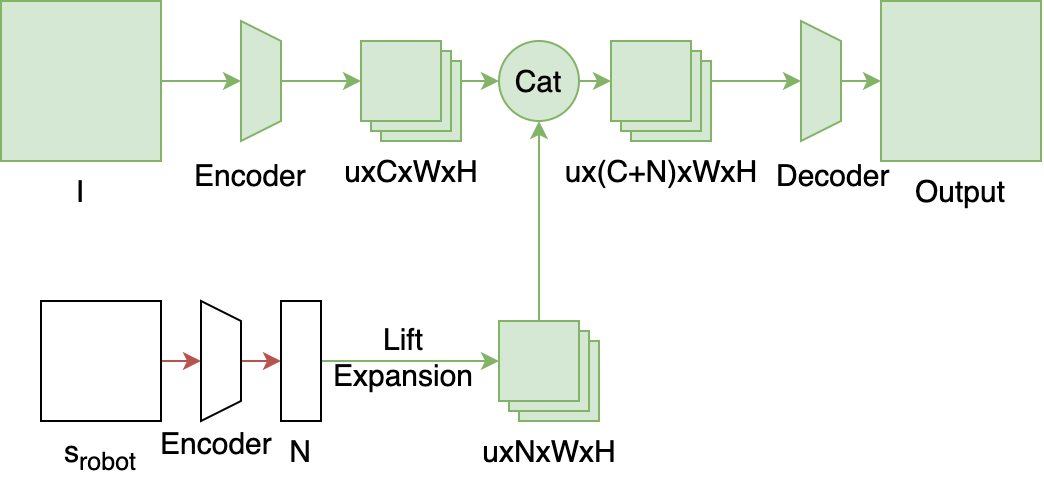}
\label{fig:partial_equi_exp}
}
\subfloat[Dynamic Filter]{
\includegraphics[height=2cm]{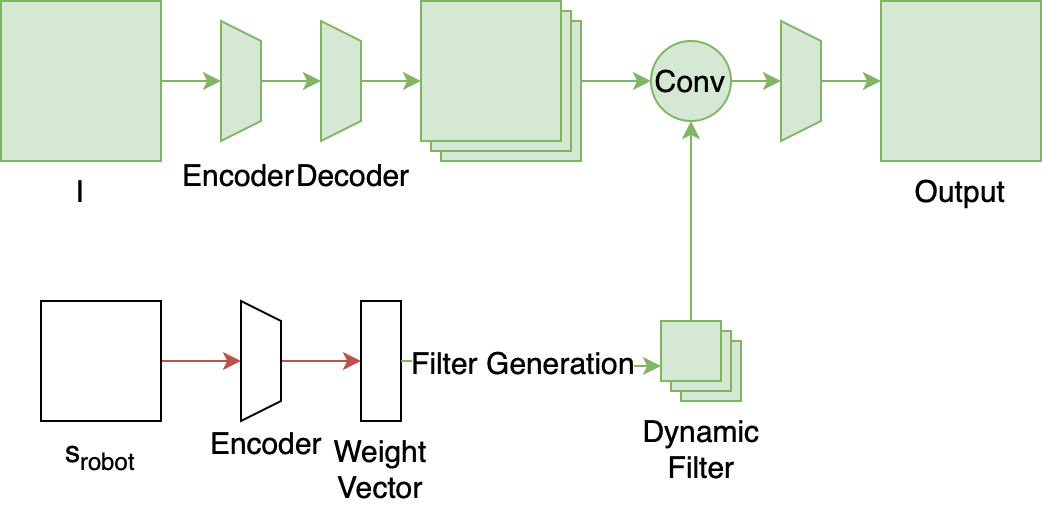}
\label{fig:partial_equi_df}
}
\caption{Illustration of approaches to partial equivariance. Green blocks are equivariant and white blocks are not.}
\label{fig:partial_equi}
\end{wrapfigure}
and action in Section~\ref{sect:problem} are $s_{\rm{rbt}}$ and $a_{\rm{arb}}$. We encode $a_{\rm{arb}}$ by simply having a separate output head for each. However, to encode $s_{\rm{rbt}}$, we need a mechanism for inserting the non-equivariant information into the neural network model without ``breaking'' the equivariance property. We explored two approaches: the lift expansion approach and the dynamic filter approach. In the lift expansion, we tile the non-equivariant information across the equivariant dimensions of the feature map as additional channels (Fig~\ref{fig:partial_equi_exp}). In the dynamic filter approach~\cite{dynamic_filter_network}, the non-equivariant data is passed through a separate pathway that outputs the weights of an equivariant kernel that is convolved into the main equivariant backbone. We constrain this filter to be equivariant by enforcing the kernel constraint of Eq.~\ref{eqn:kernel_constraint} (Fig~\ref{fig:partial_equi_df}). We empirically find that both methods have similar performance (Appendix~\ref{appendix:exp_df_vs_exp}). In the remainder of this paper, we use the dynamic filter approach because it is more memory efficient.


\underline{Encoding Gripper Symmetry Using Quotient Groups:} Another symmetry that we want to leverage is the bilateral symmetry of the gripper. The outcome of a pick action performed using a two-finger gripper in orientation $\theta$ is the same as for the gripper in orientation $\theta + k \pi$ for any integer $k$. Similarly, it is often valid to assume that the outcome of place actions is invariant~\footnote{Strictly speaking, this is true only when the grasped object is also symmetric.}. We model this invariance using the quotient group $C_u/C_2$. The $C_2 = \lbrace 0, \pi \rbrace$ action equates rotations which differ by multiples of $\pi$ in $C_u/C_2$. The steerable layer defined under the quotient group is applied with the same constraint as in Eq.~\ref{eqn:kernel_constraint}, except that the output space will be in $C_u/C_2$.

\label{sec:exp_equi_fcn}
\underline{Experimental Domains:} We evaluate the equivariant FCN approach in the Block Stacking and Bottle Arrangement tasks shown in  Fig~\ref{fig:envs}. 
Both environments have sparse rewards (+1 at goal and 0 otherwise). The world state is encoded by a 1-channel heightmap $I \in \mathbb{R}^{1 \times h \times w}$ and robot state is encoded by an image patch $H$ that describes the contents of the robotic hand. The non-spatial action variable $a_{\rm{arb}} \in \{\textsc{pick},\textsc{place}\}$ is selected by the gripper state, i.e., $a_{\rm{arb}}=\textsc{place}$ if the gripper is holding an object, and $\textsc{pick}$ otherwise. The equivariant layers of the FCN are defined over group $C_{12}$ where the output is with respect to the quotient group $C_{12}/C_2$ to encode the gripper symmetry. See Appendix~\ref{appendix:envs} and~\ref{appendix:network} for detail on the experimental domains and the FCN architecture respectively. 


\begin{wrapfigure}[12]{r}{0.46\textwidth}
\vspace{-0.8cm}
\begin{center}
\subfloat[Block Stacking]{\includegraphics[width=0.23\textwidth]{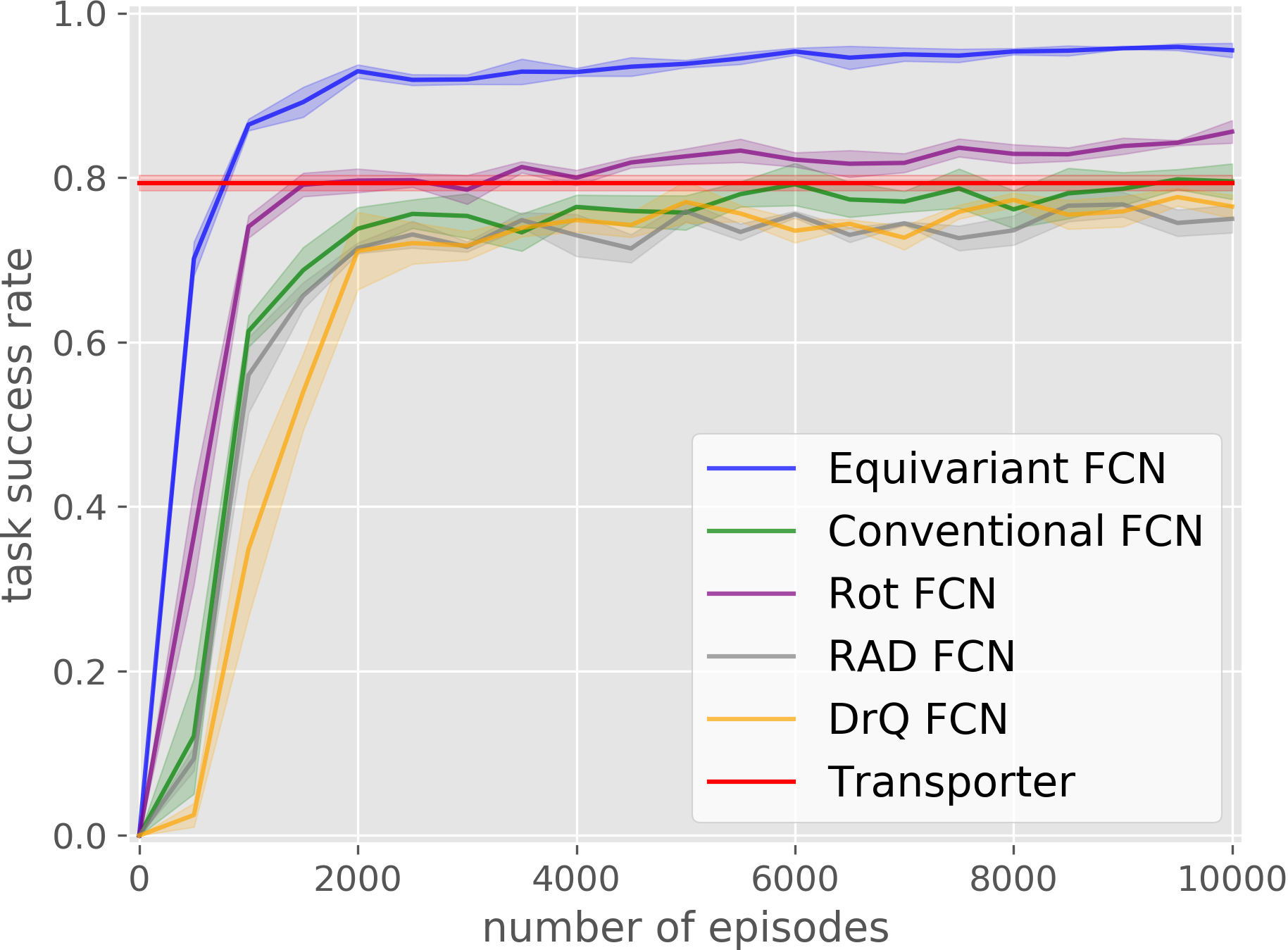}}
\subfloat[Bottle Arrangement]{\includegraphics[width=0.23\textwidth]{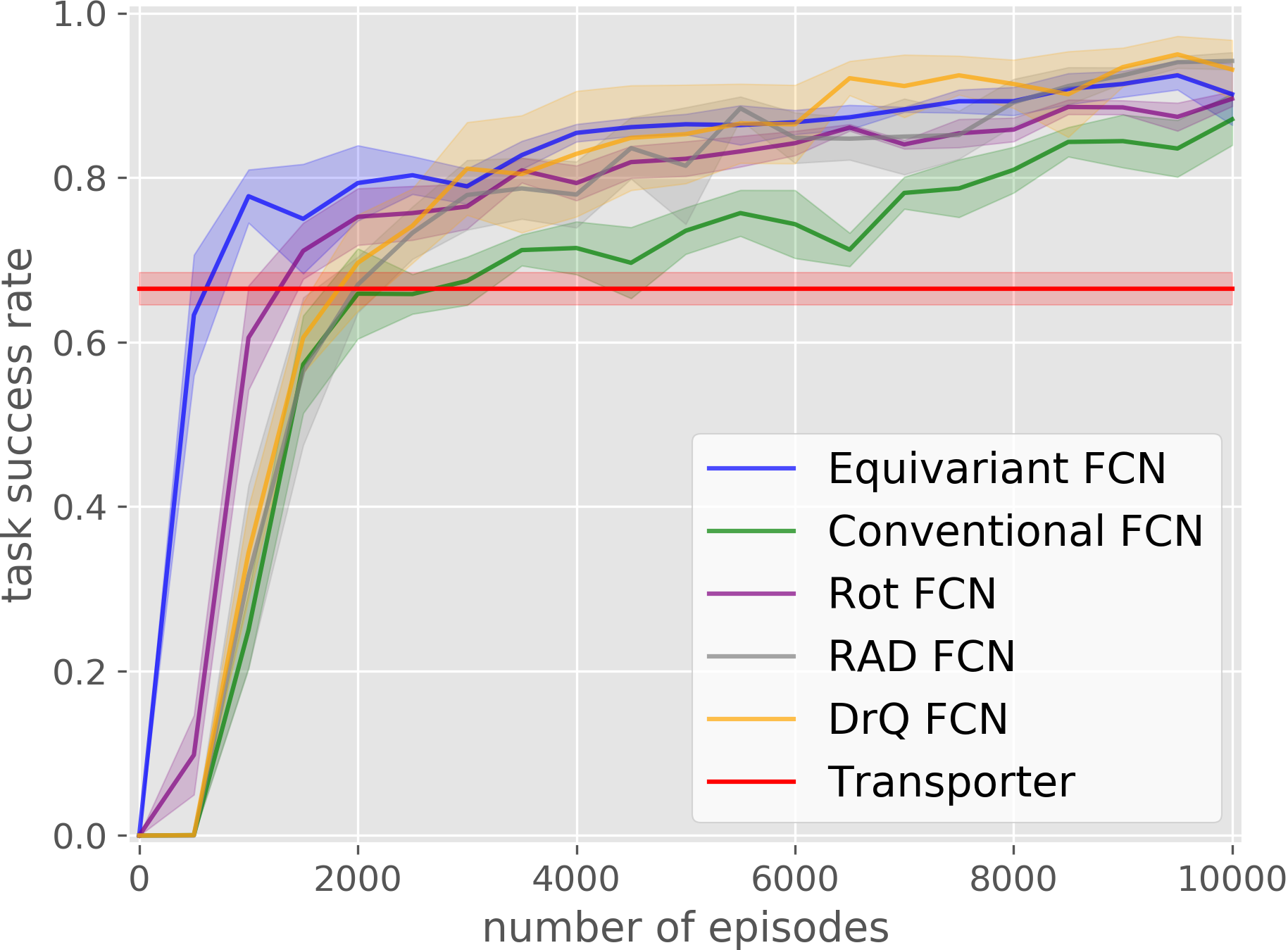}}
\end{center}
\vspace{-0.25cm}
\caption{\edit{Comparison of Equivariant FCN (blue) with baselines. Results averaged over four runs. Shading denotes standard error.}}
\label{fig:exp_fcn_equi_vs_cnn}
\end{wrapfigure}

\underline{Experimental Comparison With Baselines:} We evaluate against the following baselines: 1) Conventional FCN: FCN with 1-channel input and $6$-channel output where each output channel corresponds to a $Q$ map for one rotation in the action space (similar to~\citet{fc_gq_cnn} but without the $z$ dimension). \edit{2) RAD~\cite{rl_with_aug} FCN: same architecture as 1), while at each training step, we augment each transition in the minibatch with a rotation randomly sampled from $C_{12}$. 3) DrQ~\cite{kostrikov2020image} FCN: same architecture as 1), while at each training step, the $Q$ targets and $Q$ outputs are calculated by averaging over multiple augmented versions of the sampled transitions. Random rotations sampled from $C_{12}$ are used for the augmentation. }4) Rot FCN: FCN with 1-channel input and 1-channel output, the rotation is encoded by rotating the input and output for each $\theta$~\cite{zeng_pushing}. 5) Transporter Network~\cite{transporter}, an FCN-based architecture with the last layer being a dynamic kernel generated by a separate FCN with an input of an image crop at the pick location. 
\edit{Baseline 2) and 3) are data augmentation methods that aim to learn the symmetry encoded in our equivariant network using rotational data augmentation sampled from the same symmetry group ($C_{12}$) as used by our equivariant model.
All baselines have the same action space as the proposal.} 
More detail on the baselines is in Appendix~\ref{appendix:baseline_fcn}. All methods except the Transporter Network use SDQfD, an approach to imitation learning in spatial action spaces that combines a TD loss term with penalties on non-expert actions~\cite{asrse3}. (Transporter Network is a behavior cloning method.) 
Table~\ref{tab:expert_step} shows the number of demonstration steps.
Those expert transitions are augmented by 9 random $\SE(2)$ transformations. See Appendix~\ref{appendix:parameters} for more parameter detail. Fig~\ref{fig:exp_fcn_equi_vs_cnn} shows the results. Our equivariant FCN outperforms all baselines in the block stacking task. Notice that in the Bottle Arrangement task, the equivariant network learns faster than the baselines but converges to a similar level as RAD and DrQ. This is because the domain itself is already partially rotationally equivariant because the bottles are cylindrical and therefore our network has less of an advantage.

\begin{table}[t]
\scriptsize{
\renewcommand{\arraystretch}{0.3}
\centering
\begin{tabular}{ccccccc}
\toprule
\rule{0pt}{0ex} & Block Stacking & Bottle Arrangement & House Building & Box Palletizing & Covid Test & Bin Packing \\
\midrule
expert steps & 50 & 240 & 200 & 1000 & 2000 & 2000\\
\midrule
equivalent episodes & 8 & 20 & 20 & 28 & 111 & 125\\
\bottomrule
\end{tabular}
\vspace{0.2cm}
\caption{The number of expert steps and the (approximate) equivalent number of episodes.}
\vspace{-0.7cm}
\label{tab:expert_step}
}
\end{table}

\subsection{Equivariant Augmented State $Q$ Functions in $\SE(2)$}
\label{sect:asrse2}


The FCN approach does not scale well to challenging manipulation problems. Therefore, we design an equivariant version of the augmented state representation (ASR) method of~\cite{asrse3}, which has been shown to be faster and have better performance. The ASR method transforms the original MDP with a high dimensional action space into a new MDP with an augmented state space but a lower dimensional action space. Instead of encoding the value of all dimensions of action in a single neural network, this model encodes the value of different factorized parts of the action space such as position and orientation using separate neural networks conditioned on prior action choices.

\begin{wrapfigure}[11]{r}{0.28\textwidth}
\vspace{-0.25cm}
\centering
\includegraphics[width=0.28\textwidth]{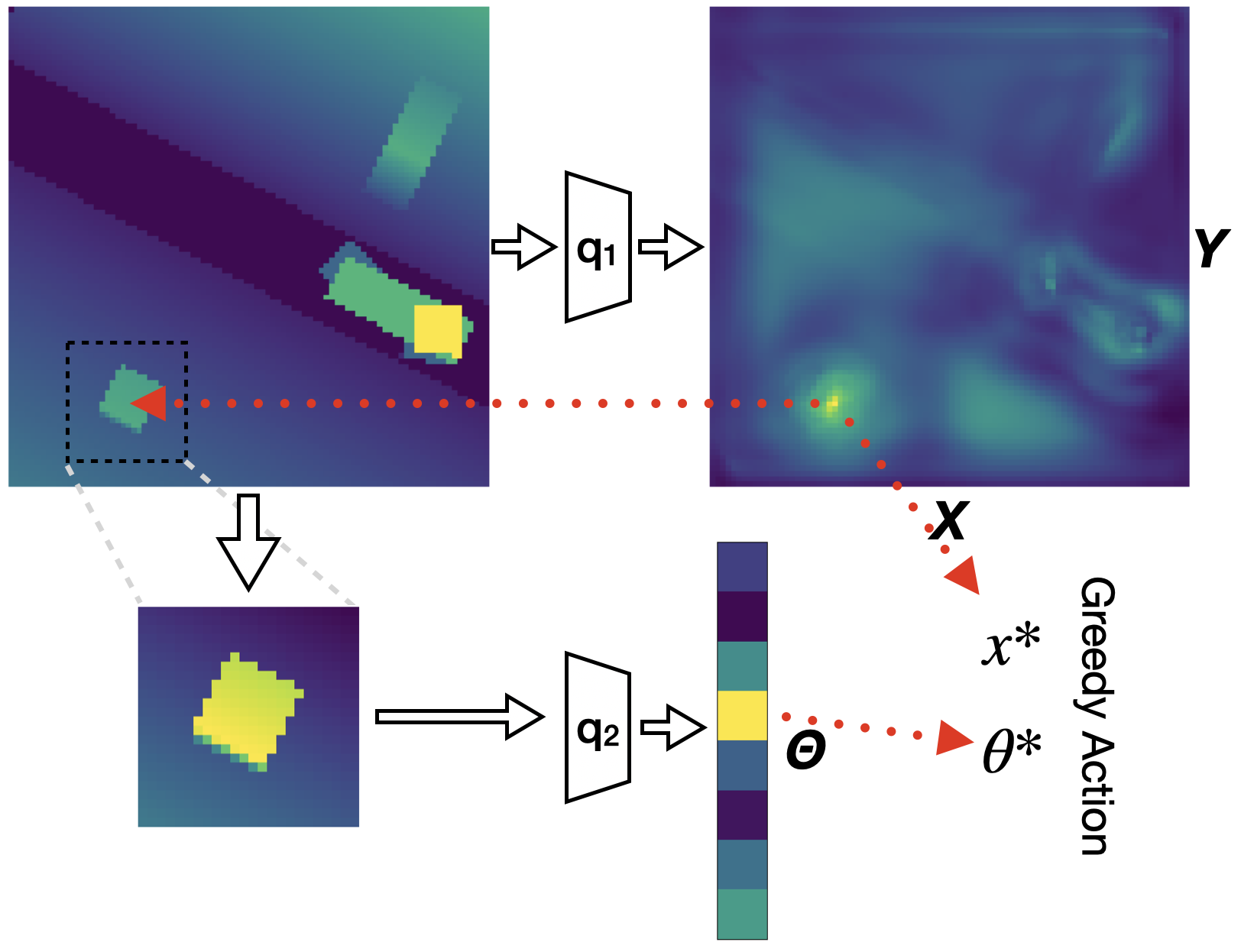}
\vspace{-0.5cm}
\caption{Illustration of the ASR approach in $\SE(2)$.}
\label{fig:asrse3}
\end{wrapfigure}

\underline{ASR in $\SE(2)$:} \edit{We explain the ASR method in the example setting of the $\SE(2)$ action space. See Fig~\ref{fig:asrse3} for an illustration. As before, actions $\hat{a}_{\rm{sp}} = (x,\theta)$ are elements of the space $\hat{\SE}(2)$, the finite approximation of $\SE(2)$ as in Section~\ref{sec:equi_conv}.
However, the $Q$ function is now computed using two separate functions, the position function $Q_1(s, x) = \mathrm{max}_\theta Q(s,(x,\theta))$ and the orientation function $Q_2((s, x), \theta) = Q(s, (x, \theta))$. $Q_1$ is encoded using a fully convolutional network $q_1: \mathbb{R}^{n \times h \times w} \rightarrow \mathbb{R}^{1 \times h \times w}$ that takes an $n$-channel image $I$ as input and produces a $1$-channel $Q$ map that describes $Q_1(s,x)$ for all $x$. We evaluate $Q_2$ on the ``augmented state'' $(s,x)$ which contains the state $s$ and the chosen $x$. The augmented state is encoded using the image patch $P = \textsc{crop}(I,x) \in \mathbb{R}^{n \times h' \times w'}$ cropped from $I$ and centered at $x$. We model $Q_2$ using the network $q_2 : \mathbb{R}^{n \times h' \times w'} \to \mathbb{R}^u$ that takes input $P$ and outputs $Q_2((s, x), \theta)$ for all $u$ different orientation $\theta$. These two networks are used together for both action selection and evaluation of target values during learning. We evaluate $x^* = \arg\max_{x} q_1(I)$, calculate $P = \textsc{crop}(I,x^*)$, and then evaluate $\theta^* = \arg\max_{\theta} q_2(P)$ and $Q^*=\max_{\theta}q_2(P)$.  Note that $Q$ maps produced by $q_1$ and $q_2$ are of size $u + hw$, significantly smaller than the $Q$ map in the FCN approach which is size $uhw$.  Essentially the ASR method takes advantage of the fact that the optimal $\theta$ depends only on the local patch $P$ given an optimal position $x$.}

\begin{wrapfigure}[13]{r}{0.46\textwidth}
\vspace{-0.7cm}
\centering
\subfloat[Block Stacking]{\includegraphics[width=0.23\textwidth]{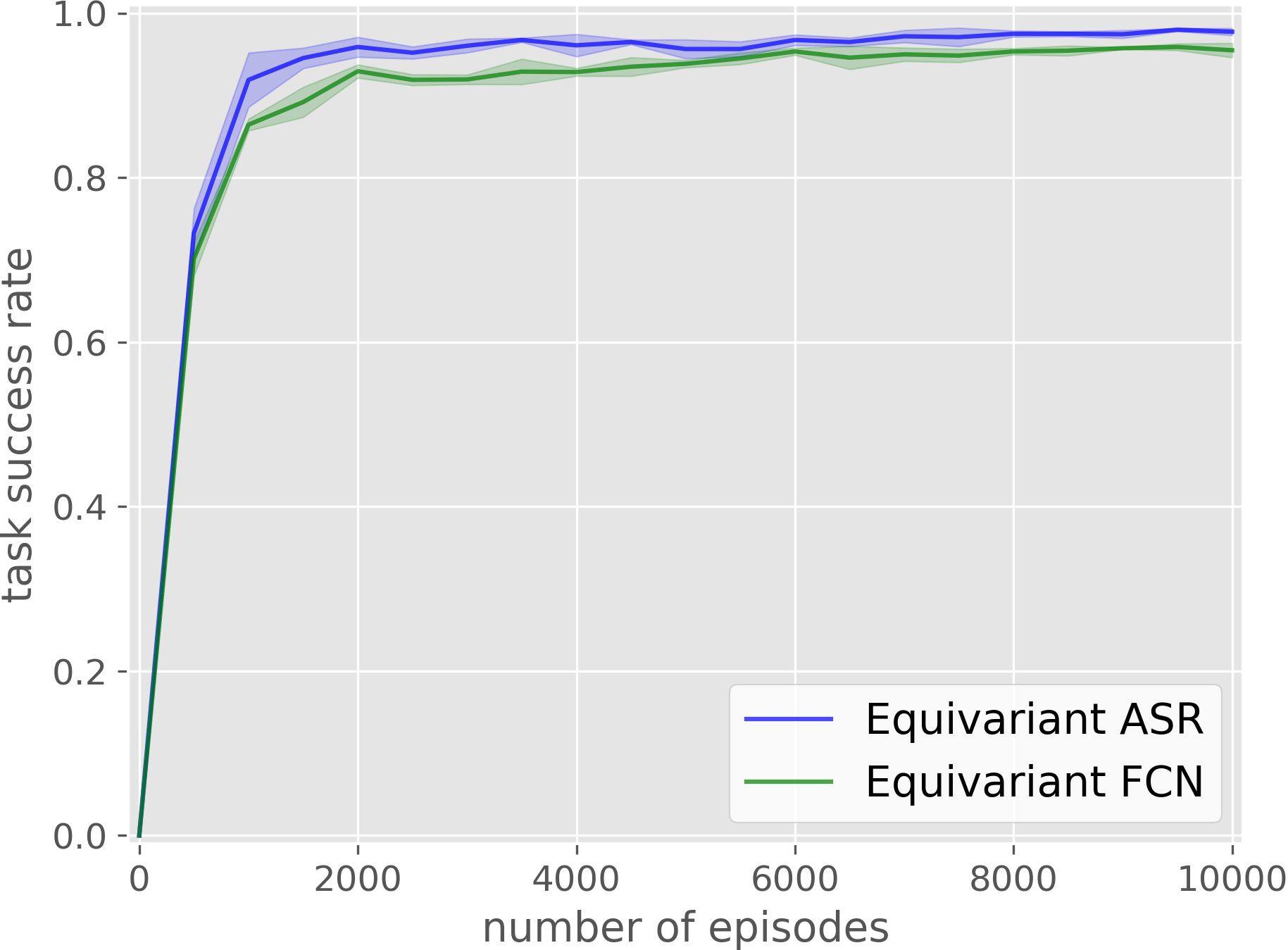}}
\subfloat[Bottle Arrangement]{\includegraphics[width=0.23\textwidth]{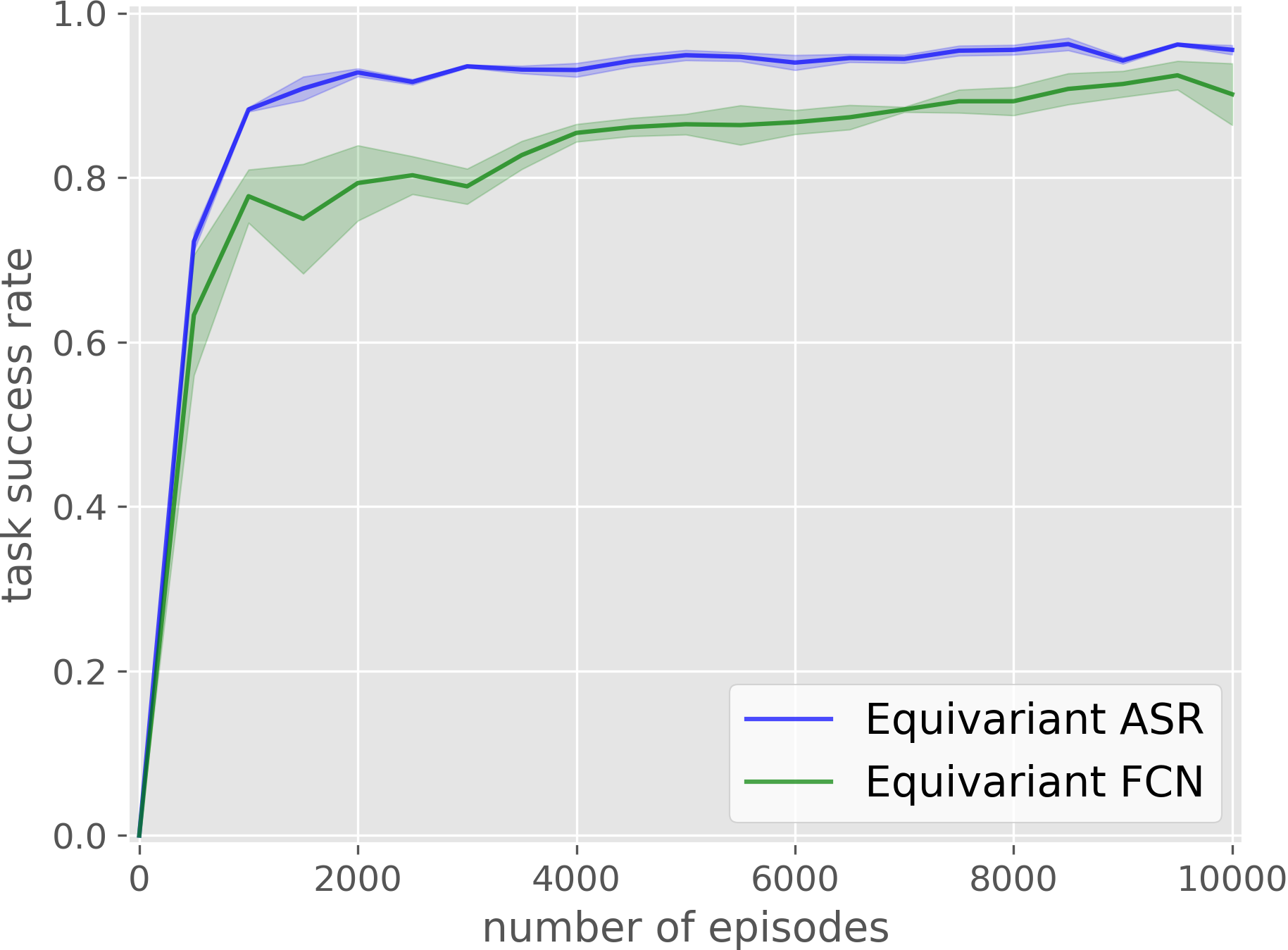}}
\caption{Comparison between Equivariant ASR (blue) and Equivariant FCN (green). Results averaged over four runs. Shading denotes standard error.}
\label{fig:exp_asr_vs_fcn}
\end{wrapfigure}

\underline{Equivariant architecture for ASR in $\SE(2)$:} We decompose the $\SE(2)$ equivariance property of Eq.~\ref{eqn:eqiv_q_fn} into two equivariance properties for $q_1$ and $q_2$, respectively: $q_1(g I) (x) = q_1(I)(g^{-1} x)$ where $g\in \hat{\SE}(2)$, and $q_2(g_\theta P) = \rho(g_\theta)q_2(P)$ where $g_\theta\in C_u$. The equivariance property of $q_1$ is similar to that of Eq.~\ref{eqn:layer_equiv_constraint} except that the output of $q_1$ has only one channel which is invariant to rotations (since it is a maximum over all rotations). This means we can rewrite the $q_1$ equivariance property as $q_1(g I) = g q_1(I)$, where $g$ on the RHS of this equation translates and rotates the output $Q$ map. In practice, we obtained the best performance for $q_1$ by enforcing equivariance to the Dihedral group $D_4$, which is generated by 90 degree rotation and reflections over the coordinate axis. For $q_2$, we used an equivariant feature map that outputs a single $u$-dimensional vector of $Q$ values corresponding to the finite cyclic group $C_u$ used. (We use $C_{12}/C_2$ and $C_{32}/C_2$ in our experiments below). We handle the partial equivariance using the same strategies as earlier. Appendix~\ref{appendix:network} describes the model details.


\underline{Experimental Comparison with Equivariant FCN:} Fig~\ref{fig:exp_asr_vs_fcn} shows a comparison between equivariant ASR (this section) and equivariant FCN (Section~\ref{sect:fcn}) for the Block Stacking and Bottle Arrangement tasks. The network $q_2$ is defined using $C_{12}$ and its quotient group $C_{12}/C_2$ to match Section~\ref{sect:fcn}. The ASR method surpasses the FCN  method in both tasks.

\label{sec:exp_equi_asr}
\underline{More Challenging Experimental Domains:} The equivariant ASR method is able to solve more challenging manipulation tasks than equivariant FCN can.  In particular, we could not run the FCN with as large a rotation space because it requires more GPU memory. We evaluate on the following four additional domains: House Building, Covid Test, Box Palletizing (introduced in~\cite{transporter}), and Bin Packing (Fig~\ref{fig:envs}(c-f)). 
All domains except Bin Packing have sparse rewards. In Bin Packing, the agent obtains a positive reward inversely proportional to the highest point in the pile after packing all objects. See Appendix~\ref{appendix:envs} for more details about the environments. We now define $q_2$ using the group $C_{32}$ and its quotient group $C_{32}/C_2$, i.e., we now encode 16 orientations ranging from 0 to $\pi$. As in Section~\ref{sect:fcn}, we use the SDQfD loss term to incorporate expert demonstrations (except for Transporter Net which uses standard behavior cloning exclusively). 
The number of expert transitions provided is shown in Table~\ref{tab:expert_step}.
9 random $\SE(2)$ augmentations are applied to the expert transitions.

\begin{figure}
\centering
\subfloat[House Building]{\includegraphics[width=0.23\linewidth]{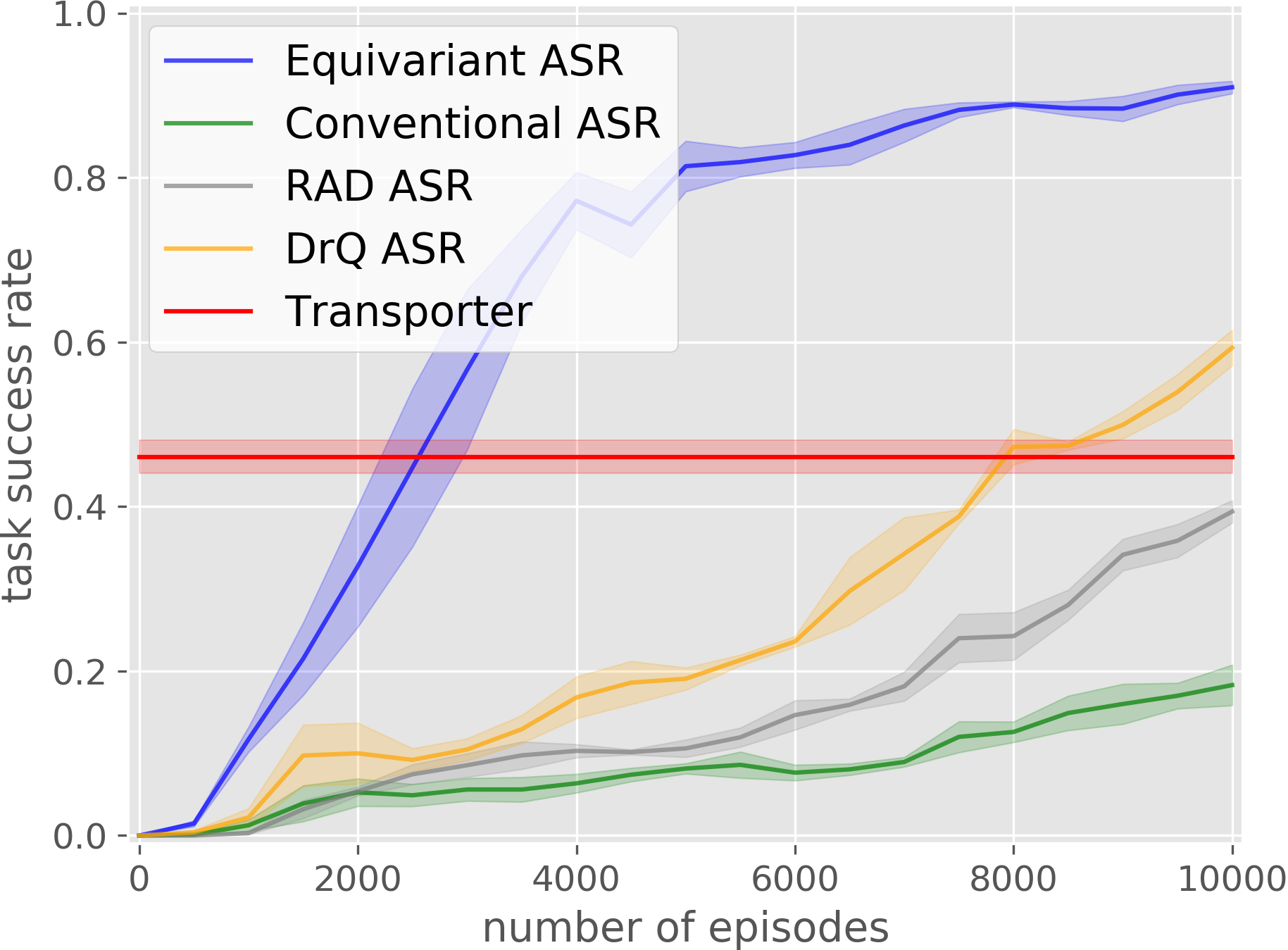}}
\subfloat[Covid Test]{\includegraphics[width=0.23\linewidth]{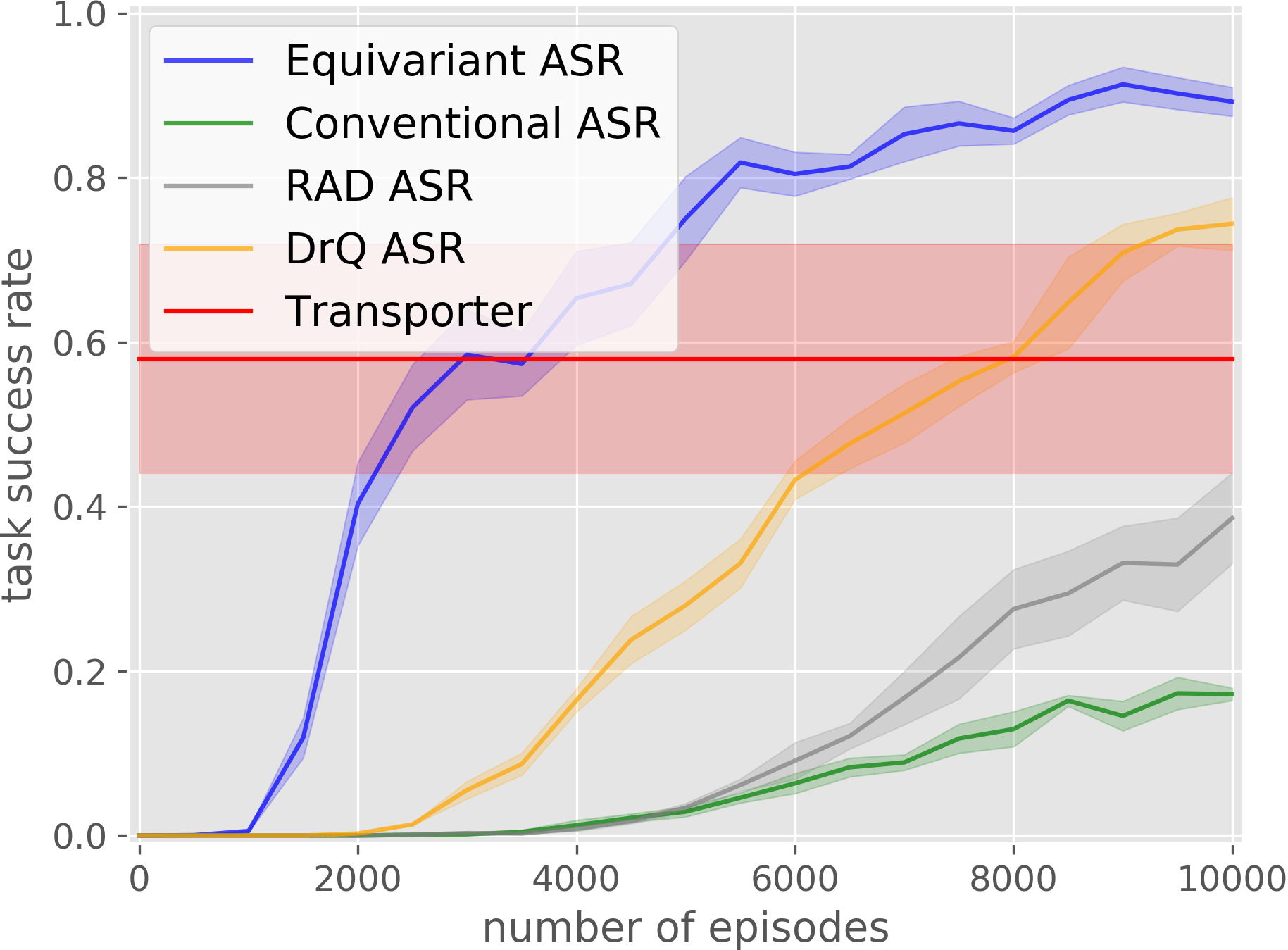}}
\subfloat[Box Palletizing]{\includegraphics[width=0.23\linewidth]{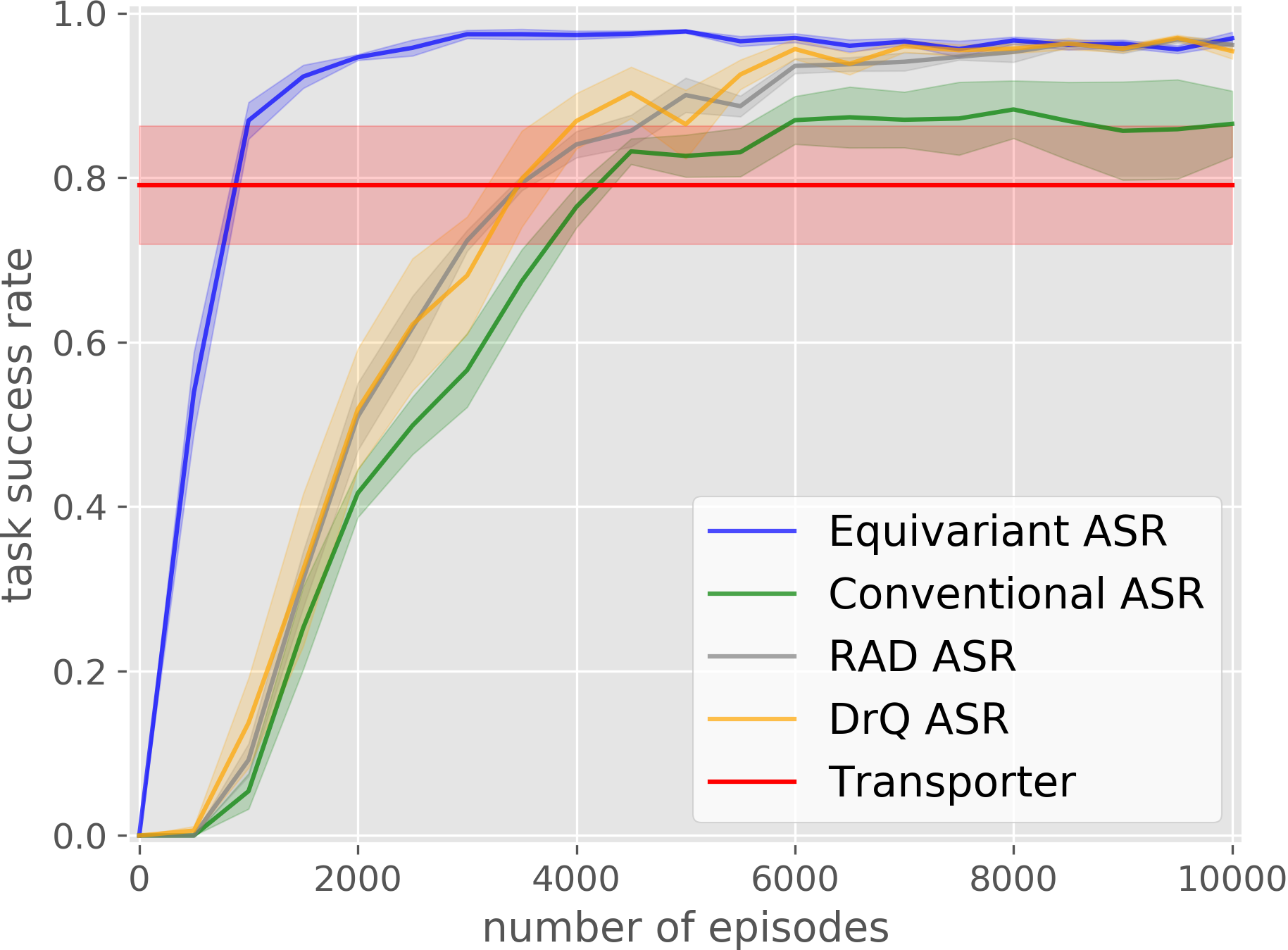}}
\subfloat[Bin Packing]{\includegraphics[width=0.23\linewidth]{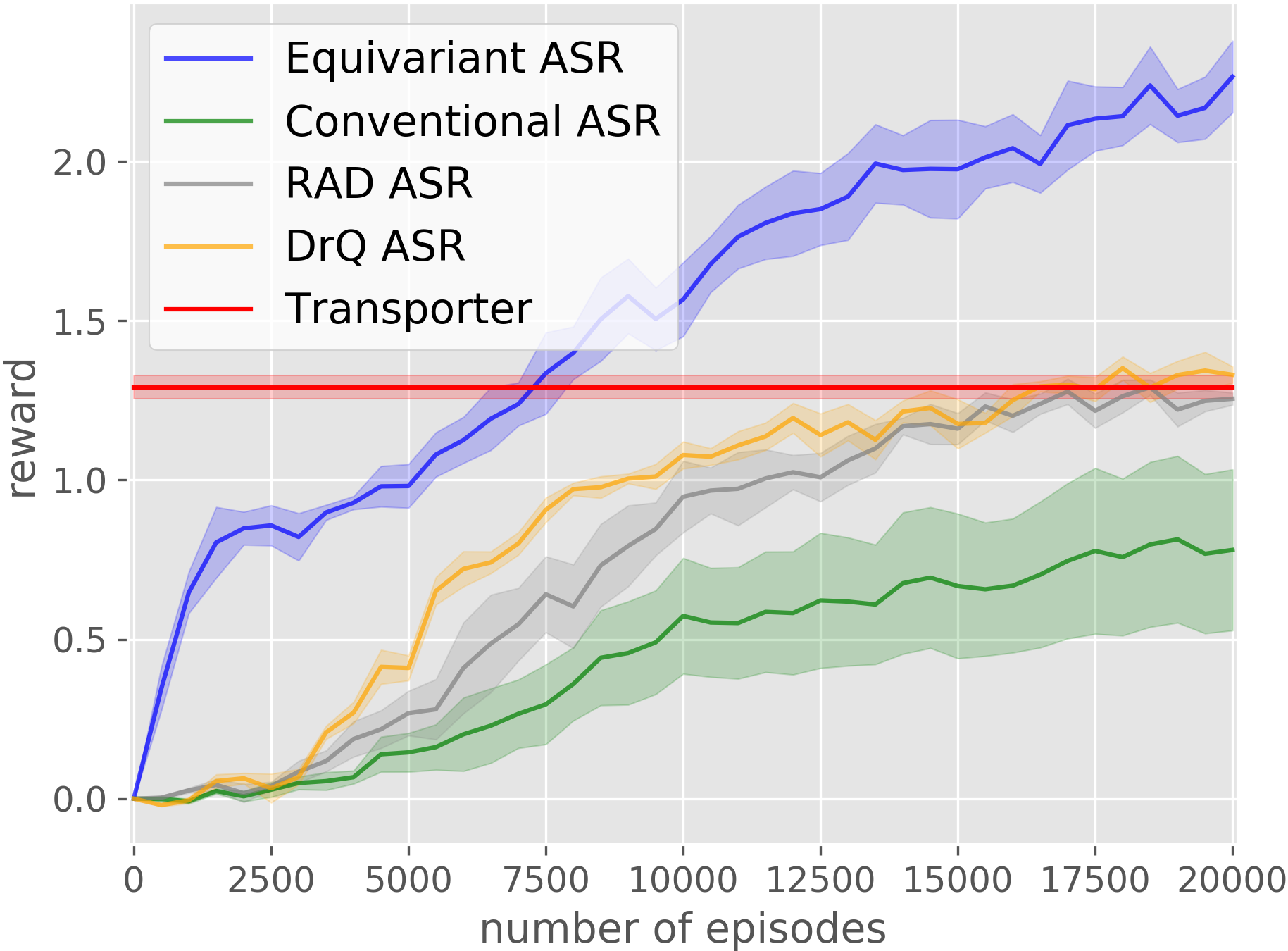}}
\caption{\edit{Comparison of Equivariant ASR (blue) with baselines. Results averaged over four runs. Shading denotes standard error.}}
\vspace{-0.2cm}
\label{fig:exp_asr_equi_vs_cnn}
\end{figure}

\underline{Experimental Comparison with Non-Equivariant Baselines:} We compare equivariant ASR against the following non-equivariant baselines: 1) Conventional ASR: ASR in $\SE(2)$ with conventional CNNs rather than equivariant layers. \edit{2) RAD~\cite{rl_with_aug} ASR: same architecture as (1) but each minibatch is augmented with a random rotation. 3): DrQ~\cite{kostrikov2020image} ASR: same architecture as (1) but each $Q$ target and $Q$ estimate are calculated by averaging over several augmented versions of the sampled transition. The augmentation in (2) and (3) is by random rotations sampled from $C_{32}$, the same group used in the equivariant model. }
4) the Transporter Network~\cite{transporter}. See Appendix~\ref{appendix:baseline_asr} for the network architecture for the baselines. The results in Fig~\ref{fig:exp_asr_equi_vs_cnn} show that equivariant ASR outperforms the other methods on all tasks, followed by DrQ, RAD, and Transporter Net, followed by Conventional ASR.

\begin{figure}[t]
\newlength{\robot}
\setlength{\robot}{0.14\linewidth}
\vspace{-0.2cm}
\centering
\subfloat{
\includegraphics[width=\robot]{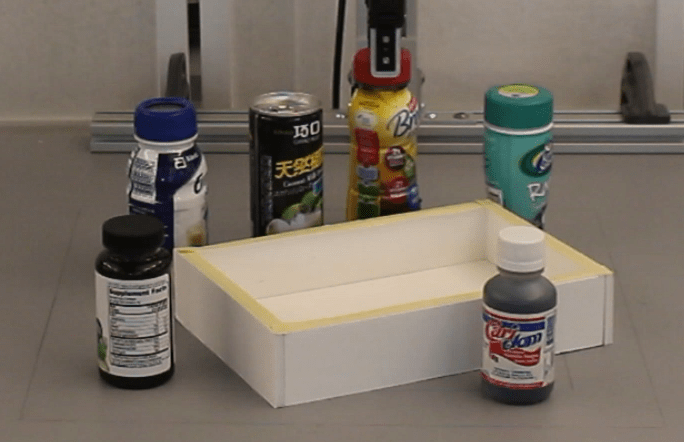}
}
\subfloat{
\includegraphics[width=\robot]{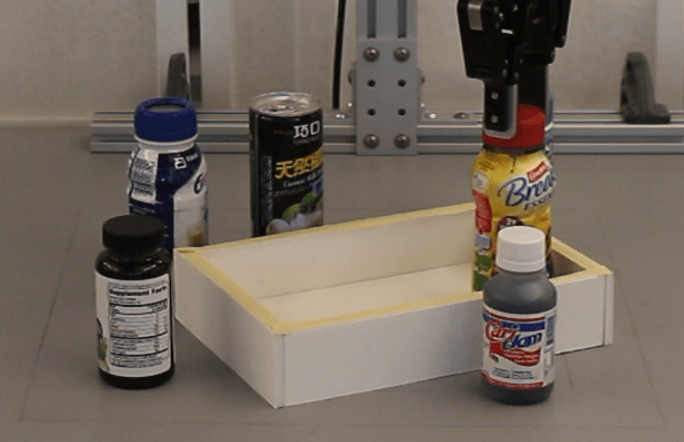}
}
\subfloat{
\includegraphics[width=\robot]{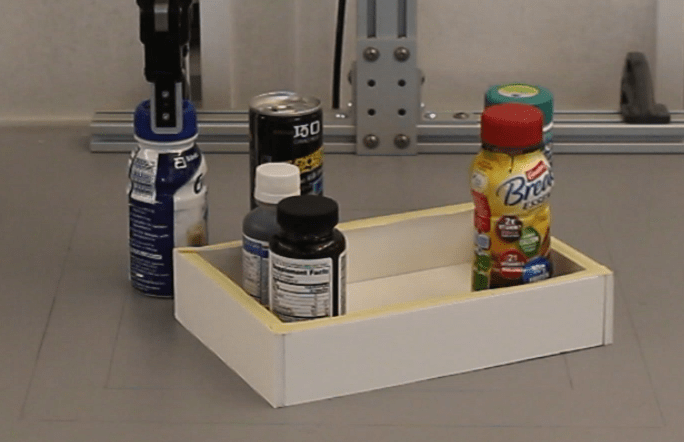}
}
\subfloat{
\includegraphics[width=\robot]{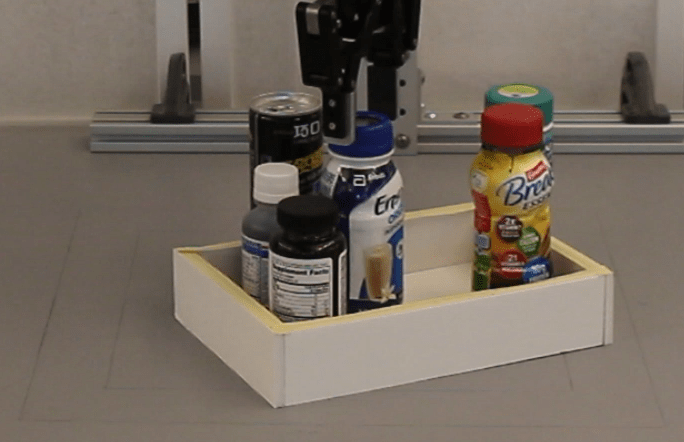}
}
\subfloat{
\includegraphics[width=\robot]{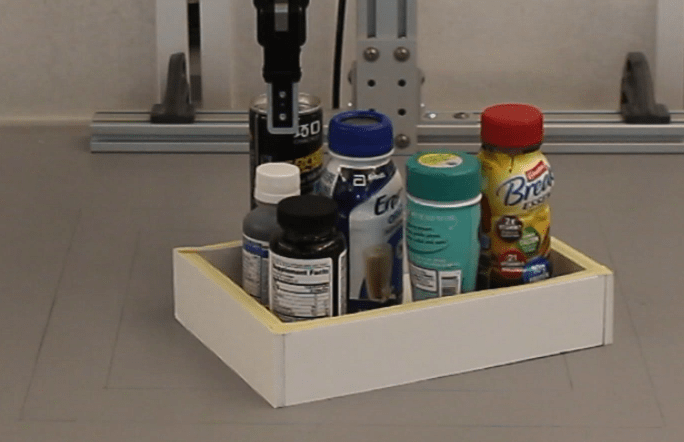}
}
\subfloat{
\includegraphics[width=\robot]{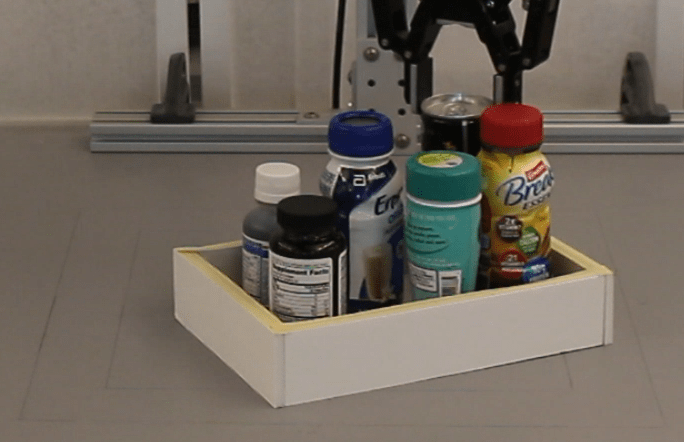}
}\\
\vspace{-0.25cm}
\subfloat{
\includegraphics[width=\robot]{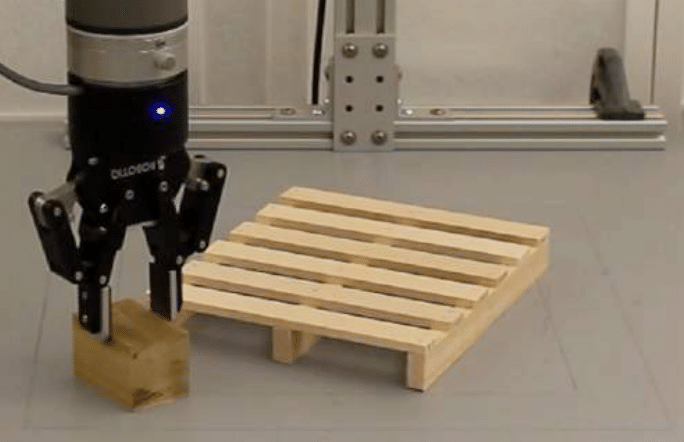}
}
\subfloat{
\includegraphics[width=\robot]{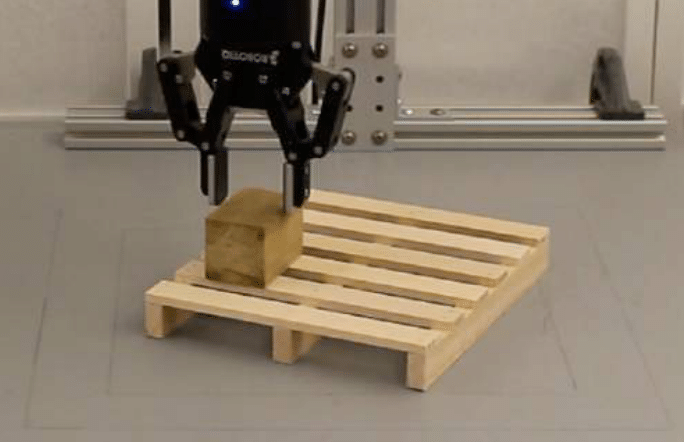}
}
\subfloat{
\includegraphics[width=\robot]{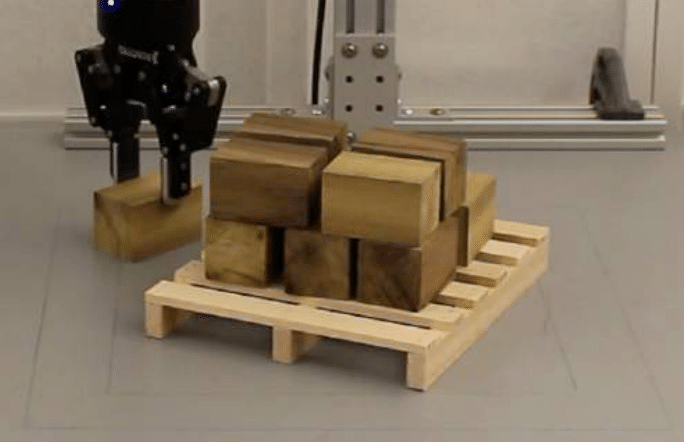}
}
\subfloat{
\includegraphics[width=\robot]{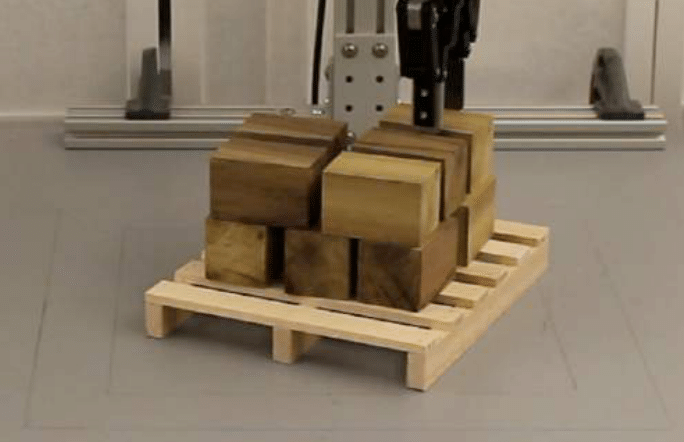}
}
\subfloat{
\includegraphics[width=\robot]{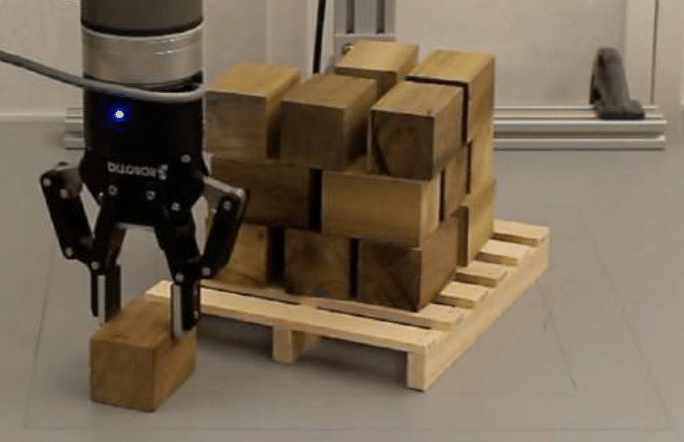}
}
\subfloat{
\includegraphics[width=\robot]{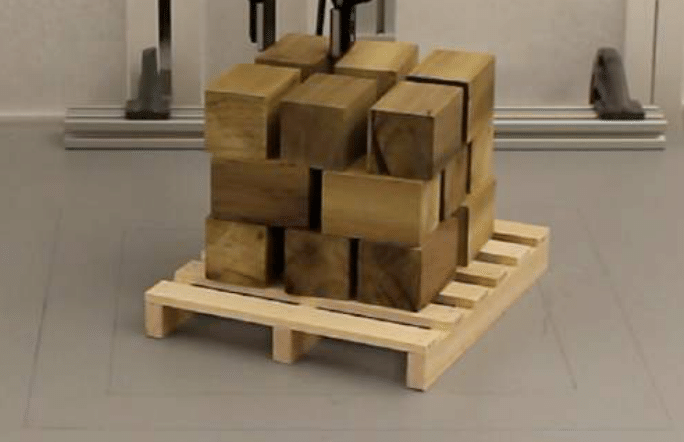}
}
\caption{Top row: the robot finishing the Bottle Arrangement task. Bottom row: the robot finishing the Box Palletizing task. Full episodes are in Appendix~\ref{appendix:robot_exp}.}
\vspace{-0.2cm}
\label{fig:robot_box}
\end{figure}

\begin{wraptable}[8]{r}{0.33\textwidth}
\centering
\scriptsize{
\vspace{-0.6cm}
\begin{tabular}{cc} \\\toprule  
Environment & SR \\\midrule
Bottle Arrangement & 90\%(18/20) \\\midrule
House Building & 100\%(20/20) \\\midrule
Box Palletizing & 95\%(19/20)\\\bottomrule
\end{tabular}
\caption{Robot experiment result}
\label{tab:robot_exp}
}
\end{wraptable}

\underline{Robot Experiment:} We evaluate the trained equivariant ASR models for Bottle Arrangement, House Building, and Box Palletizing on a Universal Robots UR5 arm equipped with a Robotiq 2F-85 gripper. The observation is provided by an Occipital Structure sensor mounted on top of the workspace. Table~\ref{tab:robot_exp} shows the results. In Bottle Arrangement, the robot shows a 90\% success rate. In one of the two failures, the arrangement is not compact enough, leaving no enough space left for the last bottle. In the other failure, the robot arranges the bottles outside of the tray. In the House Building task, the robot succeeds in all 20 episodes. In the Box Palletizing task, the robot demonstrates a 95\% success rate. In the failure, the robot correctly stacks 16 of 18 boxes, but the 17th box's placement position is offset slightly from the the rest of the stack and there is no room to place the last box. The same problem happens in another successful episode, where the fingers squeeze the boxes and make room for the last box. Fig~\ref{fig:robot_box} shows two example episodes in the robot experiment.

\subsection{Equivariant Augmented State $Q$ Functions in $\SE(3)$}

\begin{wrapfigure}[21]{r}{0.46\textwidth}
\centering
\vspace{-0.6cm}
\subfloat[House Building]{
\includegraphics[width=0.1\textwidth]{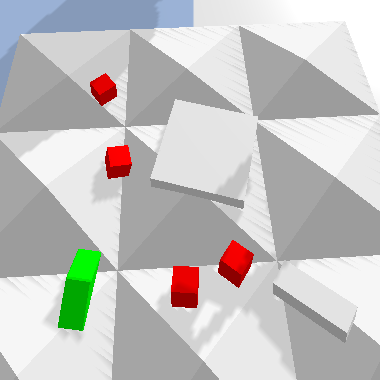}
\includegraphics[width=0.1\textwidth]{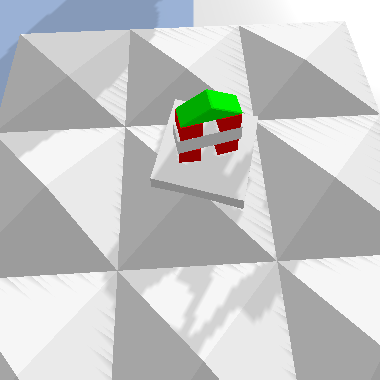}
\label{fig:env_bumpy_h4}
}
\subfloat[Box Palletizing]{
\includegraphics[width=0.1\textwidth]{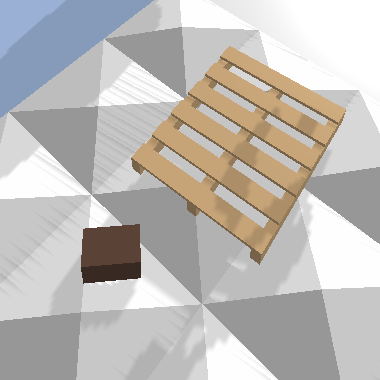}
\includegraphics[width=0.1\textwidth]{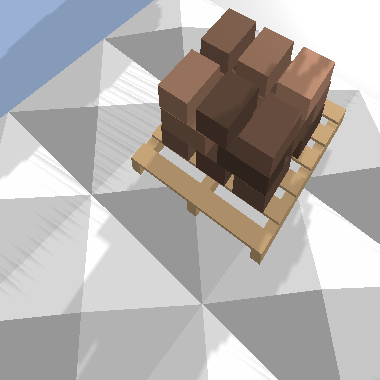}
\label{fig:env_bumpy_box}
}\\
\caption{The 6DOF experimental domains.\label{fig:envs_6d}}
\subfloat[House Building]{\includegraphics[width=0.23\textwidth]{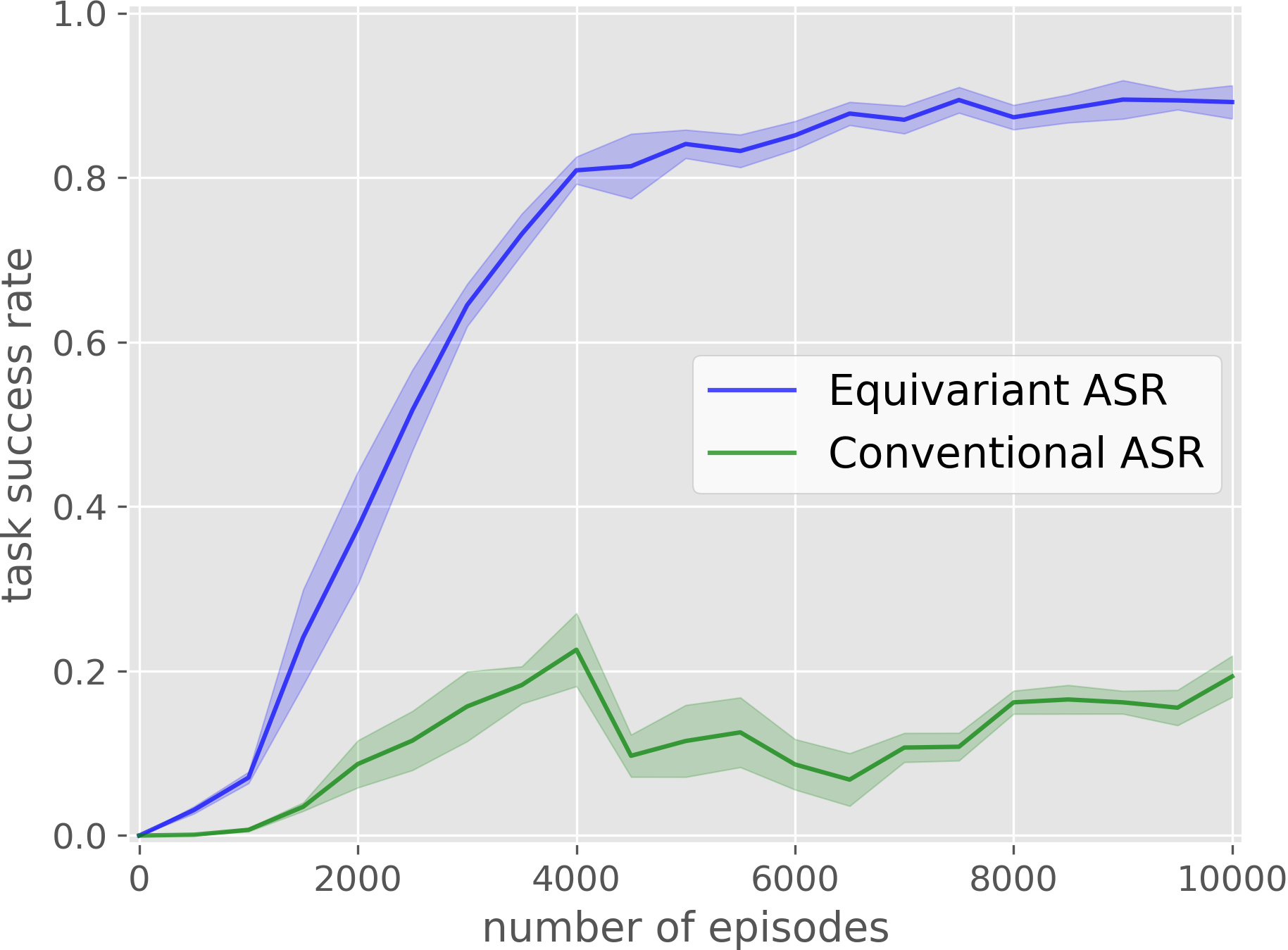}}
\subfloat[Box Palletizing 18]{\includegraphics[width=0.23\textwidth]{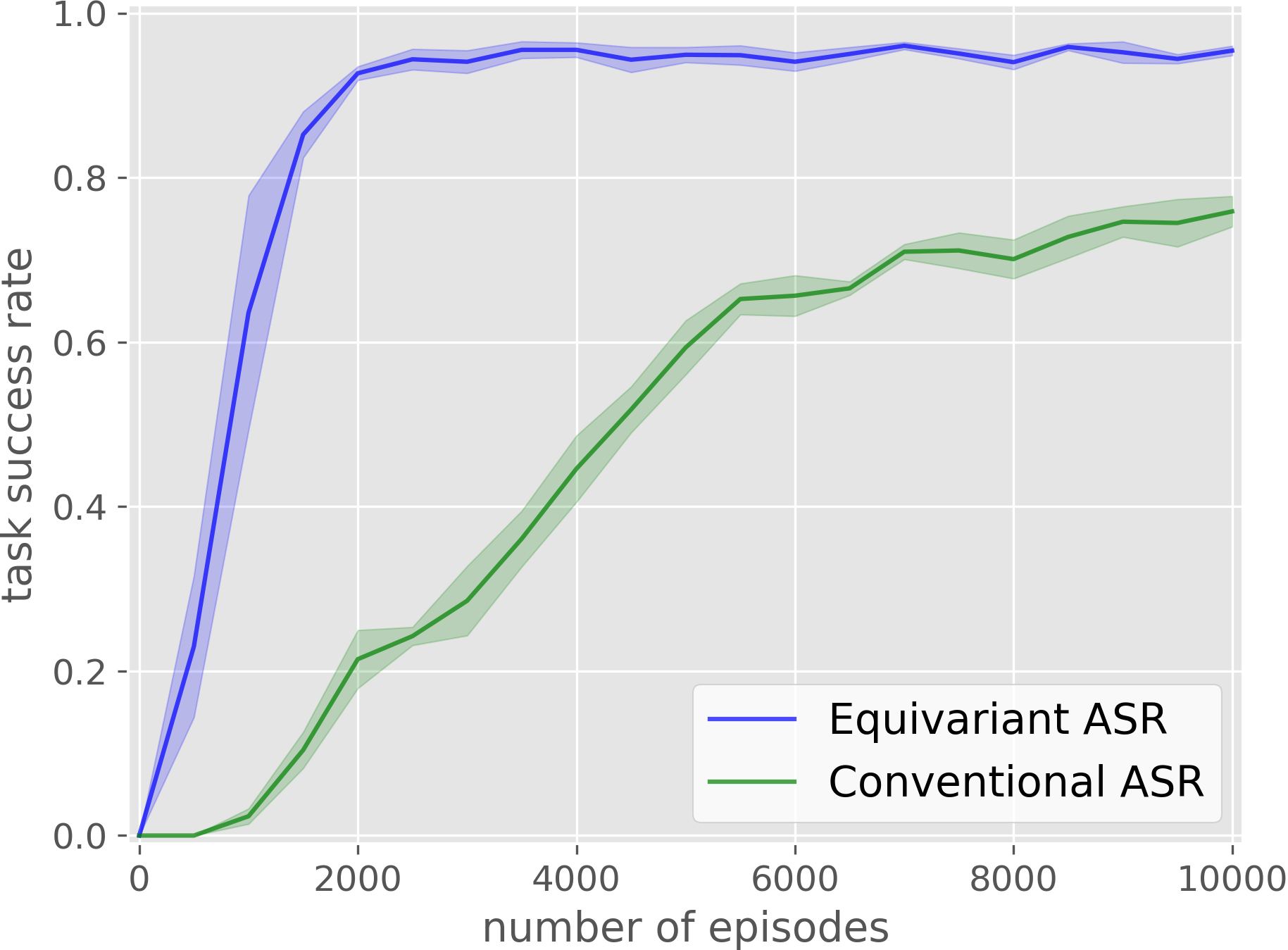}}
\caption{Comparison of the equivariant network with the baseline in 6DOF tasks. Results averaged over four runs. Shading denotes standard error.\label{fig:exp_6d}}
\end{wrapfigure}

A strength of the ASR method is that it can be extended into $\SE(3)$ by adding networks similar to $q_2$ that encode $Q$ values for additional dimensions of the action space~\cite{asrse3}.  Specifically, we add three networks to the $\SE(2)$-equivariant architecture described in Section~\ref{sect:asrse2}, \edit{$q_3$, $q_4$, and $q_5$ encoding $Q$ values for $Z$ (height above the plane), and angles $\Phi$ (rotation in XZ plane) and $\Psi$ (rotation in YZ plane) dimensions of $\SE(3)$. Each of these networks takes as input a stack of 3 orthographic projections of a point cloud along the coordinate planes. The point cloud is re-centered and rotated to encode the partial $\SE(3)$ actions. (see~\cite{asrse3} for details).}

Unfortunately, we cannot easily make these additional networks equivariant using the same methods we have proposed so far because they encode variation outside of $\SE(2)$. Instead, we create an encoding that is approximately equivariant by explicitly transforming the input to these networks in a way that corresponds to a set of candidate robot positions or orientations (called a ``deictic'' encoding in~\cite{deictic}). We will describe this idea using $q_2$ as an example. \edit{Define $q'_2(P) \in \mathbb{R}$ to output the single $Q$ value of the rotation encoded in the input image $P$. Then $q_2$ can be defined as a vector-valued function: $\hat{q}_2(P) = (q'_2(g_1^{-1} P), \dots,  q'_2(g_u^{-1} P))$, where $g_1, \dots, g_u \in C_u$. $\hat{q}_2$ is approximately equivariant because everything learned by $q'_2$ is automatically replicated in all orientations. We design deictic $q_3$, $q_4$, and $q_5$ similarly by selecting a finite subset of $\lbrace g_1,\ldots,g_K \rbrace \subset \SE(3)$ corresponding to the dimension of the action space encoded by each $q_i$. $q_i$ can then be defined by evaluating a network $q'_i$ over input $P$ transformed by $(g_i)_{i=1}^K$. For $q_3$, we evaluate over 36 translations $g_k(z) = z + k (0.18/36) + 0.02$ where $0 \leq k \leq 35$. For $q_4$ and $q_5$, we use rotations $g_k \in \{-\frac{\pi}{8}, -\frac{\pi}{12}, -\frac{\pi}{24}, 0, \frac{\pi}{24}, \frac{\pi}{12}, \frac{\pi}{8}\}$. Note we use $q_2$ for explanation, while our model uses equivariant $q_1$, $q_2$ and deictic $q_3$-$q_5$.  For an ablated version using deictic $q_2$, see Appendix~\ref{appendix:exp_deictic_se2} and~\ref{appendix:exp_deictic_se3}.}
\underline{Comparison to Non-Equivariant Approaches:} We evaluate ASR in $\SE(3)$ in the House Building and Box Palletizing domains. We modified those environments so that objects are presented randomly on a bumpy surface with a maximum out of plane orientation of 15 degrees (Fig~\ref{fig:envs_6d}). In order to succeed, the agent must correctly perform pick and place actions with the needed height and out of plane orientation. 
We evaluated the Equivariant ASR in comparison with a baseline Conventional ASR (same as~\cite{asrse3}). Both methods use SDQfD with 2000 expert demonstration steps. The results are shown in Fig~\ref{fig:exp_6d}. Our proposed approach outperforms the baseline by a significant margin. 


\section{Discussion}

In this paper, we show that equivariant neural network architectures can be used to improve $Q$ learning in spatial action spaces. We propose multiple approaches and model architectures that can be used to accomplish this and demonstrate improved sample efficiency and performance on several robotic manipulation applications both in simulation and on a physical system.
This work has several limitations and directions for future research. First, our methods apply directly only to problems in spatial action spaces. While many robotics problems can be expressed this way, it would clearly be useful to develop equivariant models for policy learning that can be used in other settings. Second, although we extend our ASR approach from $\SE(2)$ to $\SE(3)$ in the last section of this paper, this solution is not fully equivariant in $\SE(3)$ and it may be possible to do better by exploiting methods that are directly equivariant in $\SE(3)$.

\clearpage
\acknowledgments{This work is supported in part by NSF 1724257, NSF 1724191, NSF 
1763878, NSF 1750649, and NASA 80NSSC19K1474. R. Walters is supported by a Postdoctoral Fellowship from the Roux Institute and NSF grants 2107256 and 2134178.}


\bibliography{main}  

\clearpage

\appendix

\section{Proof of Proposition~\ref{prop:qstar}}
\label{appendix:proof}

We need the following Lemma regarding the visual state space $S$ and spatial action space $A$ described in Section~\ref{sect:problem}. We use the following notation: $gS = \{ gs | s \in S\}$ and $gA = \{ ga | a \in A \}$.

\begin{lemma}
\label{lemma:1}
Let $S$ be a visual state space and let $A$ by a spatial action space. Then, $\forall g \in \SE(2)$, we have that $S = gS$ and $A = gA$.
\end{lemma}

\begin{proof}
First, consider the claim that $S = gS$. We will show 1) $S \subseteq gS$ and 2) $gS \subseteq S$. 1) $S \subseteq gS$: This follows from the closure of state under $g \in \SE(2)$. 2) $gS \subseteq S$: Let $s' \in gS$. By the definition of $gS$, $\exists s \in S$ such that $gs = s'$ and $gs \in gS$. Multiplying both sides by $g^{-1}$, we have $g^{-1}(gs) \in g^{-1}(gS)$. Using Assumption~\ref{assumption:invertability}, we have $s \in S$. Using the closure of state under $g$, we have $gs \in S$ or $s' \in S$. A parallel argument can be used to show $A = gA$.
\end{proof}

\begin{customprop}{4.1}
Given an MDP $\mathcal{M} = (S,A,T,R,\gamma)$ for which Assumptions~\ref{assumption:goalinv}, \ref{assumption:transinv}, and \ref{assumption:invertability} are satisfied, the optimal $Q$ function is invariant to translation and rotation, i.e. $Q^*(s,a) = Q^*(gs,ga)$, for all $g \in \SE(2)$.
\end{customprop}

\begin{proof}[Proof of Proposition~\ref{prop:qstar}]

The Bellman optimality equations for $Q^*(s,a)$ and $Q^*(gs,ga)$ are, respectively:
\begin{equation}
\label{eqn:pf1}
Q^*(s,a) = R(s,a) + \gamma \sup_{a' \in A} \int_{s' \in S} T(s,a,s') Q^*(s',a'),
\end{equation}
and
\begin{equation}
\label{eqn:pf2}
Q^*(gs,ga) = R(gs,ga) + \gamma \sup_{a' \in A} \int_{s' \in S} T(gs,ga,s') Q^*(s',a').
\end{equation}
Using Lemma~\ref{lemma:1}, we can rewrite Eq.~\ref{eqn:pf2} as:
\begin{align}
\label{eqn:pf3}
Q^*(gs,ga) & = R(gs,ga) + \gamma \sup_{\bar{a}' \in gA} \int_{\bar{s}' \in gS} T(gs,ga,\bar{s}') Q^*(\bar{s}',\bar{a}') \\
& = R(gs,ga) + \gamma \sup_{a' \in A} \int_{s' \in S} T(gs,ga,gs') Q^*(gs',ga').
\end{align}
Using Assumptions~\ref{assumption:goalinv} and~\ref{assumption:transinv}, this can be written:
\begin{equation}
\label{eqn:pf4}
Q^*(gs,ga) = R(s,a) + \gamma \sup_{a' \in A} \int_{s' \in S} T(s,a,s') Q^*(gs',ga').
\end{equation}
Now, define a new function $\bar{Q}$ such that $\forall s,a \in S \times A$, $\bar{Q}(s,a) = Q(gs,ga)$ and substitute into Eq.~\ref{eqn:pf4}, resulting in:
\begin{equation}
\label{eqn:pf5}
\bar{Q}^*(s,a) = R(s,a) + \gamma \sup_{a' \in A} \int_{s' \in S} T(s,a,s') \bar{Q}^*(s',a').
\end{equation}
Notice that Eq.~\ref{eqn:pf5} and Eq.~\ref{eqn:pf1} are the same Bellman equation. Since solutions to the Bellman equation are unique, we have that $\forall s,a \in S \times A$, $Q^*(s,a) = \bar{Q}^*(s,a) = Q^*(gs,ga)$.


\end{proof}

\section{Equivariant Kernel Constraint} 
\label{appendix:equi_kernel_constraint}
Consider a standard convolutional layer that takes an $n \times h \times w$ feature map as input and produces an $m \times h \times w$ map as output. It computes $h_i(x) = \sum_{y,j} K_{ij}(y) I_j(x+y)$, where $j \in \{1 \dots n\}$, $i \in \{1 \dots m\}$, $I_j(x)$ is the value of the input at the $x$ pixel and the $j$ channel, $h_i(x)$ is the output at pixel $x$ and channel $i$, and $K_{ij}(y)$ is the kernel value at $y$ for the $j$ input and $i$ output channels. For a standard convolutional layer, $I_j(x)$, $h_i(x)$, and $K_{ij}(y)$ are all scalars. However, for an equivariant network over $C_u$, $h_i(x)$ becomes a $u$-element vector and $K_{ij}(y)$ becomes a $u \times u$ matrix. The $u$ elements of $h_i(x)$ encode the feature values of pixel $x$ at channel $i$ at each orientation in $C_u$. The kernel constraint is~\cite{cohen_equicnn_theory}:
\begin{equation}
K_{ij}(g_\theta y) = \rho_{\rm{out}}(g_\theta) K_{ij}(y) \rho_{\rm{in}}(g_\theta)^{-1},
\end{equation}
where $\rho_{\rm{in}}(g_\theta)$ and $\rho_{\rm{out}}(g_\theta)$ are the permutation matrix of the group element $g_\theta$ (note that for the first layer, $K_{ij}(y)$ will be a $1\times u$ matrix, and $\rho_{\rm{in}}(g_\theta)$ will be 1).

\section{Experimental Domains}
\label{appendix:envs}

\begin{figure}[t]
\centering
\subfloat[]{
\includegraphics[height=2cm]{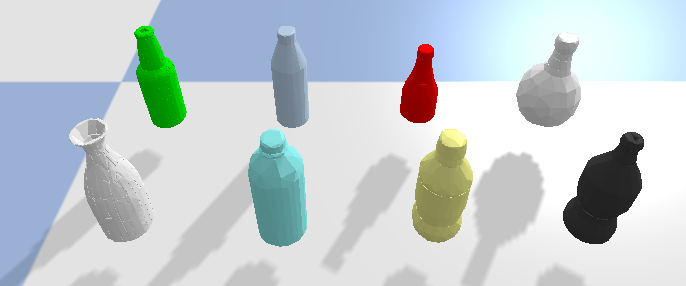}
\label{fig:env_all_bottle}
}
\subfloat[]{
\includegraphics[height=2cm]{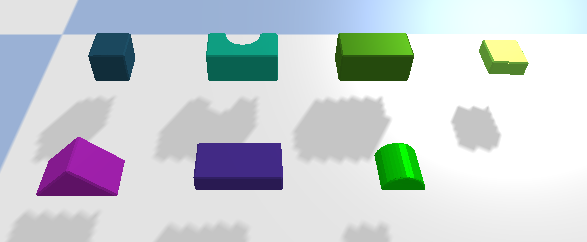}
\label{fig:env_all_packing}
}
\caption{(a) All eight bottle models in the Bottle Arrangement task. (b) All seven object models in the Bin Packing task.}
\end{figure}
\subsection{Block Stacking}
In the Block Stacking task (Fig~\ref{fig:envs_4s}), there are four cubic blocks with a fixed size of $3cm\times 3cm\times 3cm$ randomly placed in the workspace. The goal is stacking all four blocks in a stack. An optimal policy requires six steps to finish this task, and the maximal number of steps per episode is 10.
\subsection{Bottle Arrangement}
In the Bottle Arrangement task (Fig~\ref{fig:envs_bt}), six bottles with random shapes (sampled from 8 different shapes shown in Fig~\ref{fig:env_all_bottle}. The bottle shapes are generated from the 3DNet dataset~\cite{3dnet}. The sizes of each bottle are around $5cm\times 5cm\times 14cm$) and a tray with a size of $24cm\times 16cm\times 5cm$ are randomly placed in the workspace. The agent needs to arrange all six bottles in the tray. An optimal policy requires 12 steps to finish this task, and the maximal number of steps per episode is 20.
\subsection{House Building}
In the House Building task (Fig~\ref{fig:envs_h4}), there are four cubes with a size of $3cm\times 3cm\times 3cm$, a brick with a size of $12cm\times 3cm\times 3cm$, and a triangle block with a bounding box size of around $12cm\times 3cm\times 3cm$. The agent needs to stack those blocks in a specific way to build a house-like block structure as shown in Fig~\ref{fig:envs_h4}. An optimal policy requires 10 steps to finish this task, and the maximal number of steps per episode is 20.
\subsection{Covid Test}
In the Covid Test task (Fig~\ref{fig:envs_covid}), there is a new tube box (purple), a test area (gray), and a used tube box (yellow) placed arbitrarily in the workspace but adjacent to one another. Three swabs with a size of $7cm \times 1cm \times 1cm$ and three tubes with a size of $8cm \times 1.7cm \times 1.7cm$ are initialized in the new tube box. To supervise a COVID test, the robot needs to present a pair of a new swab and a new tube from the new tube box to the test area (see the middle figure in Fig~\ref{fig:envs_covid}). The simulator simulates the user testing COVID by putting the swab into the tube and randomly place the used tube in the test area. Then the robot needs to re-collect the used tube into the used tube box. Each episode includes three rounds of COVID test. An optimal policy requires 18 steps to finish this task, and the maximal number of steps per episode is 30.


\subsection{Box Palletizing}
In the Box Palletizing task (Fig~\ref{fig:envs_box18}) (some object models are derived from~\citet{transporter}), a pallet with a size of $23.2cm\times 19.2cm\times 3cm$ is randomly placed in the workspace. The agent needs to stack 18 boxes with a size of $7.2cm\times 4.5cm\times 4.5cm$ as shown in Fig~\ref{fig:envs_box18}. At the beginning of each episode and after the agent correctly places a box on the pallet, a new box will be randomly placed in the empty workspace. An optimal policy requires 36 steps to finish this task, and the maximal number of steps per episode is 40.

\subsection{Bin Packing}
In the Bin Packing task (Fig~\ref{fig:envs_bin}), eight objects (the shape of each is randomly sampled from seven different object in Fig~\ref{fig:env_all_packing}. Object models are derived from~\citet{zeng_pushing}) with a maximum size of $8cm\times 4cm\times 4cm$ and a minimum size of $4cm\times 4cm \times 2cm$ and a bin with a size of $17.6cm\times 14.4cm\times 8cm$ are randomly placed in the workspace. The agent needs to pack all eight objects in the bin while minimizing the highest point ($h_{max}$ cm) of all objects in the bin. The Bin Packing task has real value sparse rewards: a reward of $8-h_{max}$ is given when all objects are placed in the bin. An optimal policy requires 16 steps to finish this task, and the maximal number of steps per episode is 20.

\subsection{$\SE(3)$ House Building and Box Palletizing}
In the $\SE(3)$ House Building (Fig~\ref{fig:env_bumpy_h4}) and the Box Palletizing (Fig~\ref{fig:env_bumpy_box}) tasks, a bumpy surface is generated by nine pyramid shapes with a random angle sampled from 0 to 15 degrees. The orientation of the bumpy surface along the z axis is randomly sampled at the beginning of each episode. In the Bumpy House Building task, a flat platform with a size of $13cm\times 13cm$ and a height same as the highest bump is randomly placed in the workspace. The agent needs to build the house on top of the platform. In the Bumpy Box Palletizing task, the pallet is raised by the same height as the highest bump (so that it will be horizontal to the ground). All other parameters mirror the original House Building task and the original Box Palletizing task.

\section{Network Architecture}
\label{appendix:network}
All of our network architectures are implemented using PyTorch~\cite{pytorch}. We use the e2cnn~\cite{e2cnn} library to implement the steerable convolutional layers. Appendix~\ref{appendix:network_fcn} and Appendix~\ref{appendix:network_asr} respectively show the network architectures of the equivariant FCN and equivariant ASR using the dynamic filter for partial equivariance. Appendix~\ref{appendix:network_exp} shows the architecture of lift expansion partial equivariance. Appendix~\ref{appendix:network_deictic} shows the architecture of the deictic encoding.
\begin{figure}[t]
\centering
\subfloat[Equivariant FCN Network Architecture]{\includegraphics[width=0.85\linewidth]{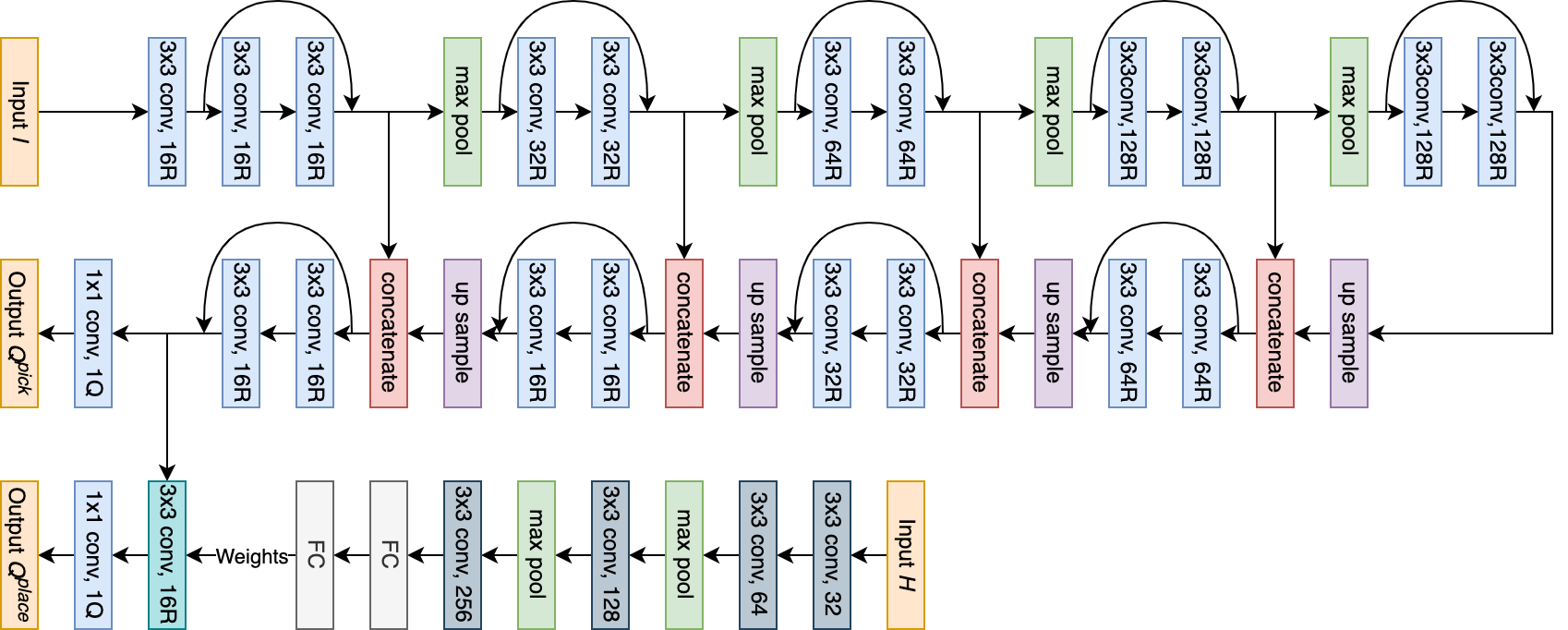}\label{fig:network_equi_fcn}}\\
\subfloat[Equivariant ASR Network Architecture]{\includegraphics[width=\linewidth]{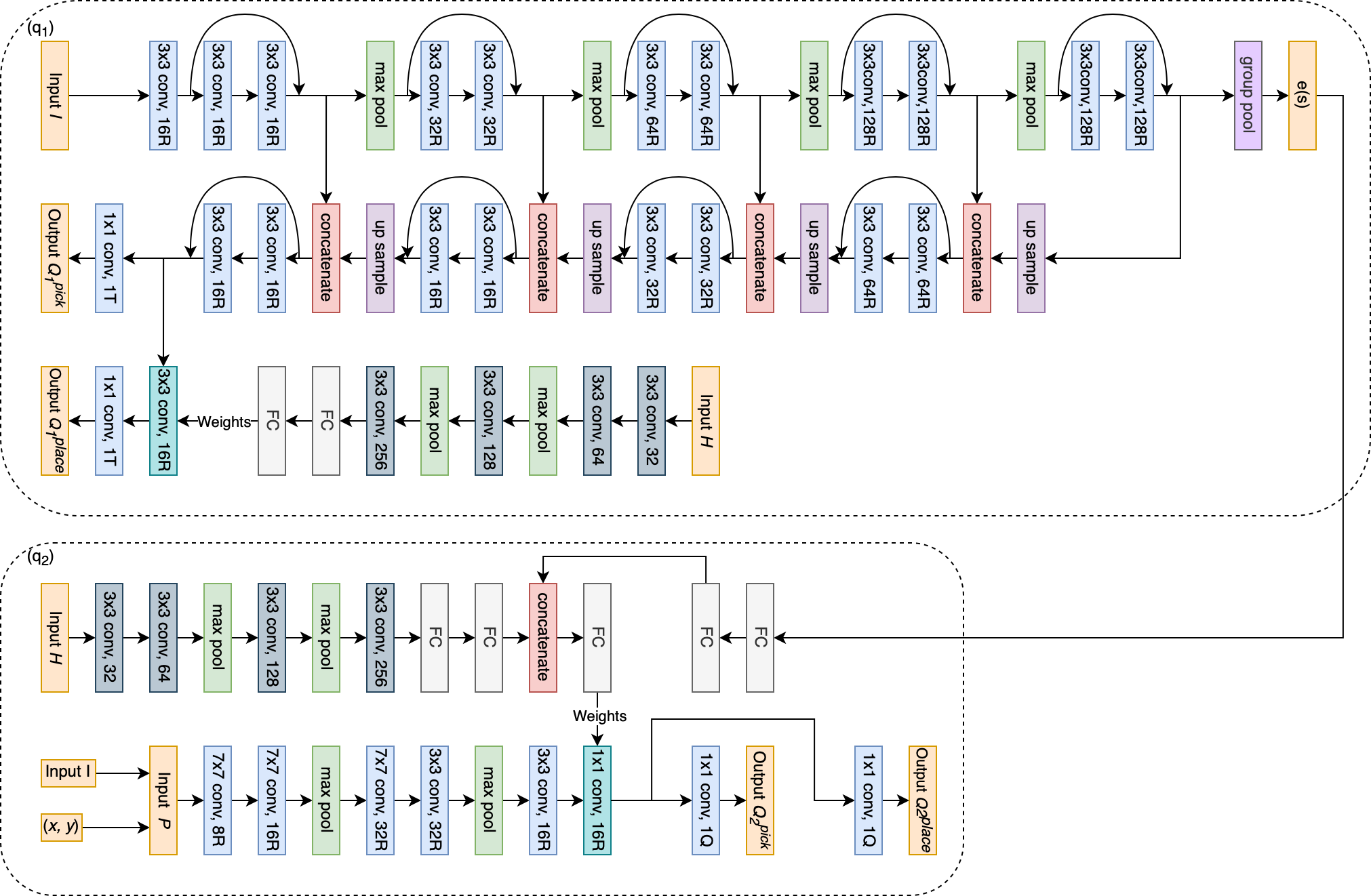}\label{fig:network_equi_asr}}
\caption{The architecture of the Equivariant FCN Network (a) and the equivariant ASR Network (b). ReLU nonlinearity is omitted in the figure. A convolutional layer with a suffix of R indicates a regular representation layer (e.g., 16R is a 16-channel regular representation layer); a convolution layer with a suffix of Q indicates a quotient representation layer (e.g., 1Q is a 1-channel quotient representation layer); a convolution layer with a suffix of T indicates a trivial representation layer (e.g., 1T is a 1-channel trivial representation layer); a convolutional layer with a suffix of a number indicates a conventional convolutional layer. The convolutional layer colored in cyan is the dynamic filter layer whose weights are from the FC layer pointing to it.}
\label{fig:network}
\end{figure}

\subsection{Equivariant FCN Architecture}
\label{appendix:network_fcn}
In the Equivariant FCN architecture (Fig~\ref{fig:network_equi_fcn}), we use we use a 16-stride UNet~\cite{unet} backbone where all layers are steerable layers. The input is viewed as a trivial representation and is turned into a 16-channel regular representation feature map after the first layer. Every layer afterward in the UNet uses the regular representation, and the output of the UNet is a 16-channel regular representation feature map. This feature map is sent to a quotient representation layer to generate the pick $Q$ value maps for each $\theta$. For the place $Q$ values, the non-equivariant information from $H$ must be Incorporated. $H$ is sent to 4 conventional convolutional layers followed by 2 FC layers. The output is a vector with the same size as the number of the free weights in a 16-channel regular representation steerable layer with a kernel size of $3\times 3$. This output vector is expanded into a steerable convolutional kernel and is convolved with the output of the UNet. The result is sent to a quotient representation layer to generate the place $Q$ value maps for each $\theta$.

\subsection{Equivariant ASR Architecture}
\label{appendix:network_asr}
Fig~\ref{fig:network_equi_asr} shows the Equivariant ASR network architecture. The $q_1$ architecture is very similar to the Equivariant FCN network. Its output is a trivial representation instead of a regular representation to generate only one $Q$ map for the $x, y$ positions. The bottleneck feature map is passed through a group pooling layer (a max pooling over the group's dimension) to form $e(s)$, a state encoding that is used by $q_2$. $q_2$ uses $e(s)$ and the feature vector from $H$ to generate the weights for a steerable dynamic filter. $q_2$ processes $P$ using a set of steerable convolution layers in the regular representation, then convolves the feature map with the dynamic filter. The result of the dynamic filter layer is sent to two separate quotient representation layers to generate pick and place values for each $\theta$.

\subsection{Lift Expansion Architecture}
\label{appendix:network_exp}
\begin{figure}[t]
\centering
\includegraphics[width=0.9\linewidth]{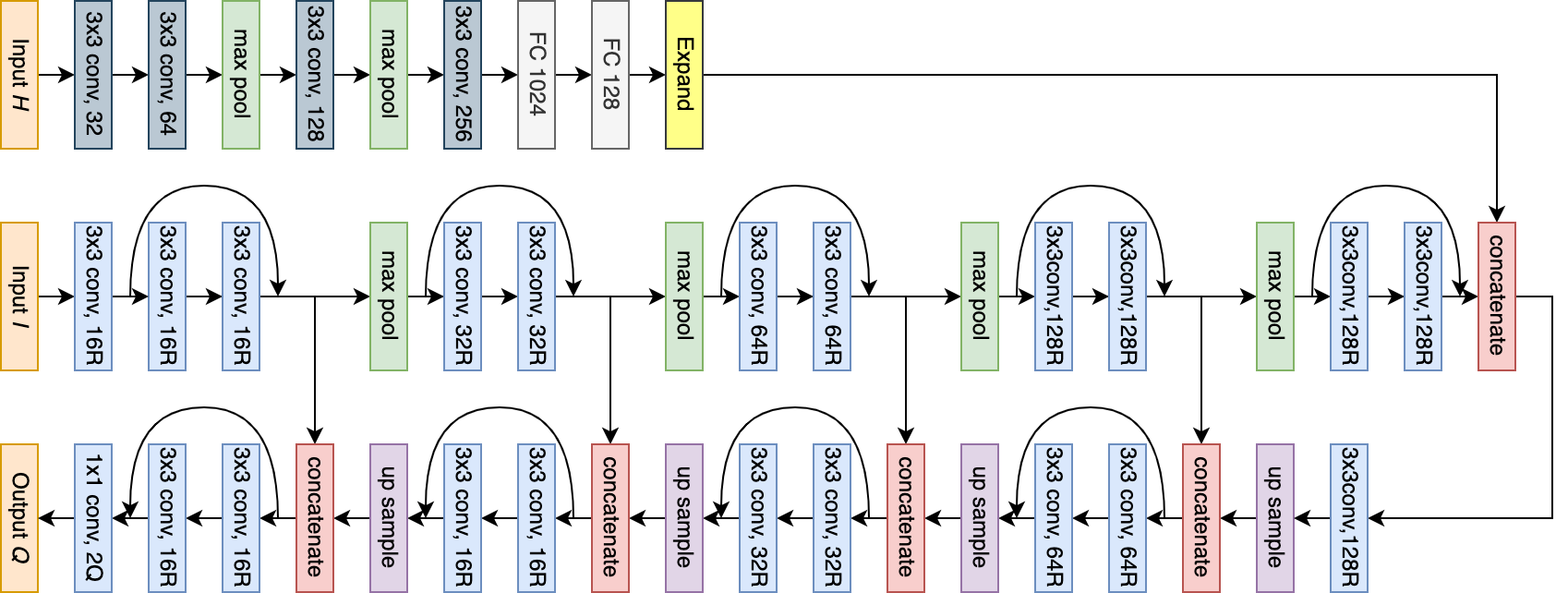}
\caption{The Equivariant FCN with Lift Expansion}
\label{fig:network_exp_fcn}
\end{figure}
Fig~\ref{fig:network_exp_fcn} shows the equivariant FCN using lift expansion for encoding the partial equivariance property. The 128-vector (the output of FC 128 in the top row) is tiled to the same size as the 128-channel regular representation feature map (the output of the rightmost convolutional layer in the middle row) and concatenated. In ASR, the same Lift Expansion network can be used in $q_1$.

\subsection{Deictic Encoding Architecture}
\label{appendix:network_deictic}
\begin{figure}[t]
\centering
\includegraphics[width=0.6\linewidth]{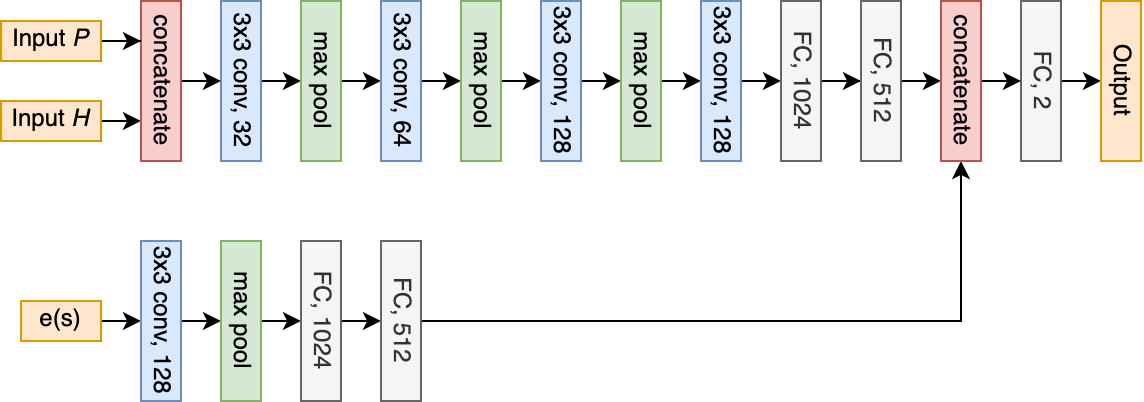}
\caption{Deictic Encoding Network Architecture}
\label{fig:network_deictic}
\end{figure}
Fig~\ref{fig:network_deictic} shows the network architecture of the Deictic Encoding network. Its output is a 2-vector, representing the values for pick and place with respect to the action (e.g., top-down rotation $\theta$ in $q_2$) encoded in the input patch $P$.

\section{Baseline Details}
\subsection{FCN Baselines}
\label{appendix:baseline_fcn}
\begin{figure}[t]
\centering
\includegraphics[width=0.8\linewidth]{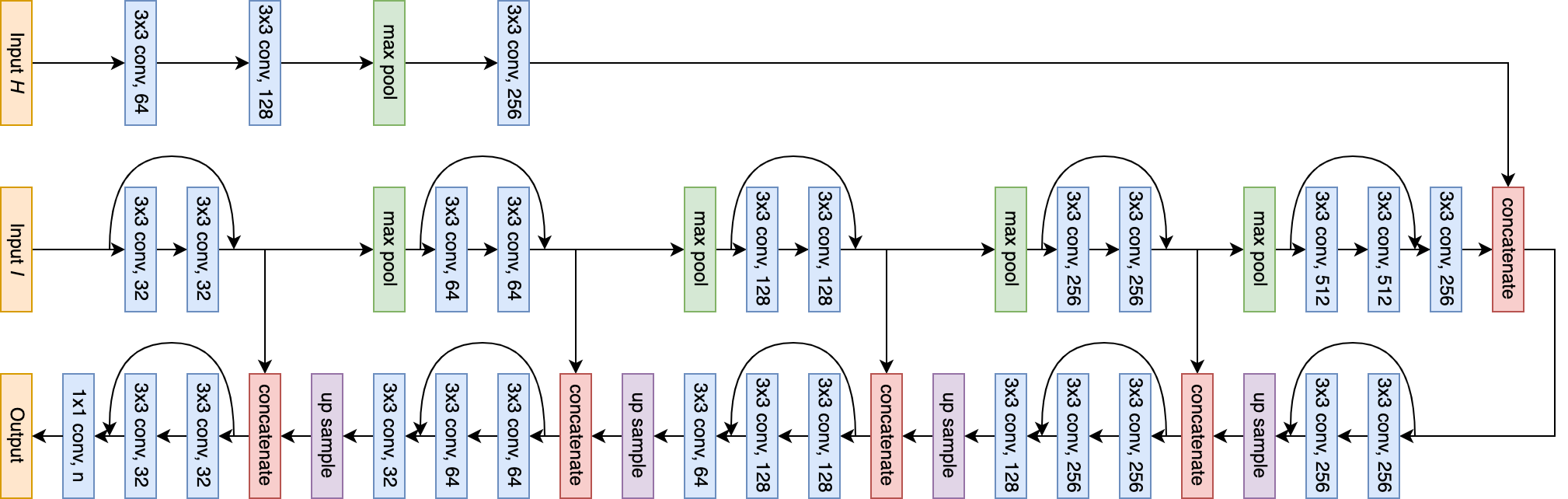}
\caption{The baseline FCN architecture}
\label{fig:network_fcn_baseline}
\end{figure}

Fig~\ref{fig:network_fcn_baseline} shows the baseline FCN architecture. For the Conventional FCN baseline, the number of output channels $n=2\times |\Theta|=12$ (i.e., one pick and one place channel for each rotation). 
\edit{The RAD~\cite{rl_with_aug} baseline uses the same baseline architecture, but during the training, each transition in the minibatch is applied with a rotational augmentation randomly sampled from $C_{12}$. The DrQ~\cite{kostrikov2020image} baseline uses the same baseline architecture, but the $Q$ targets are calculated by averaging over $K$ augmented versions of the sampled transitions; the $Q$ estimates are calculated by averaging over $M$ augmented versions of the sampled transitions. Random rotation sampled from $C_{12}$ is used for the augmentation, and we use $K=M=2$ as in~\cite{kostrikov2020image}. Note that in RAD and DrQ, since we are learning an equivariant $Q$ network instead of an invariant $Q$ network, we apply the rotational augmentation on both the state and action, rather than only augmenting the state as in the prior works.} The Rot FCN baseline uses the same network backbone, but the number of output channels $n=2$ (for pick and place, respectively). Rotations are encoded by rotating the input and output accordingly for each $\theta$ in the action space~\cite{zeng_pushing}. The Transporter baseline uses three FCNs (one for picking and two for placing) with the same FCN backbone shown in Fig~\ref{fig:network_fcn_baseline}. For placing, there are two networks with the same architecture for features (with an input of $I$) and filters (with an input of $H$), and the outputs of both are 3-channel feature maps. The correlation between them forms the 1-channel output. Rotations are encoded by rotating the input $H$ for each $\theta$ in the action space. The pick network is the same as the Rot FCN baseline.

\subsection{ASR Baselines}
\label{appendix:baseline_asr}
\begin{figure}[t]
\centering
\includegraphics[width=0.8\linewidth]{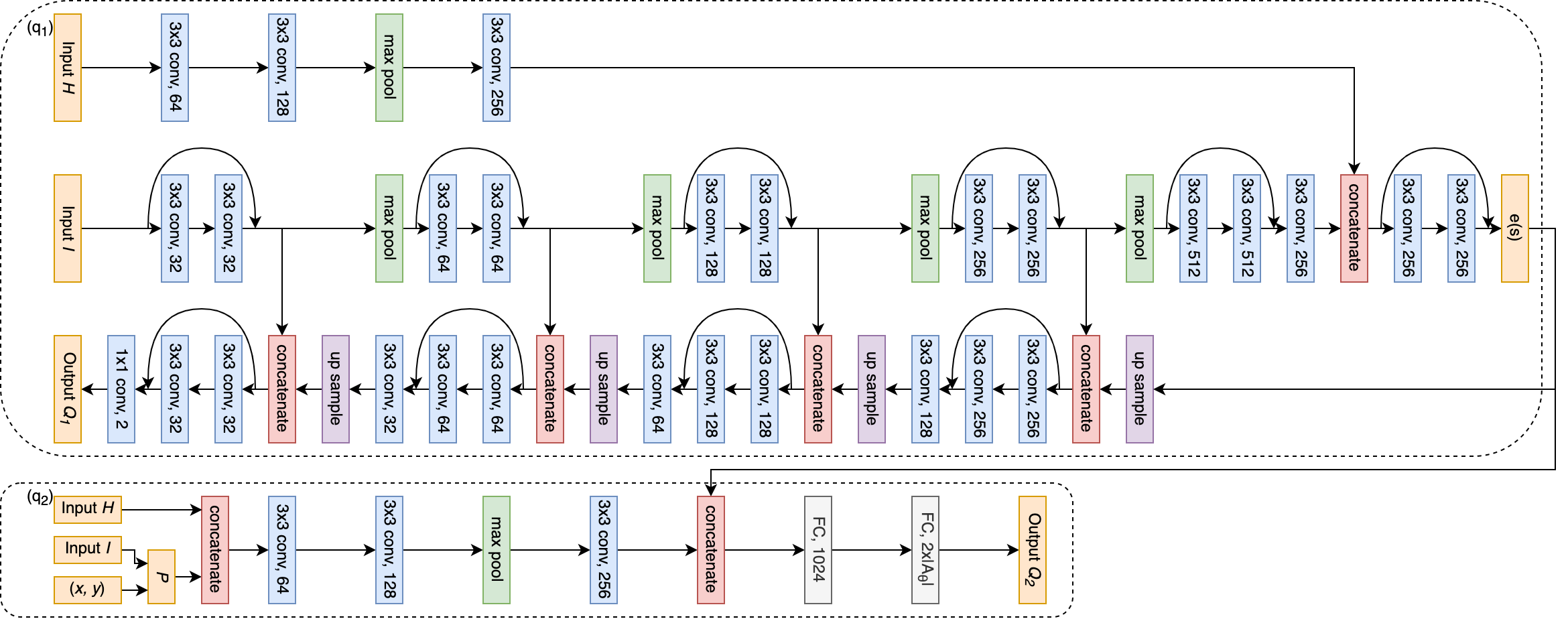}
\caption{The baseline ASR architecture}
\label{fig:network_asr_baseline}
\end{figure}
Fig~\ref{fig:network_equi_asr} shows the network architecture for the Conventional ASR baseline. \edit{The RAD~\cite{rl_with_aug} baseline uses the same baseline architecture, but during the training, each transition in the minibatch is applied with a rotational augmentation randomly sampled from $C_{32}$. The DrQ~\cite{kostrikov2020image} baseline uses the same baseline architecture, but the $Q$ targets are calculated by averaging over $K$ augmented versions of the sampled transitions; the $Q$ estimates are calculated by averaging over $M$ augmented versions of the sampled transitions. Random rotation sampled from $C_{32}$ is used for the augmentation, and we use $K=M=2$ as in~\cite{kostrikov2020image}.} The Transporter network baseline uses the same architecture as in Appendix~\ref{appendix:baseline_fcn}.

\section{Training Details}
\subsection{SDQfD}
\label{appendix:sdqfd}
SDQfD (Strict Deep $Q$ Learning from Demonstrations~\cite{asrse3}) is a variation of DQfD~\cite{dqfd} that is better suited for large action spaces. It penalizes all actions that have a $Q$ value larger than the expert action's $Q$ value minus a non-expert margin. Let $A^{s,a^e}$ be the set of actions to penalize, $A^{s,a^e}$ is defined as: 
\begin{equation}
A^{s,a^e} = \big\{ a \in A \big| Q(s,a) > Q(s, a^e)-l(a^e, a) \big\}
\end{equation}
where $l(a^e, a)=l$ if $a^e \neq a$ and $0$ otherwise. The margin loss term is defined as:
\begin{equation}
\mathcal{L}_{SLM} = \frac{1}{|A^{s,a^e}|}\sum_{a \in A^{s,a^e}} \Big[Q(s, a) + l(a^e, a) - Q(s, a^e) \Big]
\label{eqn:slm}
\end{equation}

$\mathcal{L}_{SLM}$ is combined with the TD loss $\mathcal{L}_{TD}$: $\mathcal{L}=\mathcal{L}_{TD} + w\mathcal{L}_{SLM}$ where $w$ is the weight for the margin loss. Note that $\mathcal{L}_{SLM}$ is only applied for expert transitions, while on-policy transitions only apply the TD loss term.

\subsection{Parameters}
\label{appendix:parameters}
We implement our experimental environments using the PyBullet simulator~\cite{pybullet}. The workspace has a size of $0.4m\times 0.4m$. In Section~\ref{sec:exp_equi_fcn}, $I$ covers the workspace with a size of $90\times90$ pixels, and is padded with 0 to $128\times128$ pixels (this padding is required for the Rot FCN baseline because it needs to rotate the image to encode $\theta$. To ensure a fair comparison, we apply the same padding to all methods). In Section~\ref{sec:exp_equi_asr}, $I$ covers the workspace with a size of $128\times128$ pixels. The in-hand image $H$ is a $24\times 24$ image crop centered and aligned with the previous pick in $\SE(2)$ experiments. In $\SE(3)$ experiments, $H$ is a three-channel orthographic projection image (with a size of $3\times 40\times40$) of a point cloud centered and aligned with the previous pick. The image patch $P$ has a size of $24\times 24$ in $\SE(2)$ experiments and a size of $40\times 40$ in $\SE(3)$ experiments. In $\SE(2)$ experiments, $z$ is selected by reading the height value of the area around the selected $xy$ position.

We train our models using PyTorch~\cite{pytorch} with the Adam optimizer~\cite{adam} with a learning rate of $10^{-4}$ and weight decay of $10^{-5}$. We use Huber loss~\cite{huberloss} for the TD loss and cross entropy loss for the behavior cloning loss. The discount factor $\gamma$ is $0.95$. The batch size is $16$ for SDQfD agents and $8$ for behavior cloning agents. In SDQfD, we use the prioritized replay buffer~\cite{per} with prioritized replay exponent $\alpha=0.6$ and prioritized importance sampling exponent $\beta_0=0.4$ as in~\citet{per}. The expert transitions are given a priority bonus of $\epsilon_d=1$ as in~\citet{dqfd}. The buffer has a size of 100,000 transitions. The weight $w$ for the margin loss term is $0.1$, and the margin $l=0.1$.

\section{Ablation Studies}
\subsection{\edit{RAD and DrQ with more data augmentation operators}}
\label{appendix:exp_rad_drq}
\begin{figure}[t]
\centering
\subfloat[Block Stacking]{\includegraphics[width=0.25\linewidth]{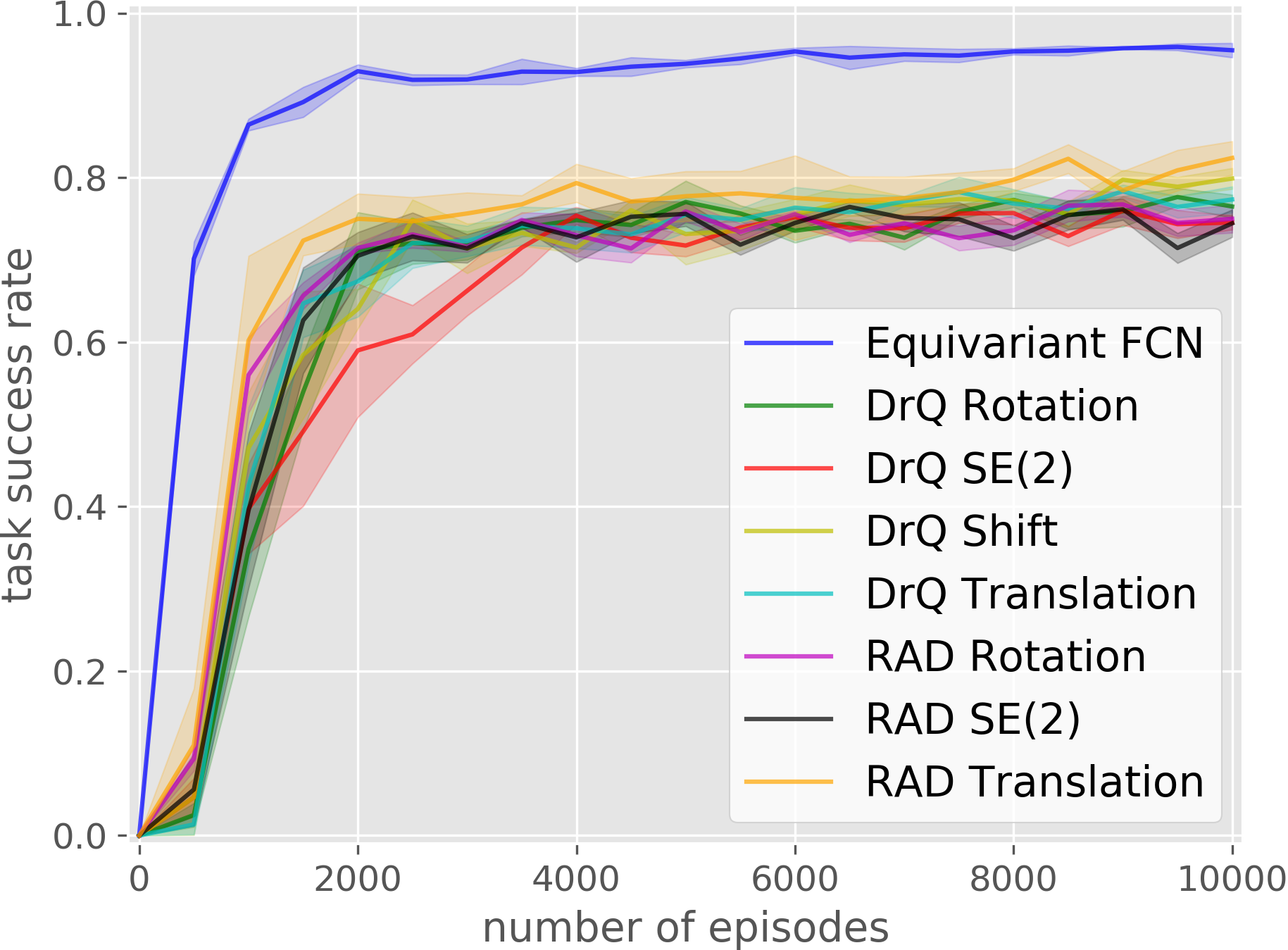}}
\subfloat[Bottle Arrangement]{\includegraphics[width=0.25\linewidth]{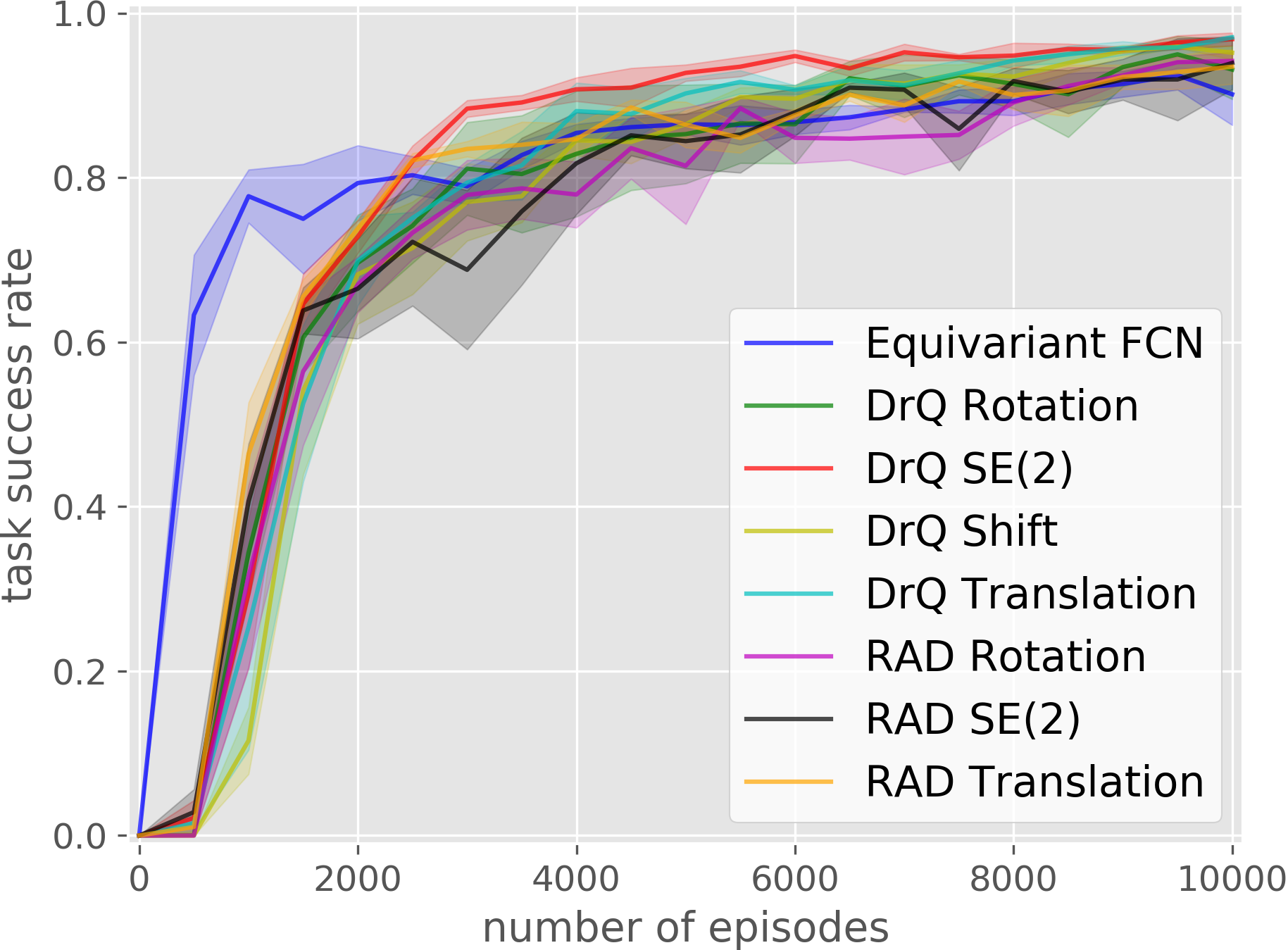}}\\
\subfloat[House Building]{\includegraphics[width=0.25\linewidth]{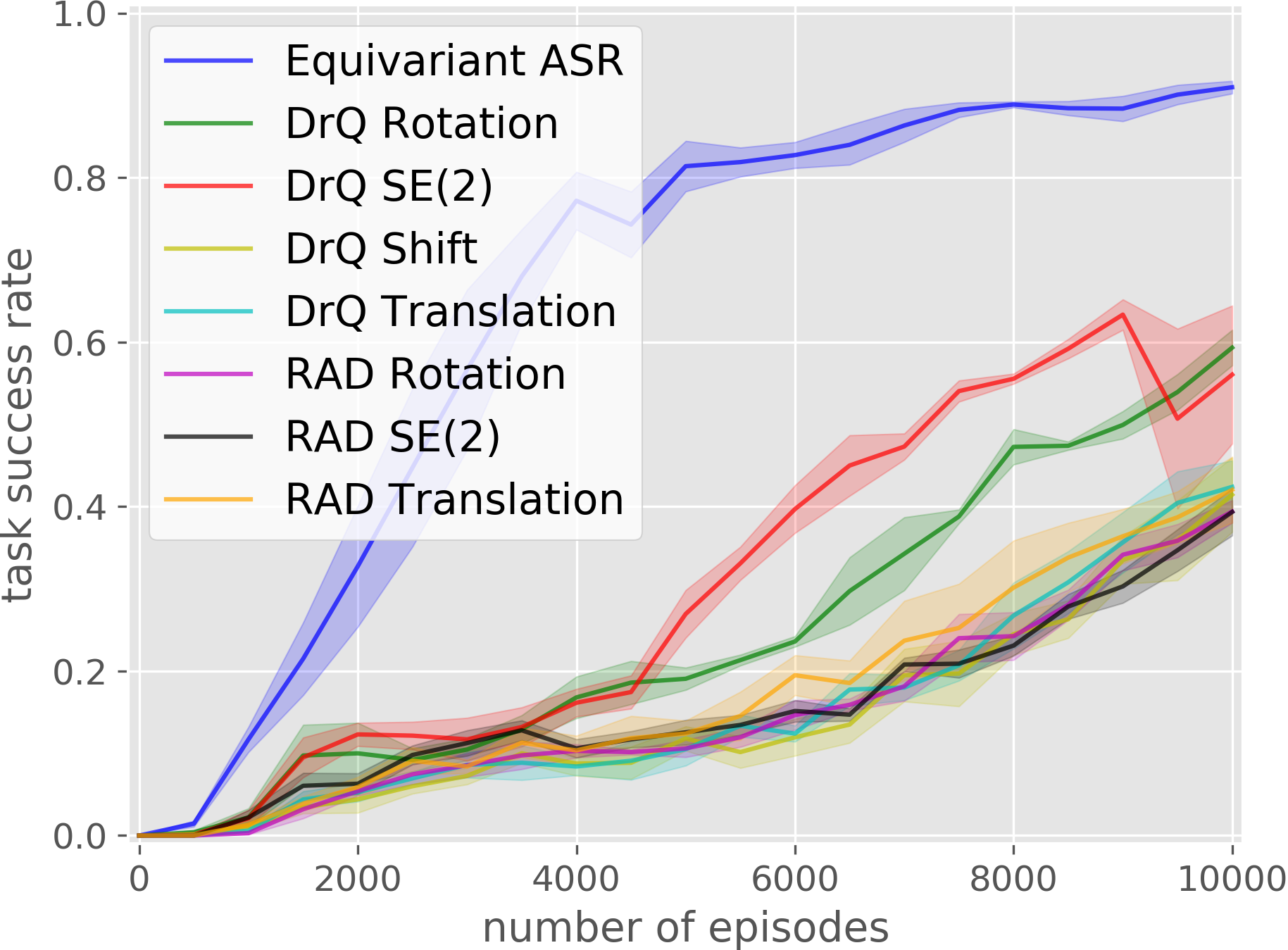}}
\subfloat[Covid Test]{\includegraphics[width=0.25\linewidth]{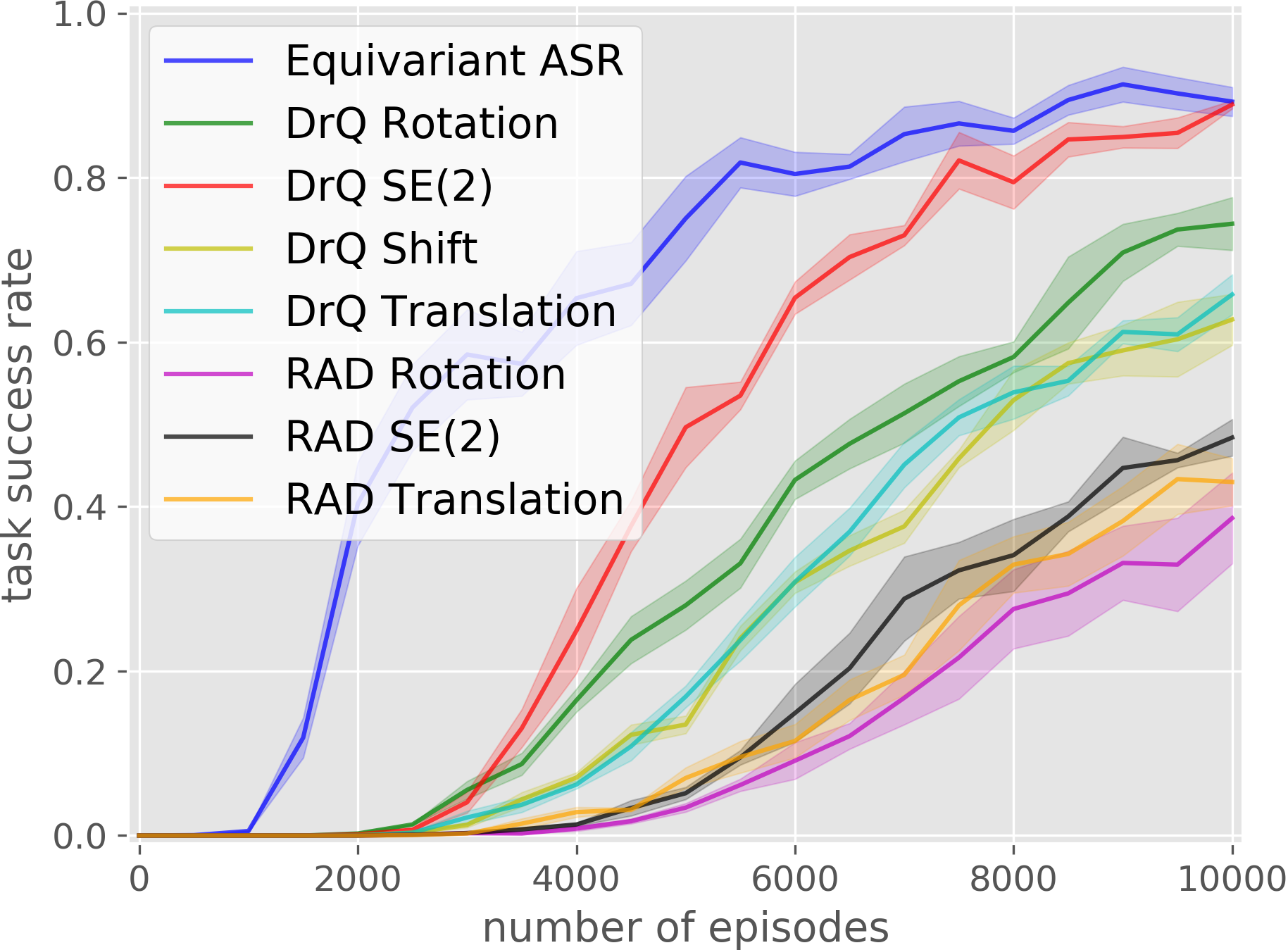}}
\subfloat[Box Palletizing]{\includegraphics[width=0.25\linewidth]{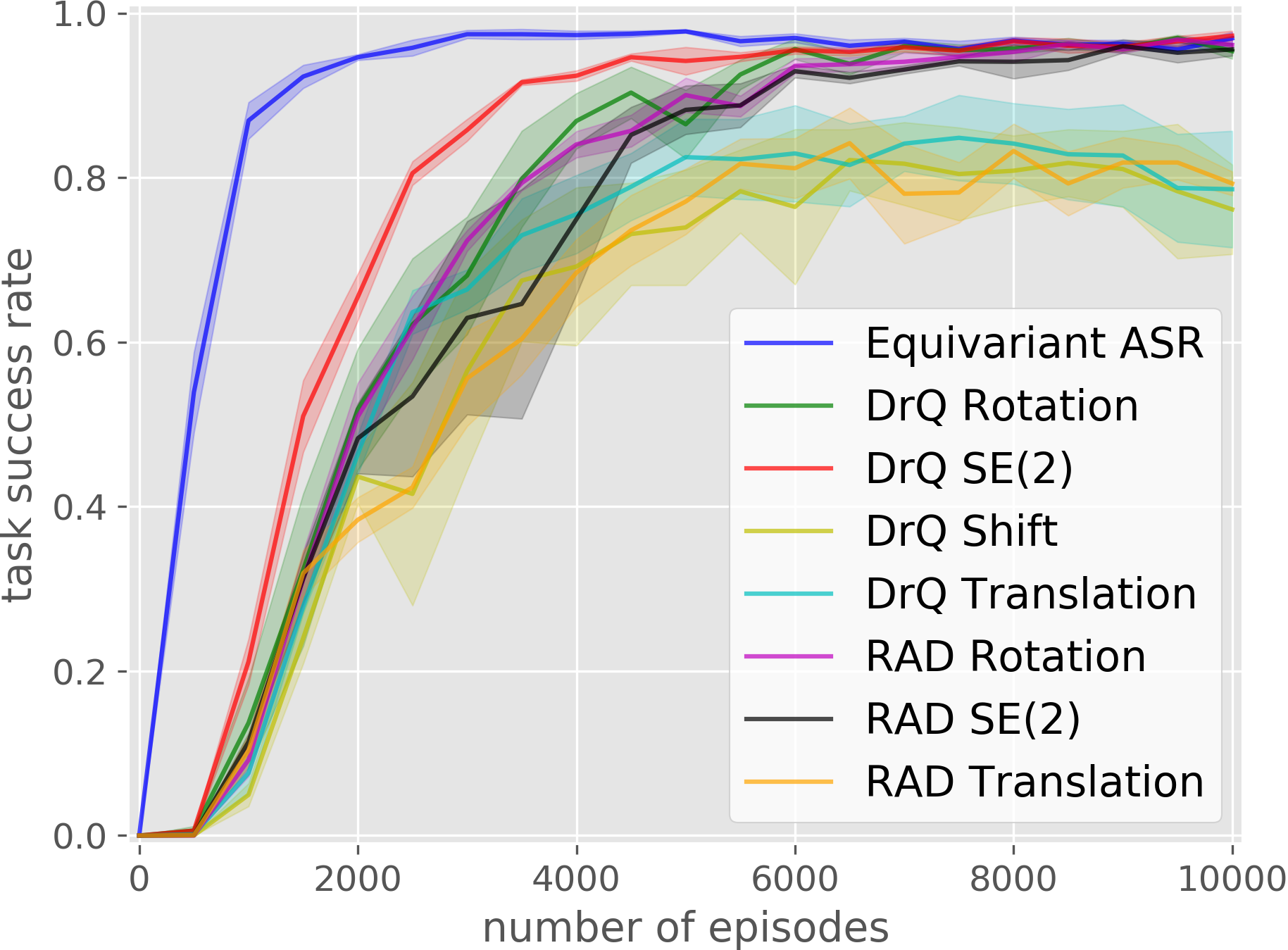}}
\subfloat[Bin Packing]{\includegraphics[width=0.25\linewidth]{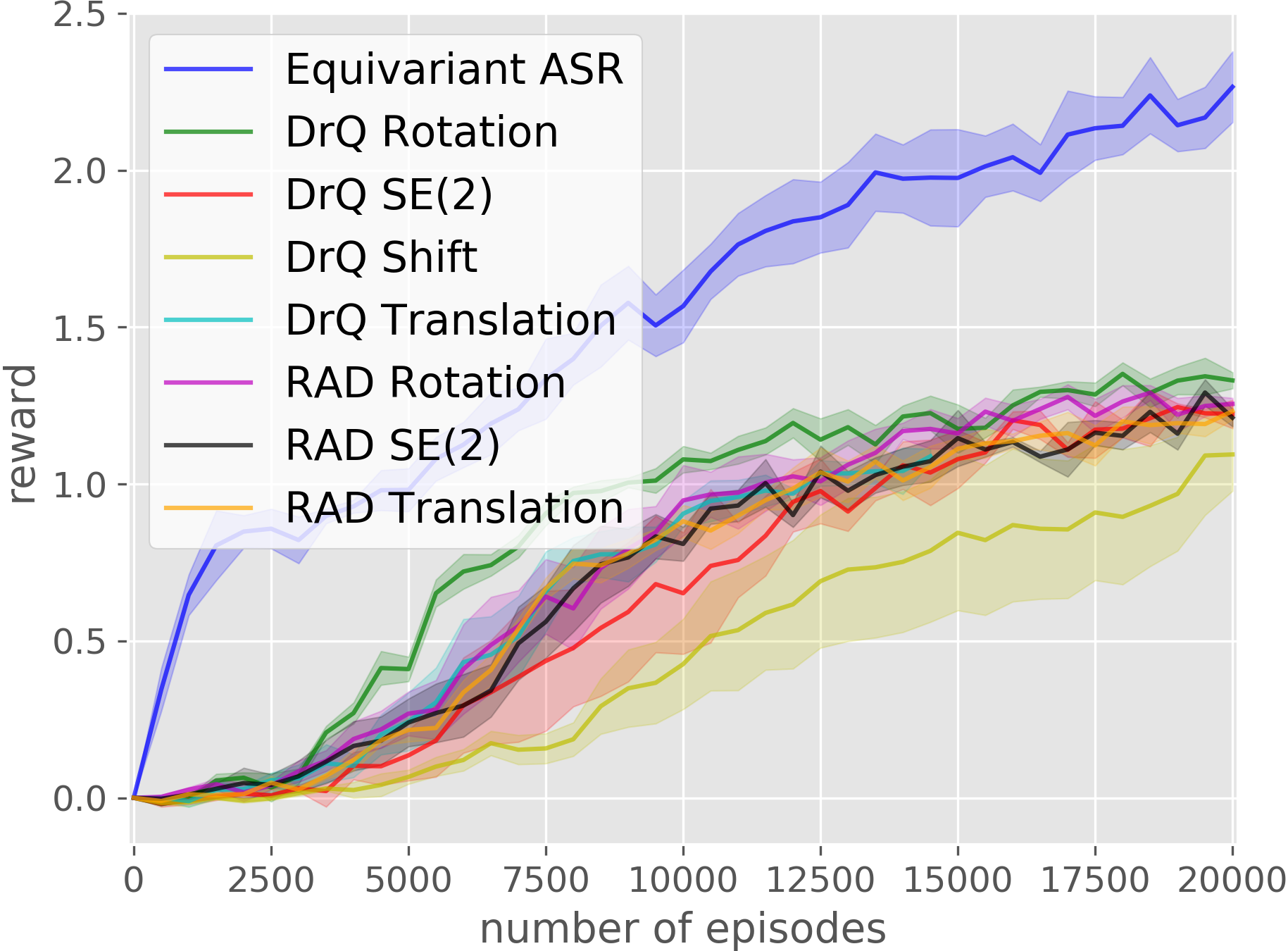}}
\caption{\edit{Comparison against RAD and DrQ with more data augmentation operators in equivariant FCN (a-b) and equivariant ASR (c-f). Results averaged over four runs. Shading denotes standard error.}}
\label{fig:exp_rad_drq}
\end{figure}

\edit{In Section~\ref{sec:exp_equi_fcn} and Section~\ref{sec:exp_equi_asr}, we compare the equivariant architectures with RAD~\cite{rl_with_aug} and DrQ~\cite{kostrikov2020image} with rotational data augmentation. In this experiment, we run the comparison with more data augmentation operators: 1) Rotation: random rotation in $C_{12}$ and $C_{32}$, same as in Section~\ref{sec:exp_equi_fcn} and Section~\ref{sec:exp_equi_asr}. 2) Translation: random translation. 3) SE(2): the combination of 1) and 2). 4) Shift: random shift of $\pm$ 4 pixels as in~\cite{kostrikov2020image}. Note that only 1) is a fair comparison because our equivariant models do not inject extra translational knowledge into the network. Even though, the equivariant networks outperforms all data augmentation methods in five out of the six environments.}

\subsection{Dynamic Filter vs Lift Expansion}
\label{appendix:exp_df_vs_exp}
\begin{figure}[t]
\centering
\subfloat[Block Stacking]{\includegraphics[width=0.25\linewidth]{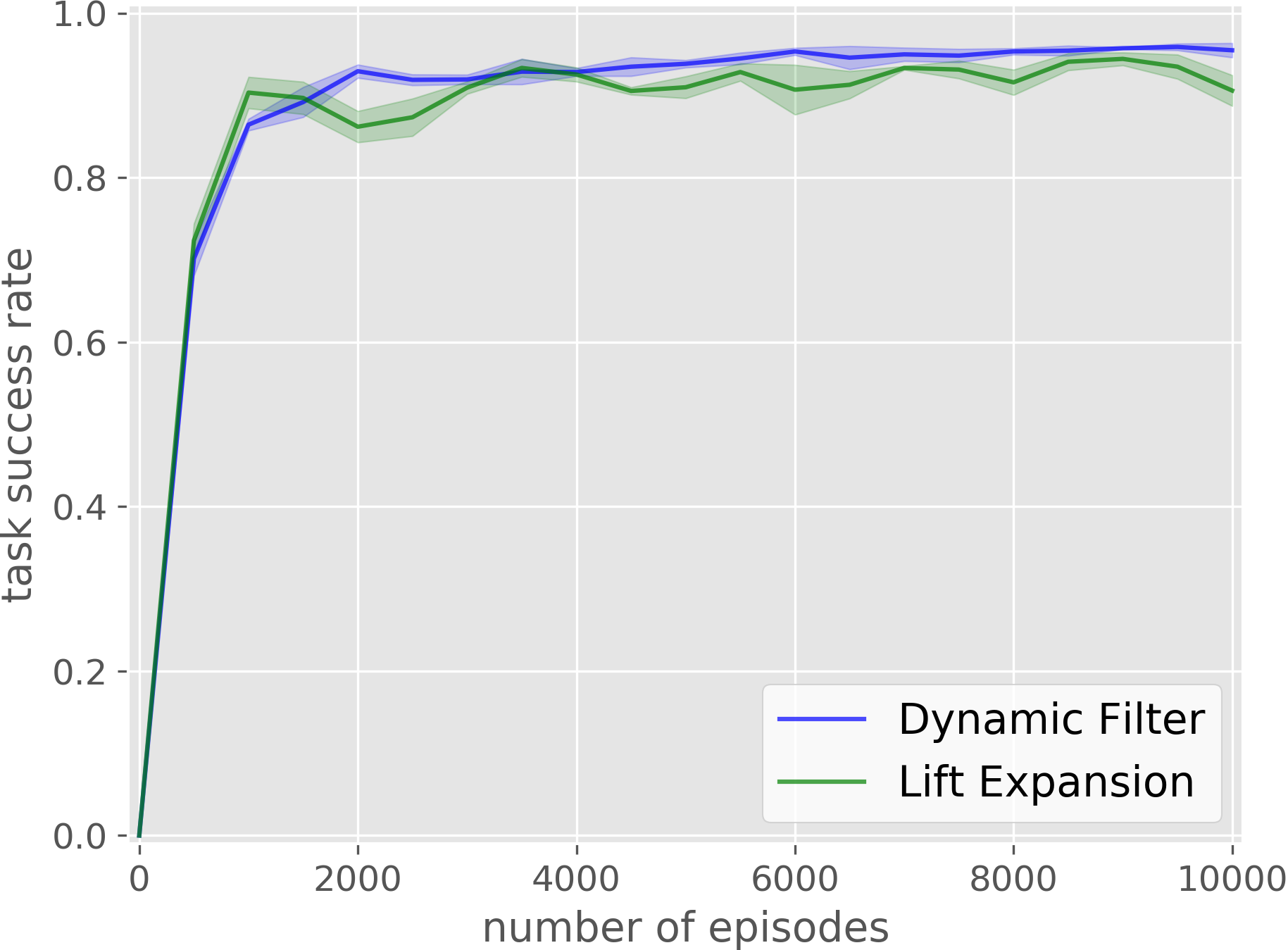}}
\subfloat[Bottle Arrangement]{\includegraphics[width=0.25\linewidth]{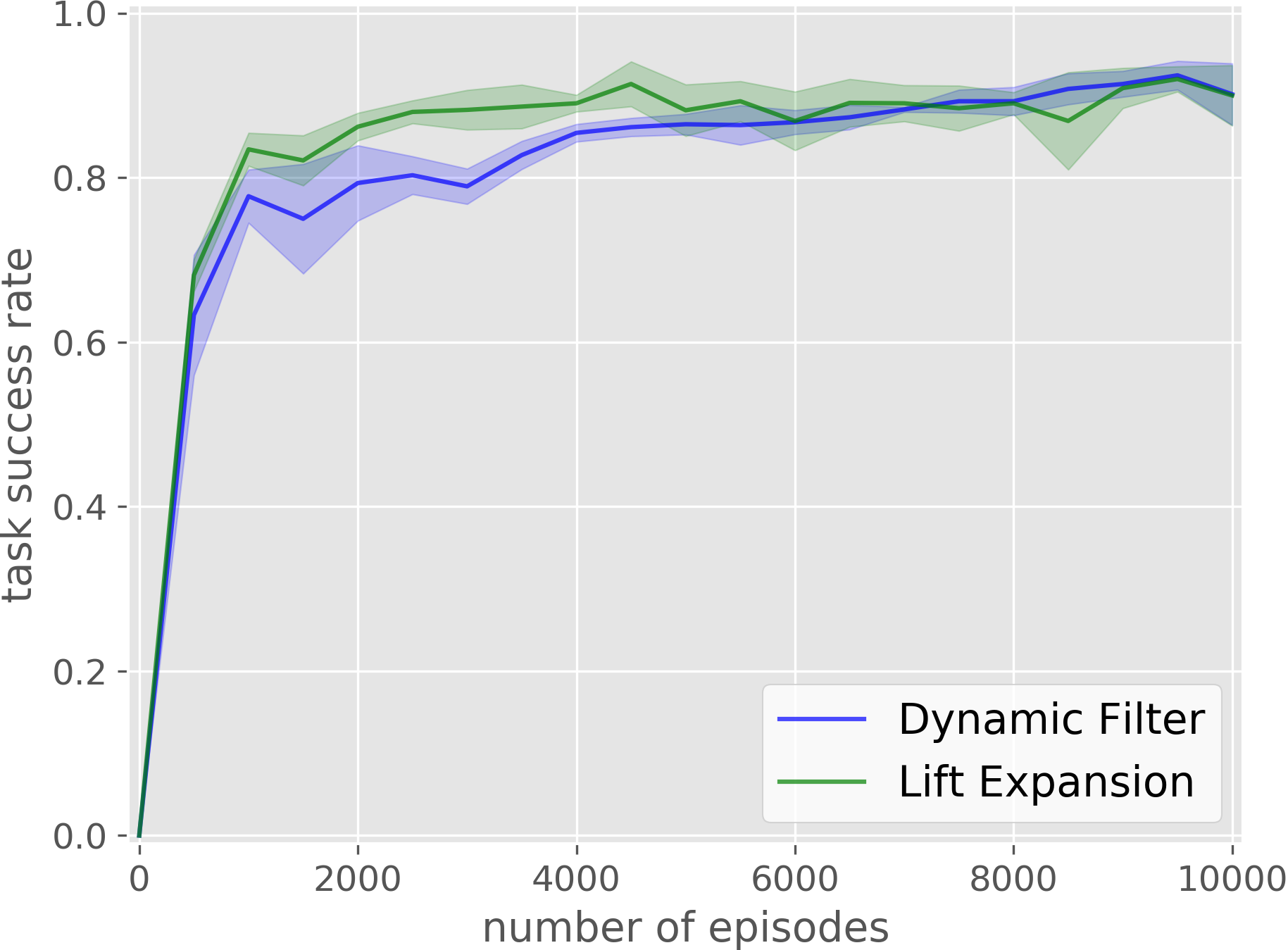}}\\
\subfloat[House Building]{\includegraphics[width=0.25\linewidth]{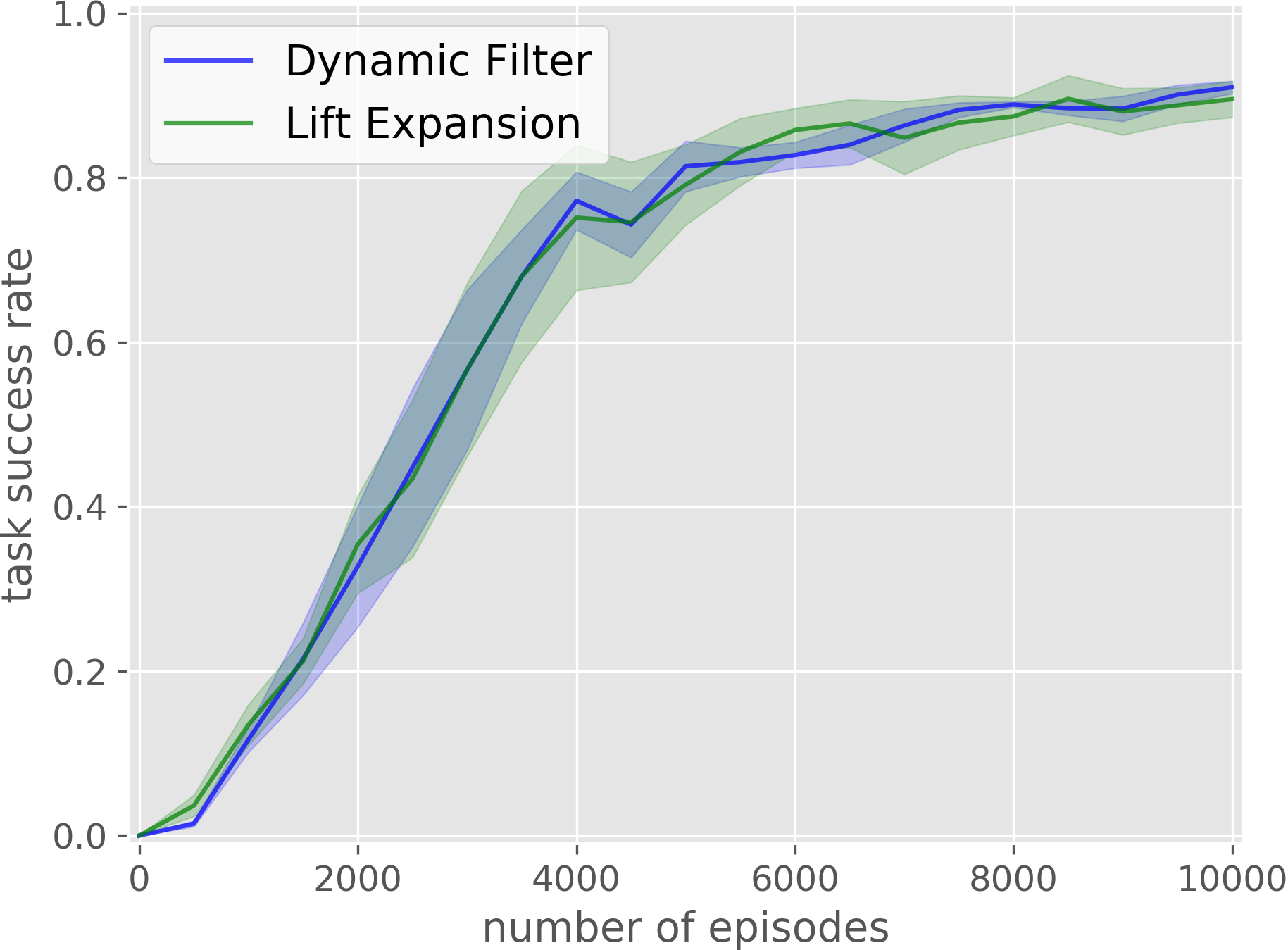}}
\subfloat[Covid Test]{\includegraphics[width=0.25\linewidth]{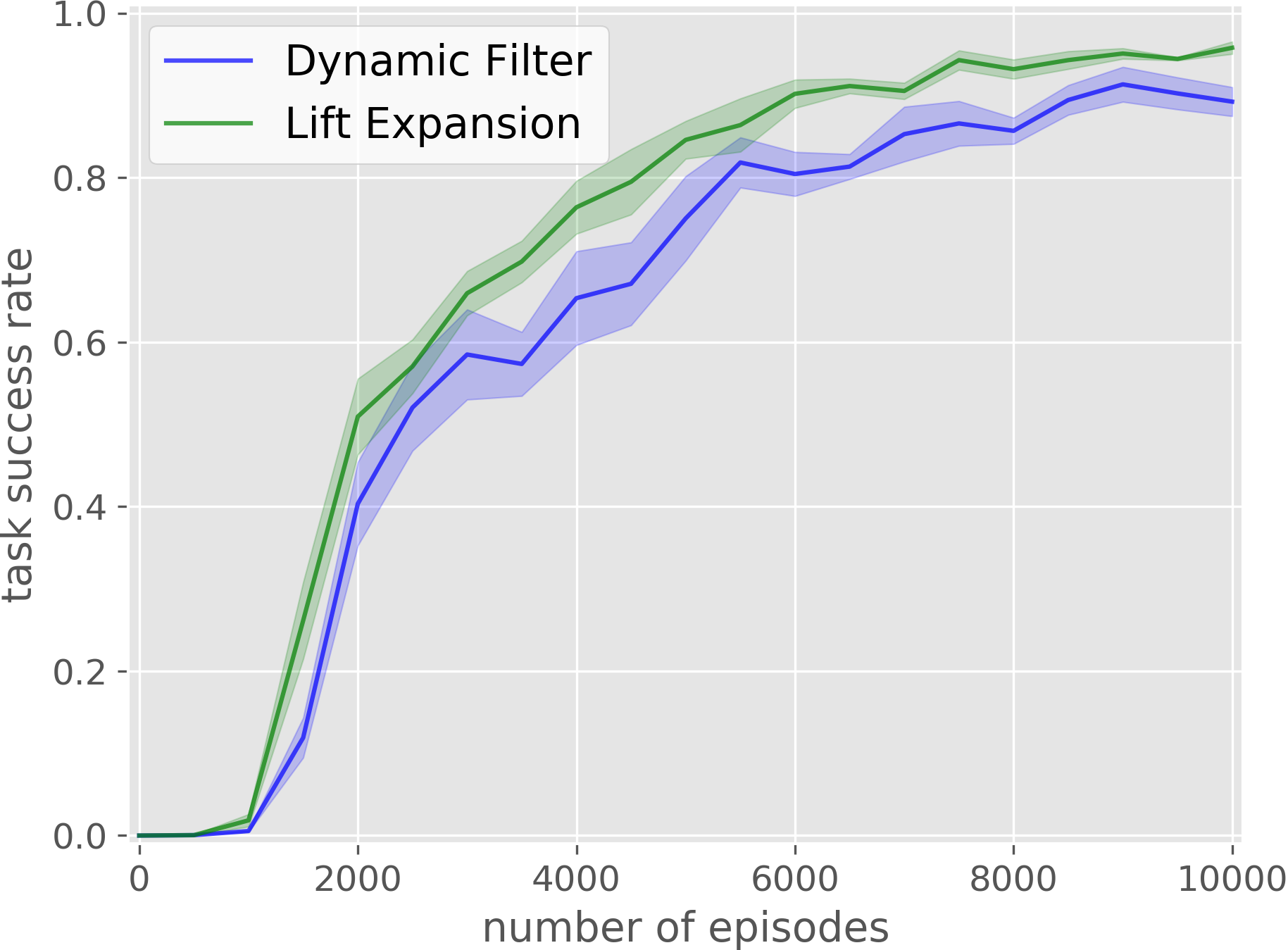}}
\subfloat[Box Palletizing]{\includegraphics[width=0.25\linewidth]{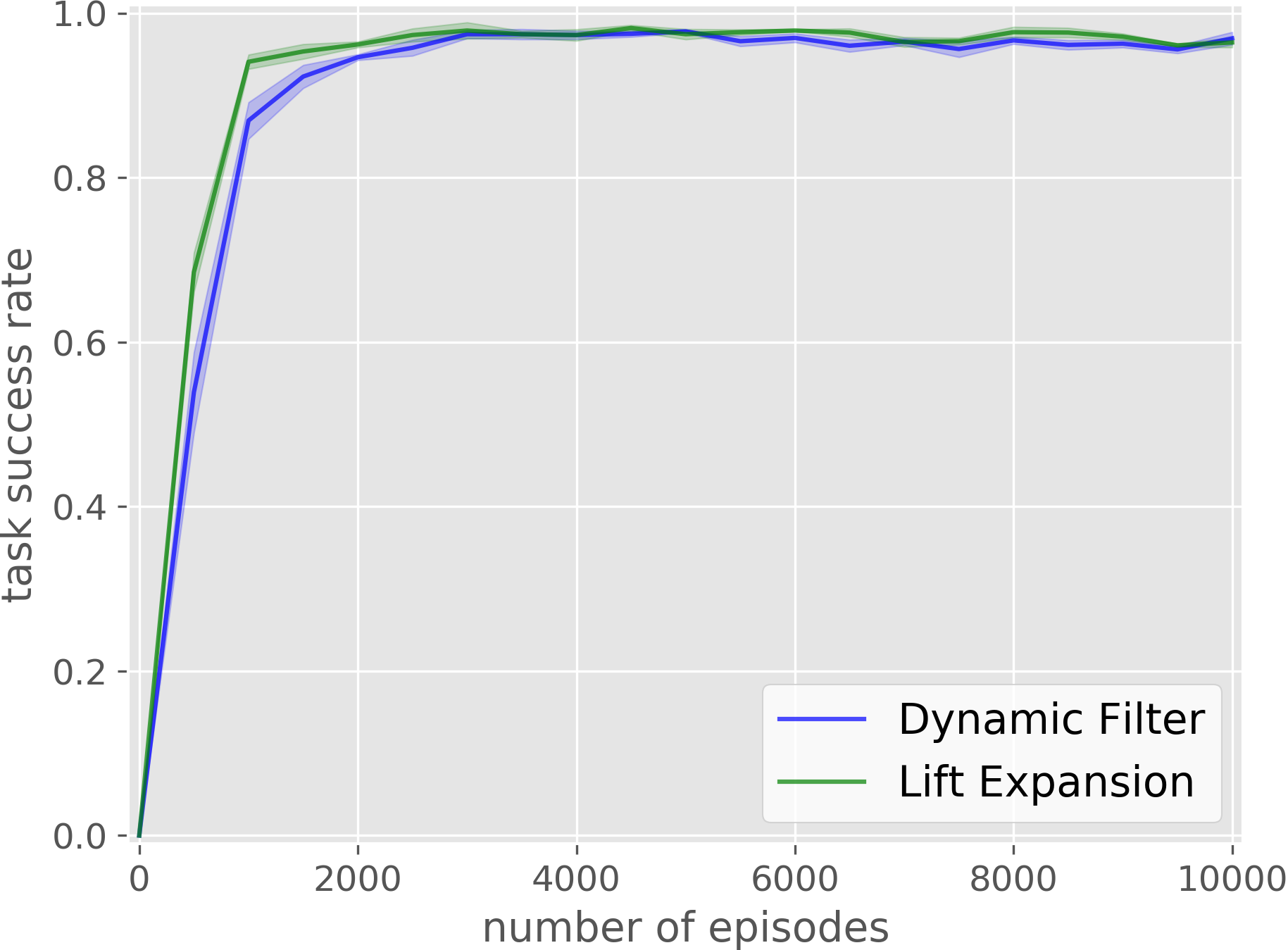}}
\subfloat[Bin Packing]{\includegraphics[width=0.25\linewidth]{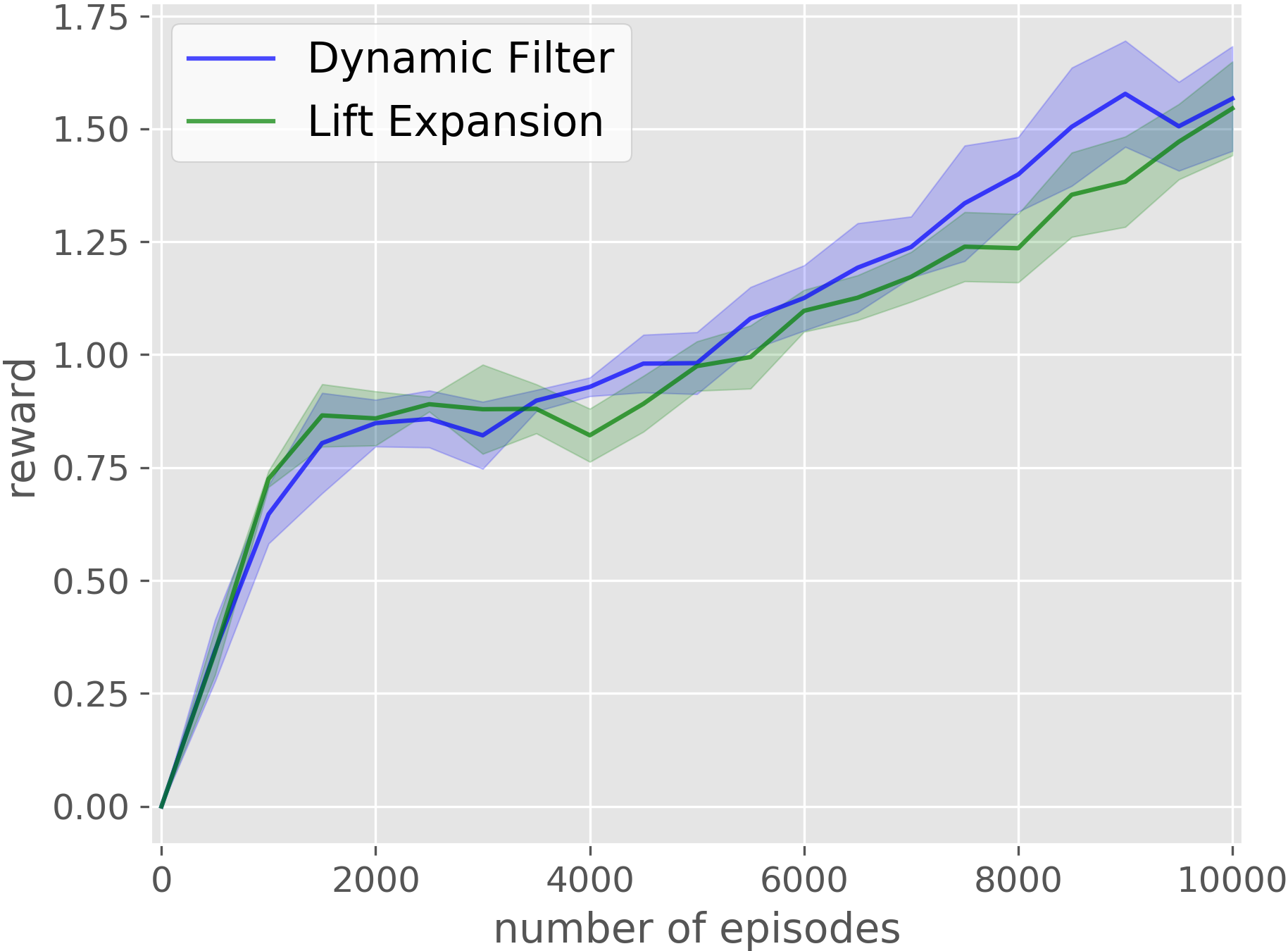}}
\caption{Comparison between Dynamic Filter and Lift Expansion in equivariant FCN (a-b) and equivariant ASR (c-f). Results averaged over four runs. Shading denotes standard error.}
\label{fig:exp_df_vs_exp}
\end{figure}


In this experiment, we compare the Dynamic Filter and Lift Expansion methods for encoding partial equivariance property. We evaluated both the equivariant FCN architecture and the equivariant ASR architecture (note that we only test this variation in $q_1$. $q_2$ uses the Dynamic Filter regardless of the architecture of $q_1$). The results are shown in Fig~\ref{fig:exp_df_vs_exp}. Both methods generally perform equally well.

\subsection{Equivariant Network in Behavior Cloning}
\begin{table}[t]
\centering
\begin{tabular}{|c|c|c|}
\hline
Environment & Block Stacking & Bottle Arrangement \\
\hline
Equivariant FCN & \textbf{0.881} & \textbf{0.781} \\
\hline
Transporter & 0.804 & 0.663 \\
\hline
\end{tabular}
\caption{Comparison between equivariant FCN and Transporter network. Results averaged over four runs.\label{tab:exp_fcn_bc}}
\end{table}


In this experiment, we evaluate the performance of our equivariant network in a behavior cloning setting compared with the Transporter network~\cite{transporter}. Both methods use the same cross entropy loss function and the same data augmentation strategy. The experimental parameters mirror Section~\ref{sec:exp_equi_fcn}. The results are shown in Table~\ref{tab:exp_fcn_bc}. The equivariant network outperforms the Transporter network in both environments.

\subsection{Equivariant ASR Ablations}
\subsubsection{Only Using Equivariant Network in $q_1$ or $q_2$}
\begin{figure}[t]
\centering
\subfloat[House Building]{\includegraphics[width=0.25\linewidth]{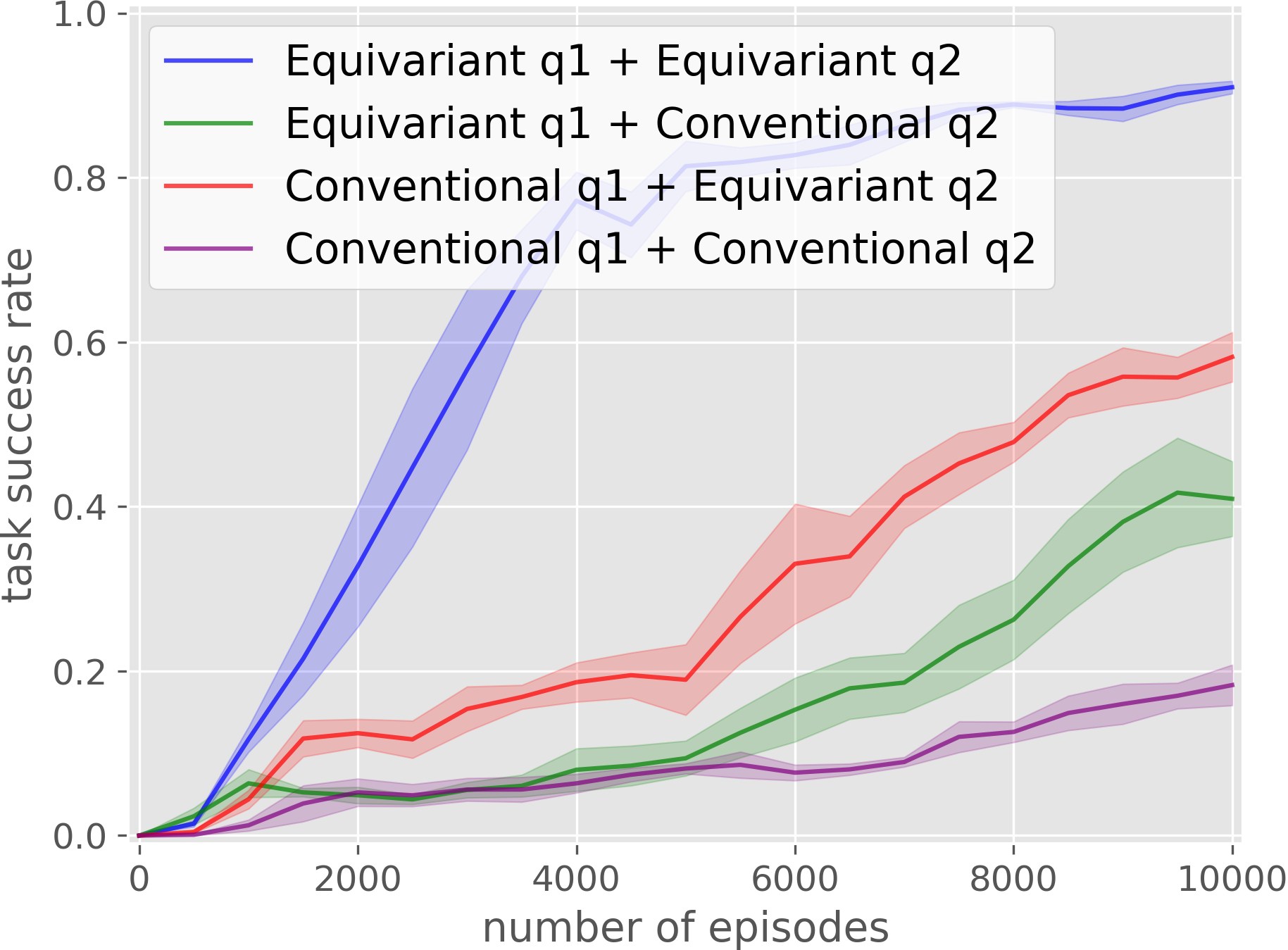}}
\subfloat[Covid Test]{\includegraphics[width=0.25\linewidth]{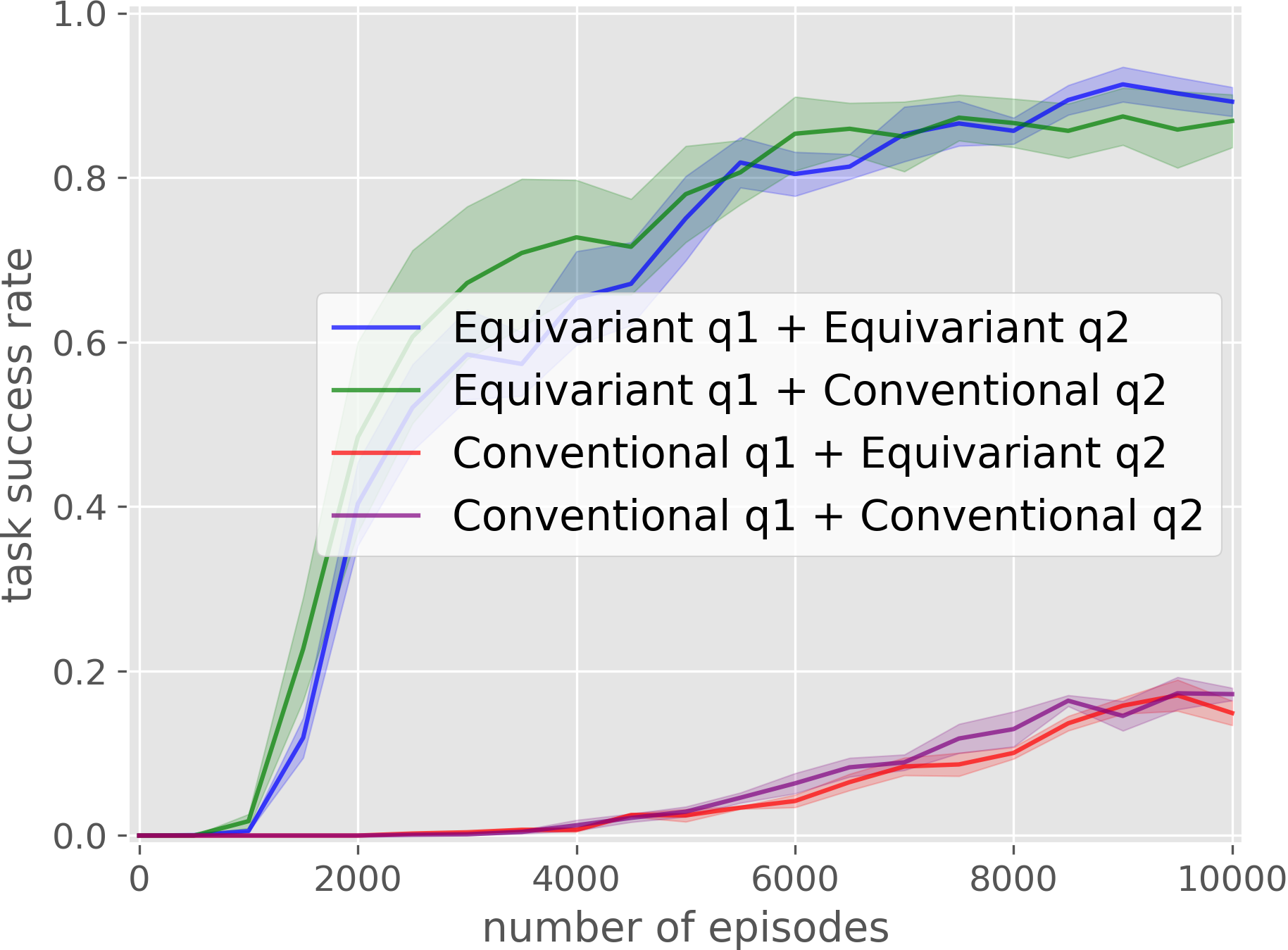}}
\subfloat[Box Palletizing]{\includegraphics[width=0.25\linewidth]{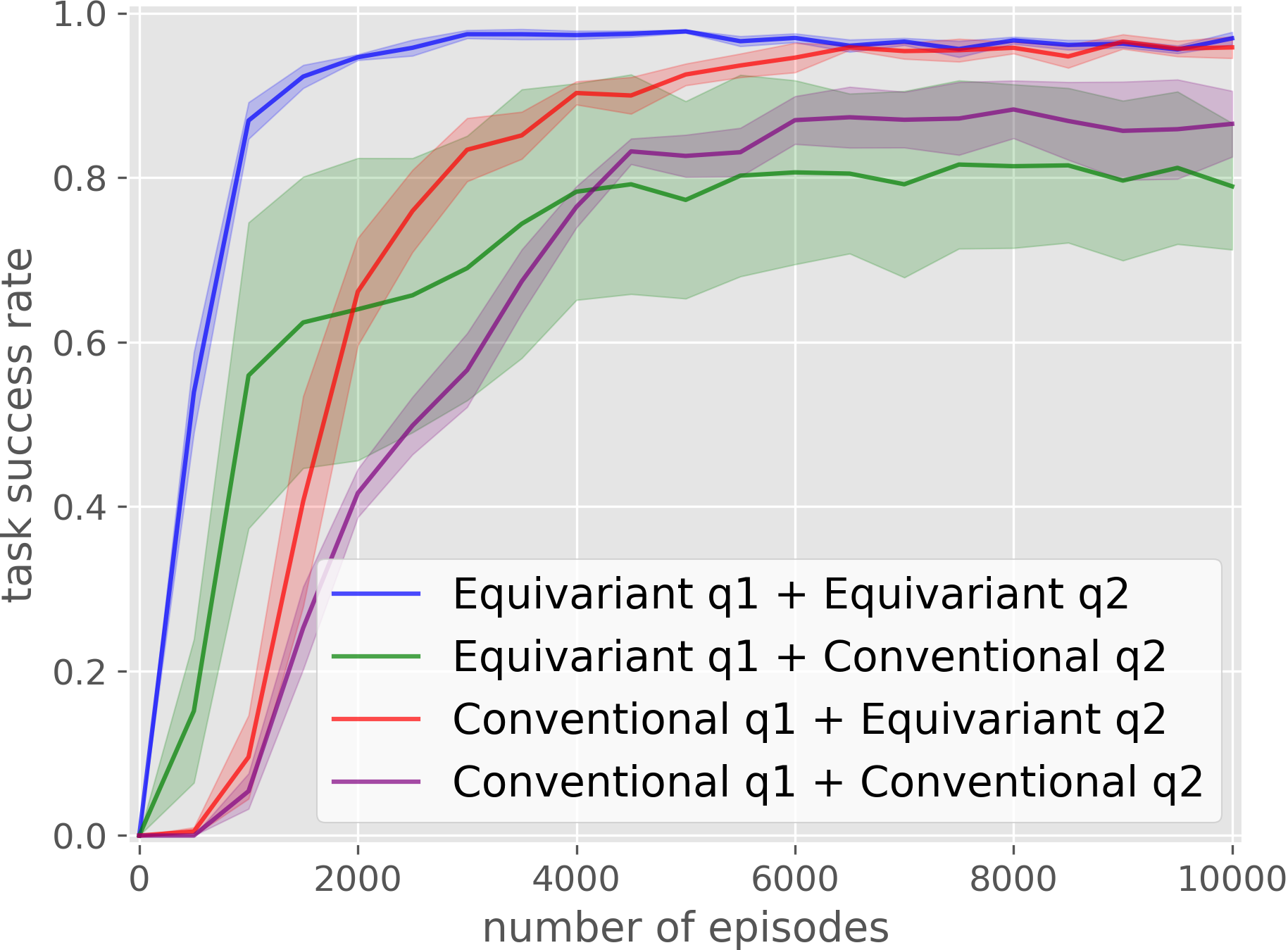}}
\subfloat[Bin Packing]{\includegraphics[width=0.25\linewidth]{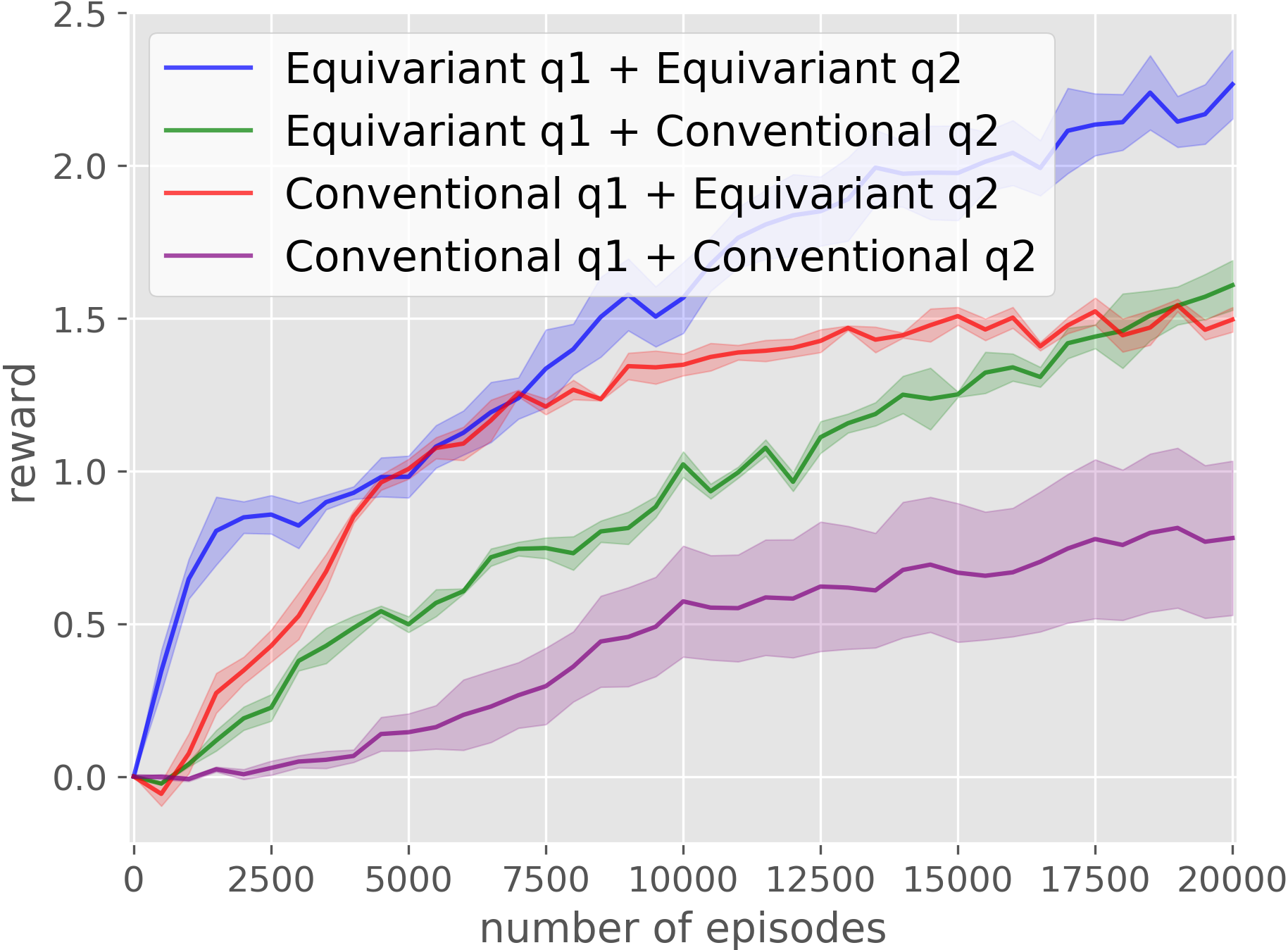}}
\caption{Comparison of four variations of equivariant/conventional and $q_1$/$q_2$ combinations. Results averaged over four runs. Shading denotes standard error.}
\label{fig:exp_asr_q1q2}
\end{figure}

In this ablation study, we evaluate the effect of the equivariant network by only applying it in $q_1$ or $q_2$. There are four variations: 1) Equivariant $q_1$ + Equivariant $q_2$: both $q_1$ and $q_2$ use the equivariant network; 2) Equivariant $q_1$ + Conventional $q_2$: $q_1$ uses the equivariant network, $q_2$ uses the conventional convolutional network; 3) Conventional $q_1$ + Equivariant $q_2$: $q_1$ uses the conventional convolutional network, $q_2$ uses the equivariant network; 4) Conventional $q_1$ + Conventional $q_2$: both $q_1$ and $q_2$ use the conventional convolutional network. The results are shown in Fig~\ref{fig:exp_asr_q1q2}, where Using the equivariant network in both $q_1$ and $q_2$ (blue) always shows the best performance. Note that only applying the equivariant network in $q_2$ (red) demonstrates a greater improvement compared with only applying the equivariant network in $q_1$ (green) in three out of four environments. This is because $q_2$ is responsible for providing the TD target for both $q_1$ and $q_2$~\cite{asrse3}, which raises its importance in the whole system.

\subsubsection{Symmetry Group in $q_1$}
\begin{figure}[t]
\centering
\subfloat[House Building]{\includegraphics[width=0.25\linewidth]{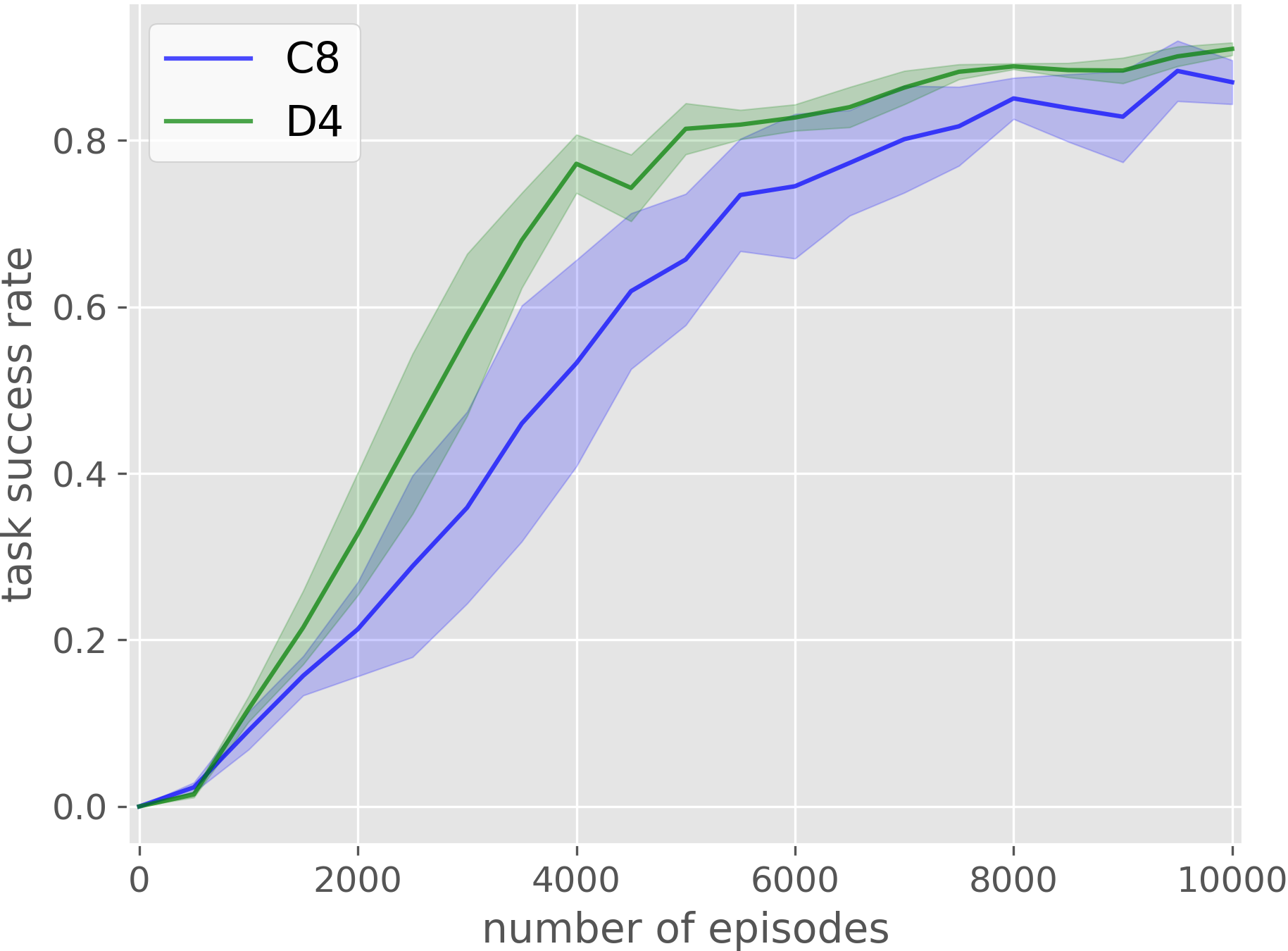}}
\subfloat[Covid Test]{\includegraphics[width=0.25\linewidth]{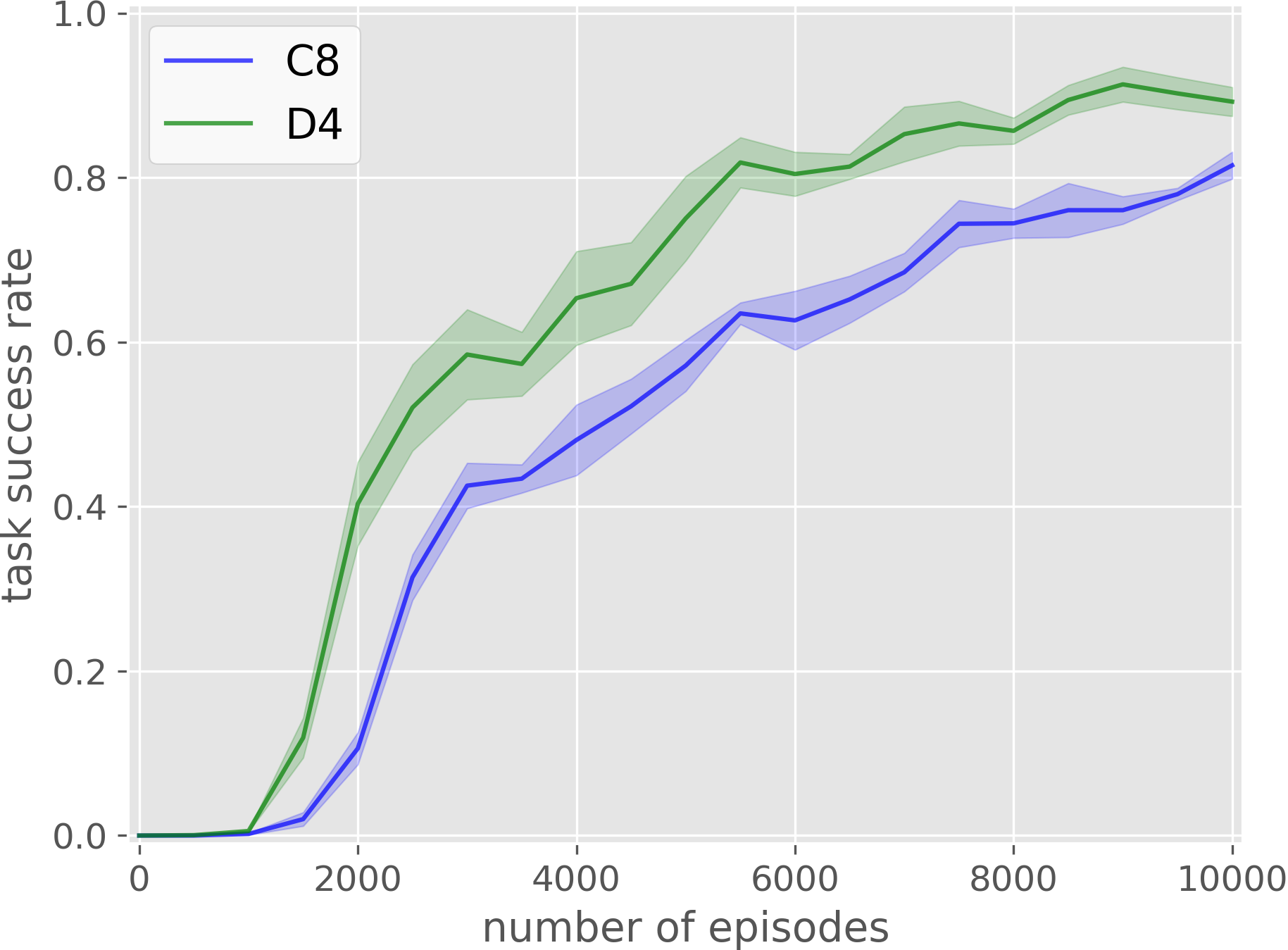}}
\subfloat[Box Palletizing]{\includegraphics[width=0.25\linewidth]{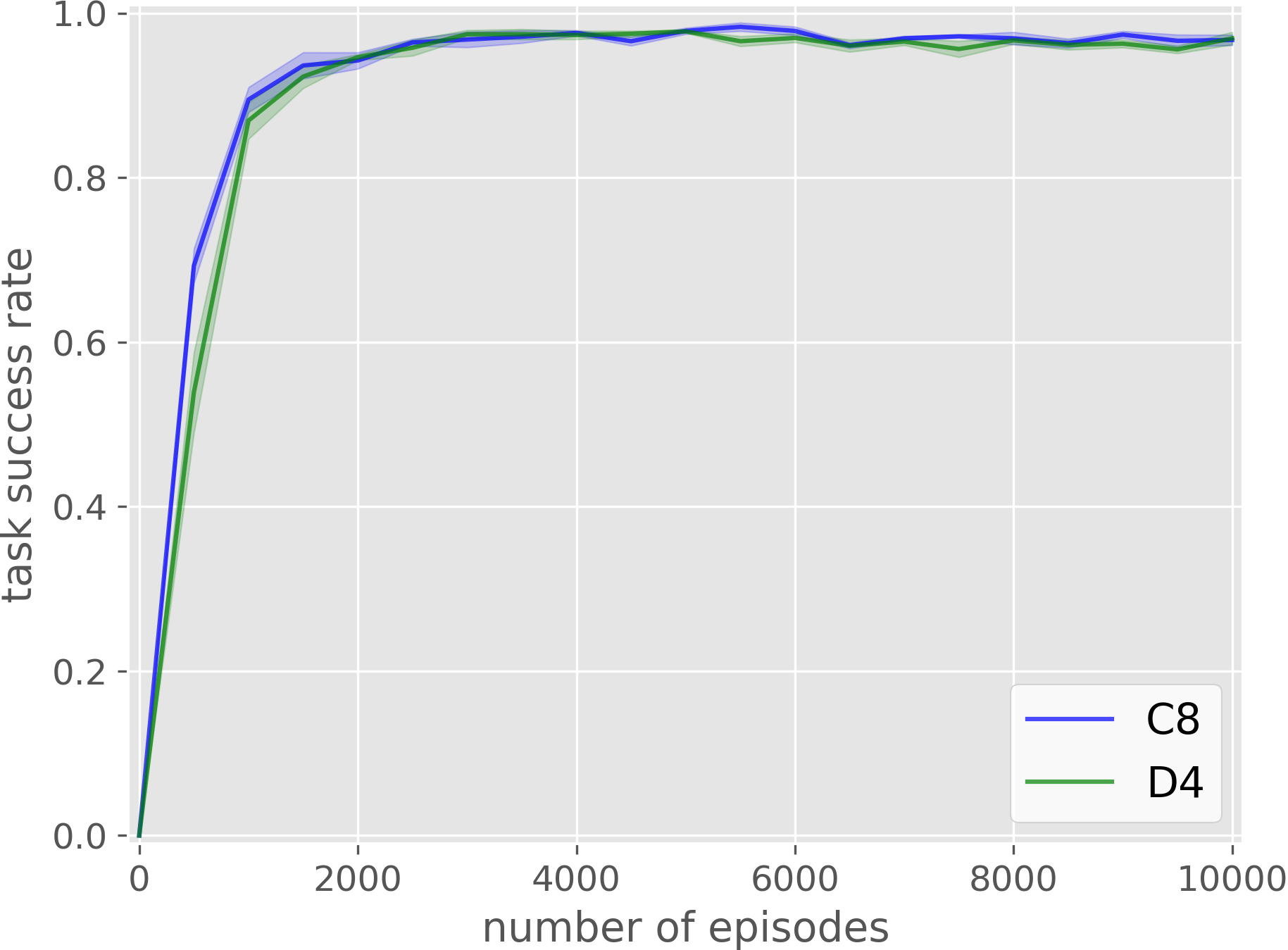}}
\subfloat[Bin Packing]{\includegraphics[width=0.25\linewidth]{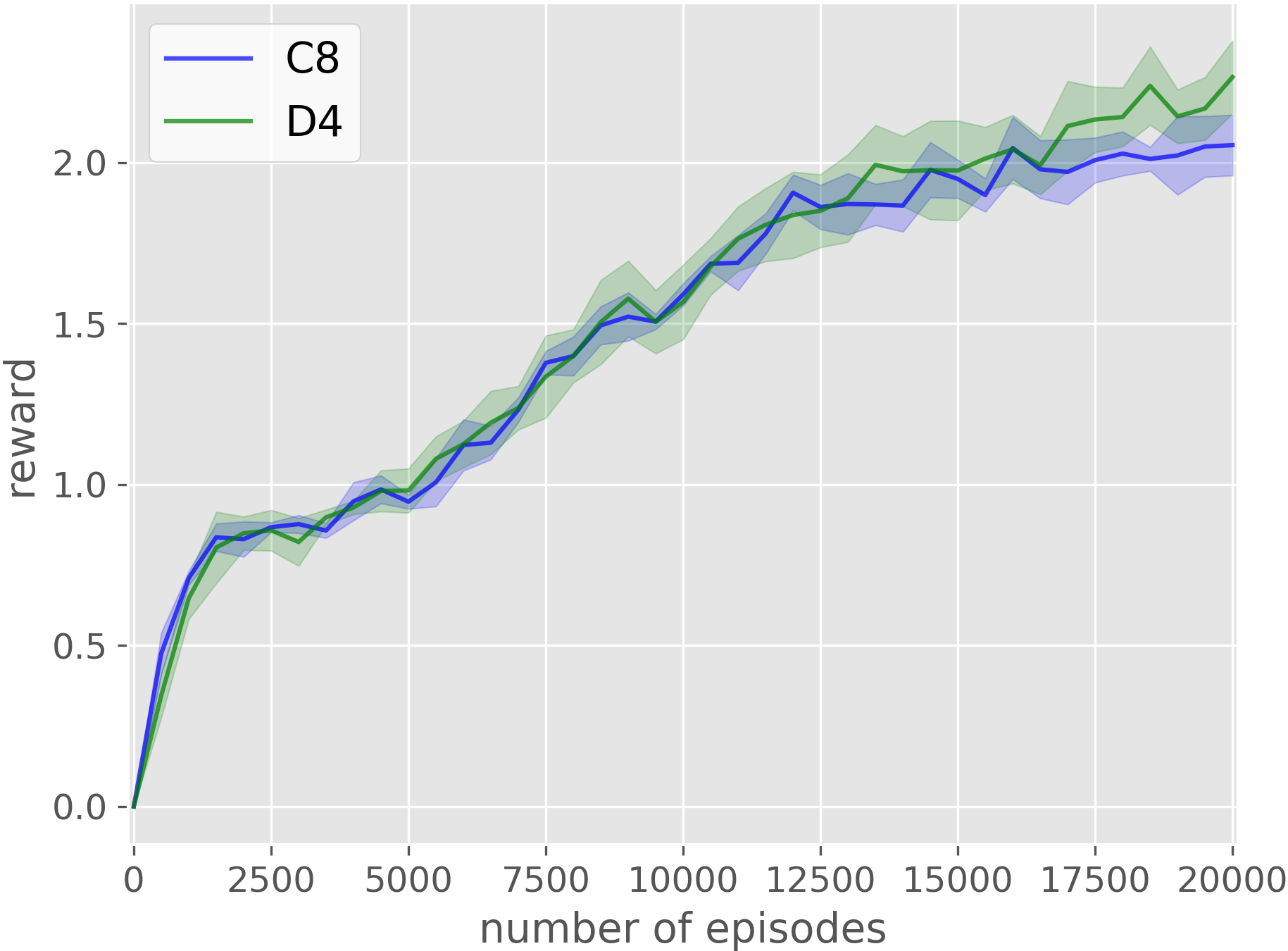}}
\caption{Comparison of two different symmetry groups for $q_1$. Results averaged over four runs. Shading denotes standard error.}
\label{fig:exp_asr_D4_vs_C8}
\end{figure}
In this experiment, we evaluate two different symmetry groups that $q_1$ can be defined upon: the Cyclic group $C_8$ that encodes eight rotations every 45 degrees, and the Dihedral group $D_4$ that encodes four rotations every 90 degrees and reflection. Both groups have an order of 8, i.e., the network will be equally heavy. As is shown in Fig~\ref{fig:exp_asr_D4_vs_C8}, $D_4$ has a minor advantage over $C_8$.

\subsubsection{Deictic Encoding in $\SE(2)$}
\label{appendix:exp_deictic_se2}
\begin{figure}[t]
\centering
\subfloat[House Building]{\includegraphics[width=0.25\linewidth]{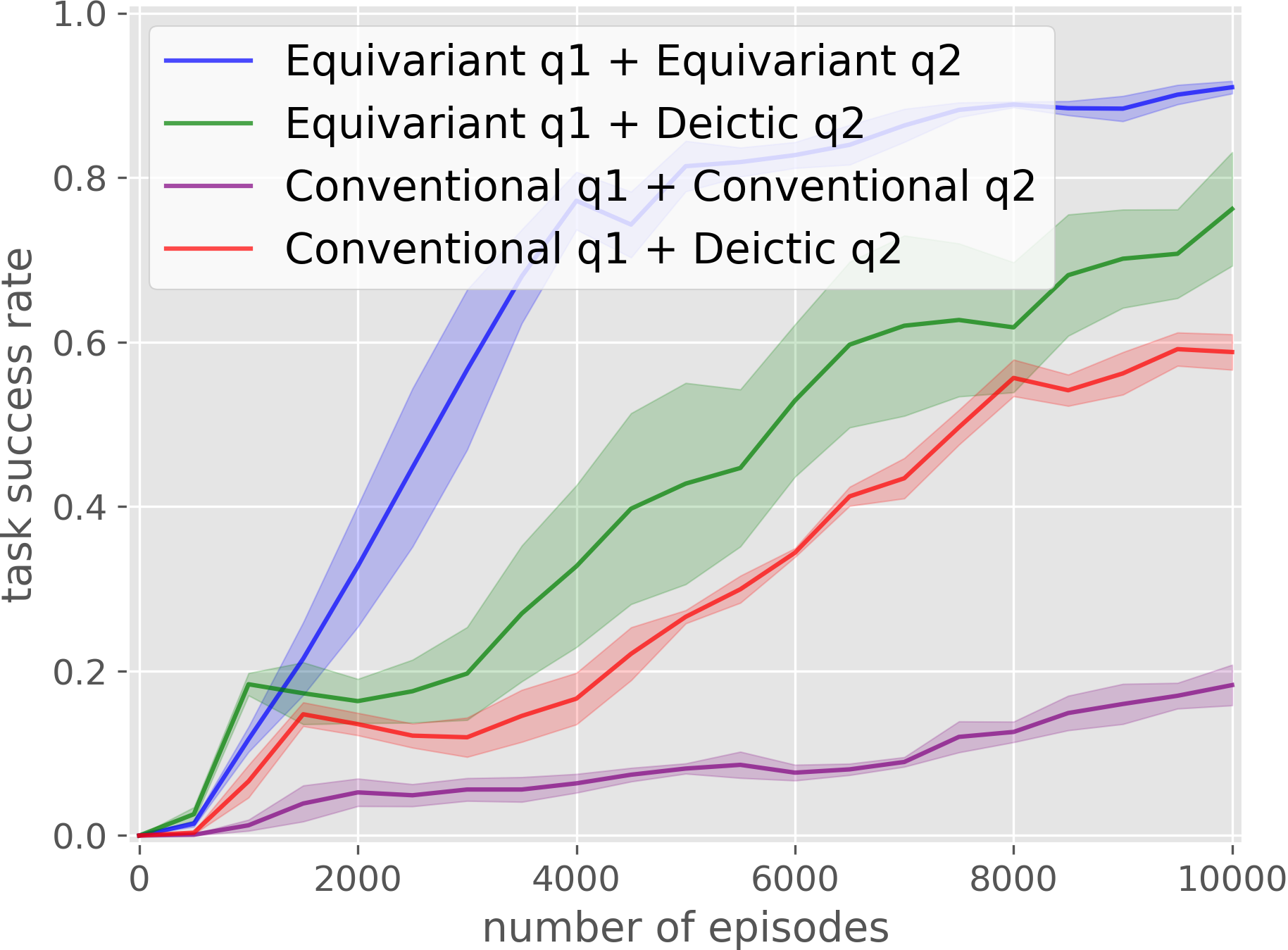}}
\subfloat[Covid Test]{\includegraphics[width=0.25\linewidth]{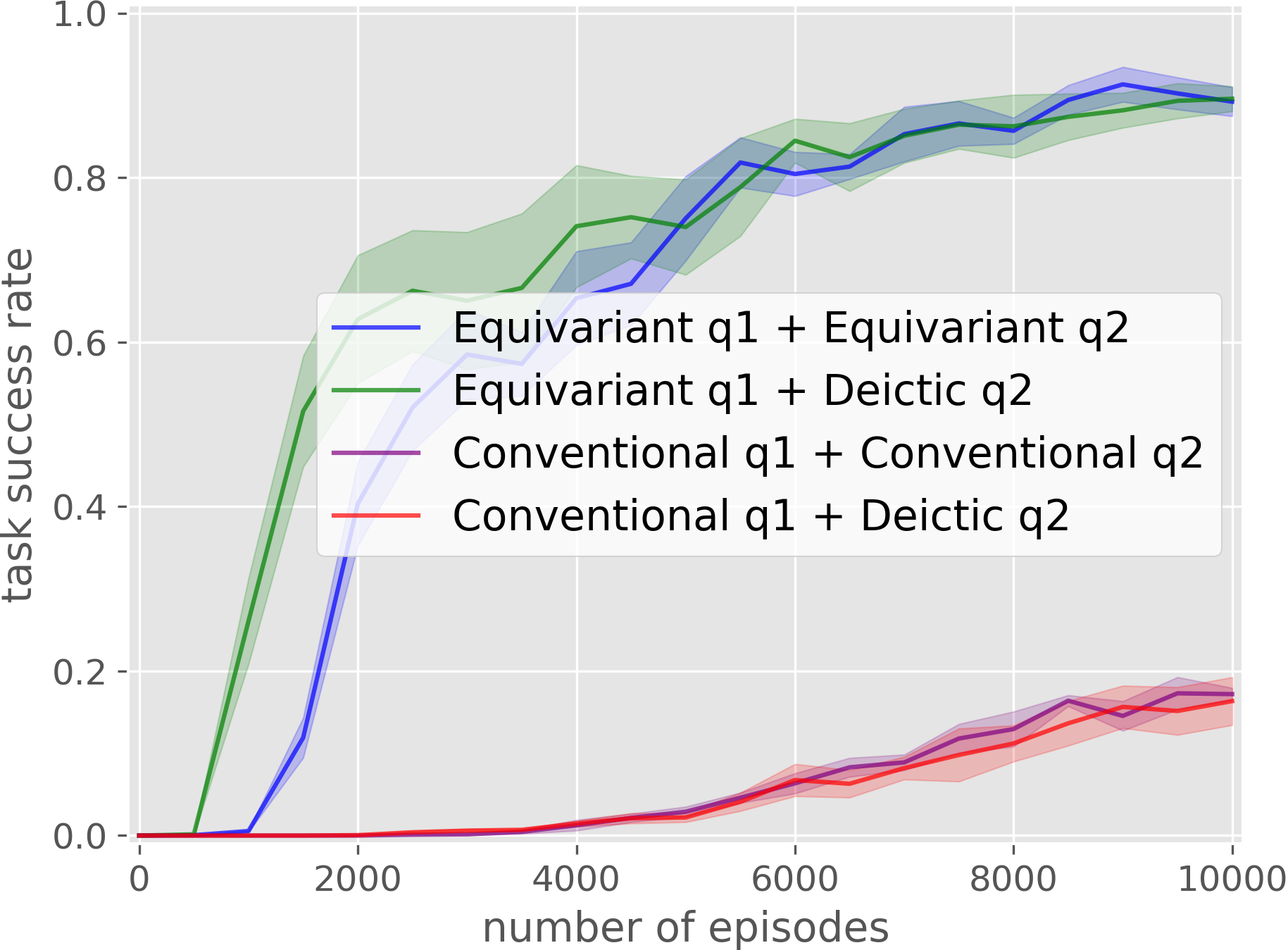}}
\subfloat[Box Palletizing]{\includegraphics[width=0.25\linewidth]{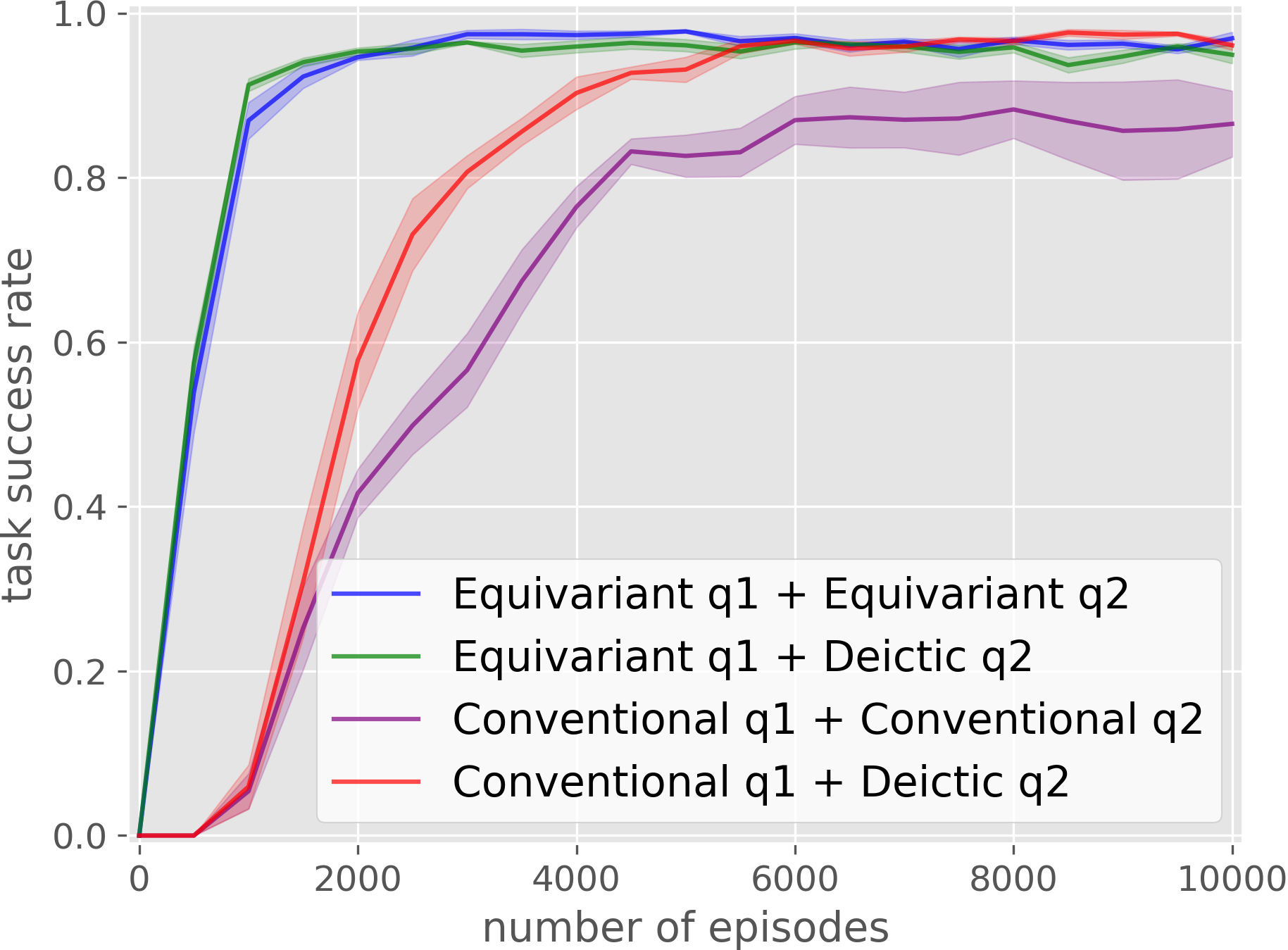}}
\subfloat[Bin Packing]{\includegraphics[width=0.25\linewidth]{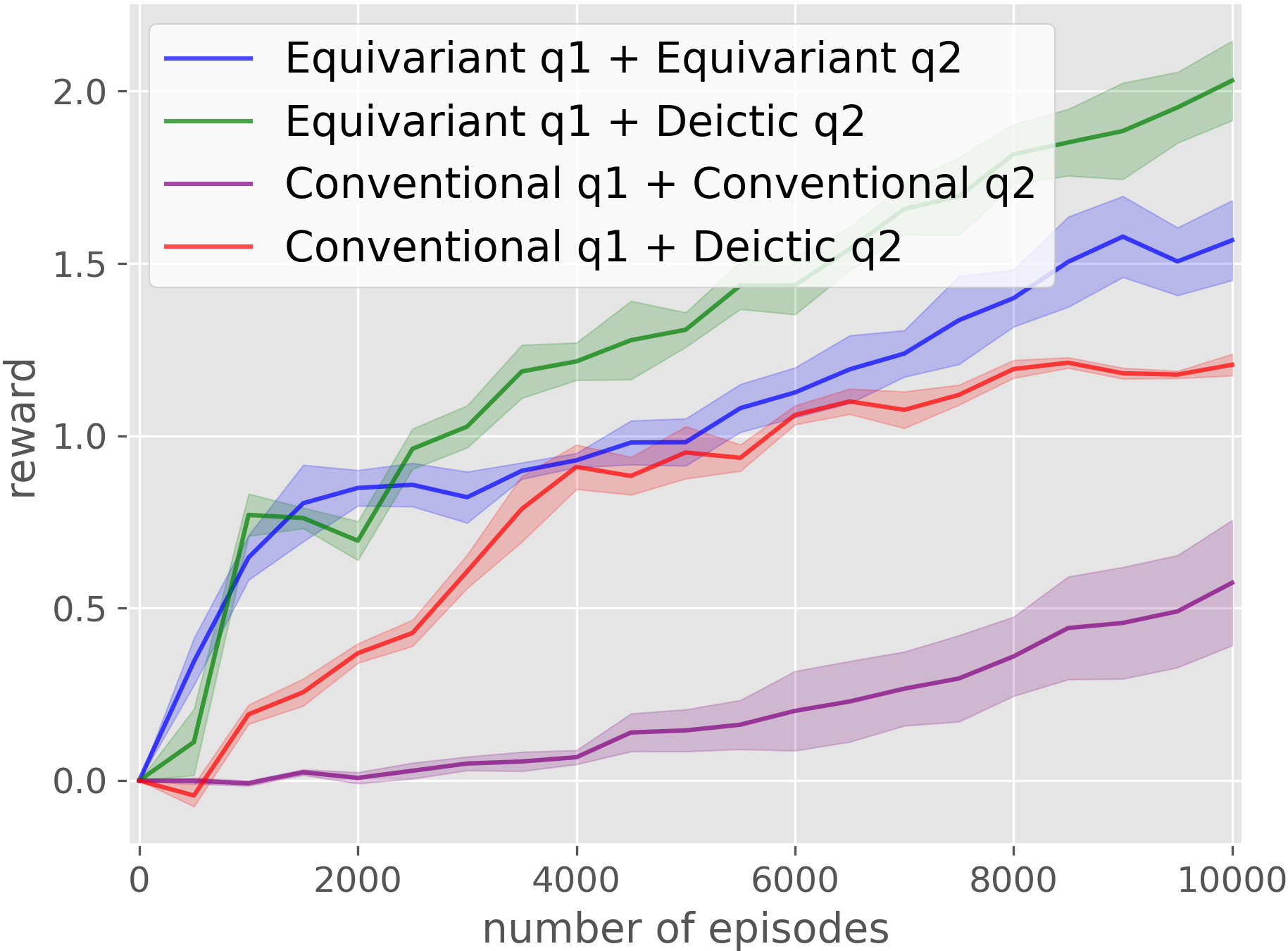}}
\caption{Comparison of the deictic encoding and baselines. Results averaged over four runs. Shading denotes standard error.}
\label{fig:exp_asr_deictic}
\end{figure}
This experiment compares the deictic encoding equipped with the equivariant ASR and the conventional ASR in $\SE(2)$. The comparison is conducted in the following four variations: 1) Equivariant $q_1$ + Equivariant $q_2$: both $q_1$ and $q_2$ use the equivariant network; 2) Equivariant $q_1$ + Deictic $q_2$: $q_1$ uses the equivariant network, $q_2$ uses the deictic encoding; 3) Conventional $q_1$ + Conventional $q_2$: both $q_1$ and $q_2$ use the conventional convolutional network; 4) Conventional $q_1$ + Deictic $q_2$: $q_1$ uses the equivariant network, $q_2$ uses the deictic encoding. The results are shown in Fig~\ref{fig:exp_asr_deictic}. When $q_1$ is using the equivariant network, using the deictic encoding in $q_2$ (green) outperforms using equivariant network in $q_2$ (blue) in Bin Packing, while the equivariant $q_2$ outperforms in House Building. In Covid Test and Box Palletizing, they tends to have similar performance. When $q_1$ uses conventional CNN, using deictic encoding in $q_2$ (red) generally provides a significant performance, compared with using conventional CNN in $q_2$ (purple). In Covid Test, the use of the deictic encoding does not make a big difference. We suspect that this is because in Covid Test the bottleneck of the whole system is $q_1$.

\subsubsection{Deictic Encoding in $\SE(3)$}
\label{appendix:exp_deictic_se3}
\begin{figure}[t]
\centering
\subfloat[Bumpy House Building]{\includegraphics[width=0.25\textwidth]{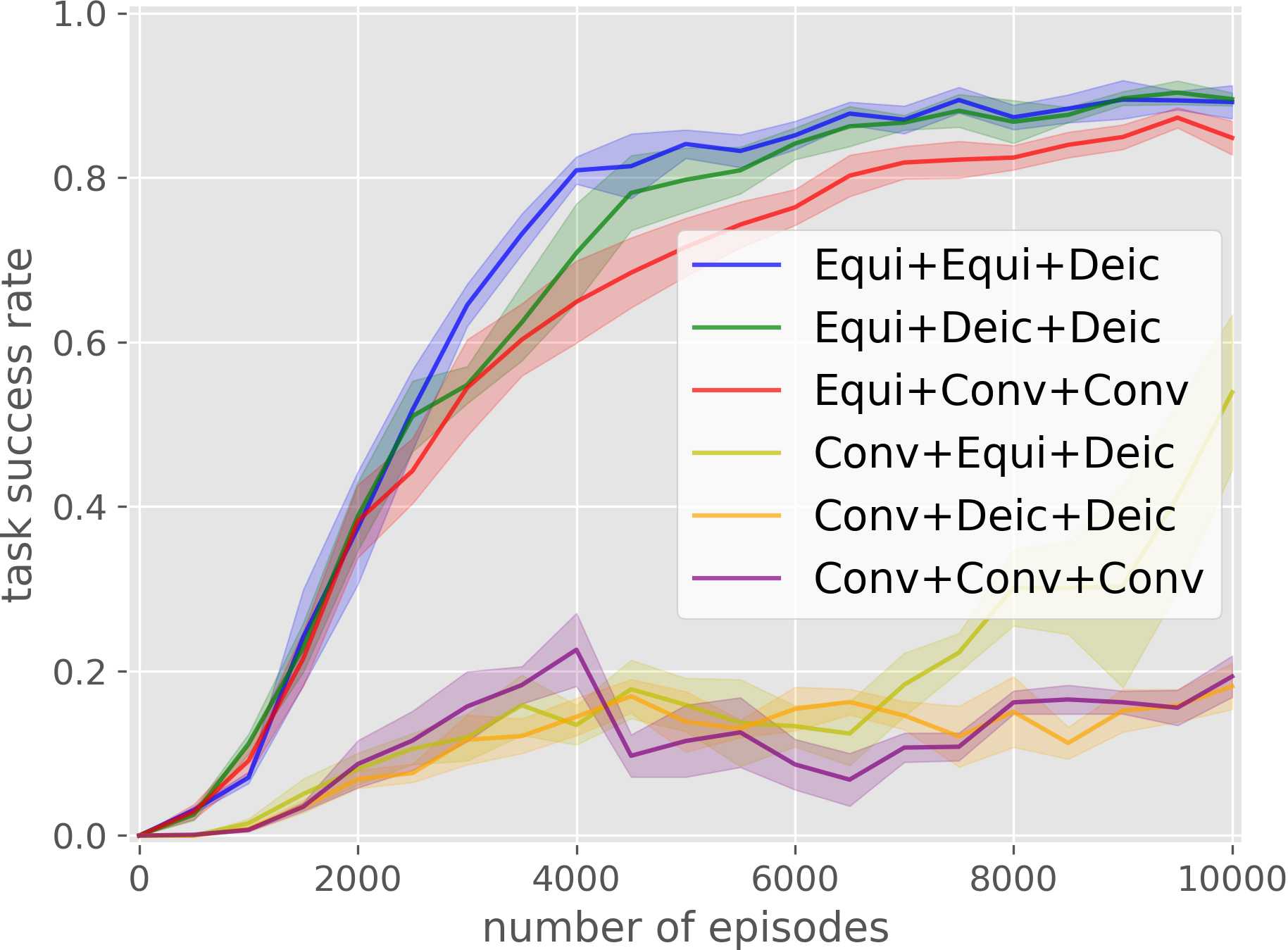}}
\subfloat[Bumpy Box Palletizing]{\includegraphics[width=0.25\textwidth]{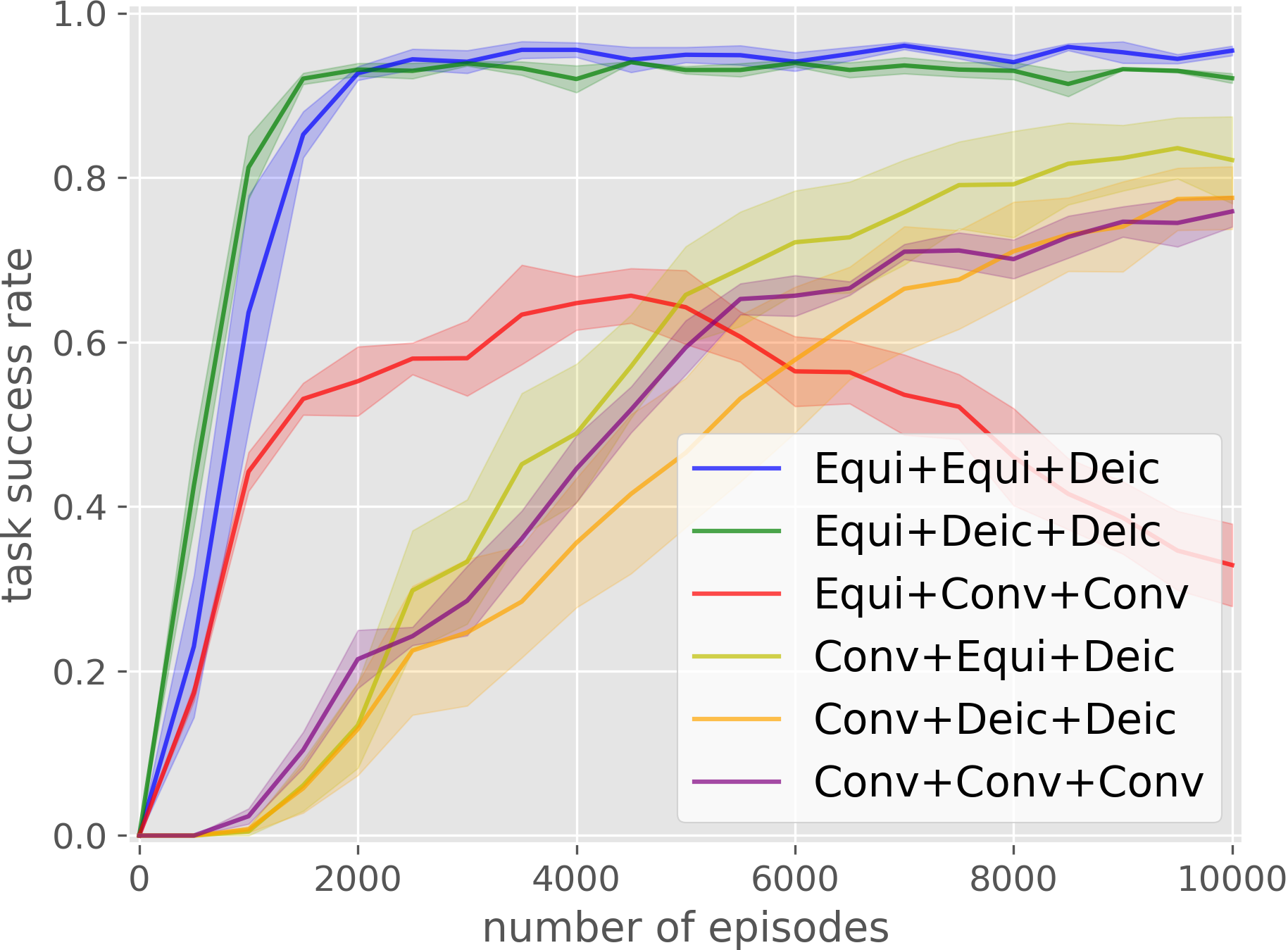}}
\caption{Comparison of different network choices in $SE(3)$. Results averaged over four runs. Shading denotes standard error.}
\label{fig:exp_6d_ablation}
\end{figure}
This experiment studies the different network choices for $q_1$ (equivariant network, conventional network), $q_2$ (equivariant network, deictic encoding, conventional network), and $q_3-q_5$ (deictic encoding, conventional network). We evaluate two proposed approaches: 1) Equi+Equi+Equi uses the equivariant network in $q_1$ and $q_2$ and deictic encoding in $q_3$ through $q_5$ (the three components in the name mean the architecture of $q_1$, $q_2$, and $q_3$-$q_5$); 2) Equi+Deic+Deic uses equivariant network in $q_1$, and deictic encoding in $q_2$ through $q_5$. We compare the proposal with the following baselines: 1) Equi+Conv+Conv uses equivariant network in $q_1$, and conventional convolutional network in $q_2$ through $q_5$; 2) Conv+Equi+Deic uses conventional convolutional network in $q_1$, equivariant network in $q_2$, and deictic encoding in $q_3$ through $q_5$; 3) Conv+Deic+Deic uses conventional convolutional network in $q_1$, and deictic encoding in $q_2$ through $q_5$; 4) Conv+Conv+Conv uses conventional convolutional network in $q_1$ through $q_5$. The results are shown in Fig~\ref{fig:exp_6d_ablation}, where our two proposed approaches outperform the baseline architectures in both environments. Note that swapping the conventional convolutional network with the equivariant network or the deictic encoding generally improves the performance, except that Equi+Conv+Conv in Bumpy Box Palletizing underperforms Conv+Conv+Conv. We suspect that this is because the target of $q_1$ given by the conventional convolutional networks is less stable.

\section{\edit{Runtime Analysis}}
\begin{table}[t]
\setlength\tabcolsep{2.5pt}%
\centering
\begin{tabular}{ccccccc}
\toprule
method & Conventional FCN & RAD FCN & DrQ FCN & Rot FCN & Equivariant FCN & Equivariant ASR \\
\midrule
time(s) & 0.08 & 0.09 & 0.22 & 0.42 & 0.72 & 0.45\\
\bottomrule
\end{tabular}
\caption{The average time for each training step in a rotation space of $C_{12}/C_2$}
\label{tab:runtime_r6}
\begin{tabular}{ccccc}
\toprule
method & Conventional ASR & RAD ASR & DrQ ASR & Equivariant ASR \\
\midrule
time(s) & 0.09 & 0.14 & 0.45 & 0.49\\
\bottomrule
\end{tabular}
\caption{The average time for each training step in a rotation space of $C_{32}/C_2$}
\label{tab:runtime_r16}
\end{table}

\edit{Table~\ref{tab:runtime_r6} and Table~\ref{tab:runtime_r16} shows the average runtime in the setting of the experiments in Section~\ref{sec:exp_equi_fcn} and Section~\ref{sec:exp_equi_asr}, respectively. The runtime is calculated by averaging over 500 training steps on a single Nvidia RTX 2080 Ti GPU. Both the Equivariant FCN and Equivariant ASR requires a longer time for each training step. However, Equivariant ASR is faster and similar to the best performing data augmentation method DrQ.}

\section{Robot Experiment}
\label{appendix:robot_exp}

\begin{figure}[t]
\centering
\includegraphics[width=0.2\linewidth]{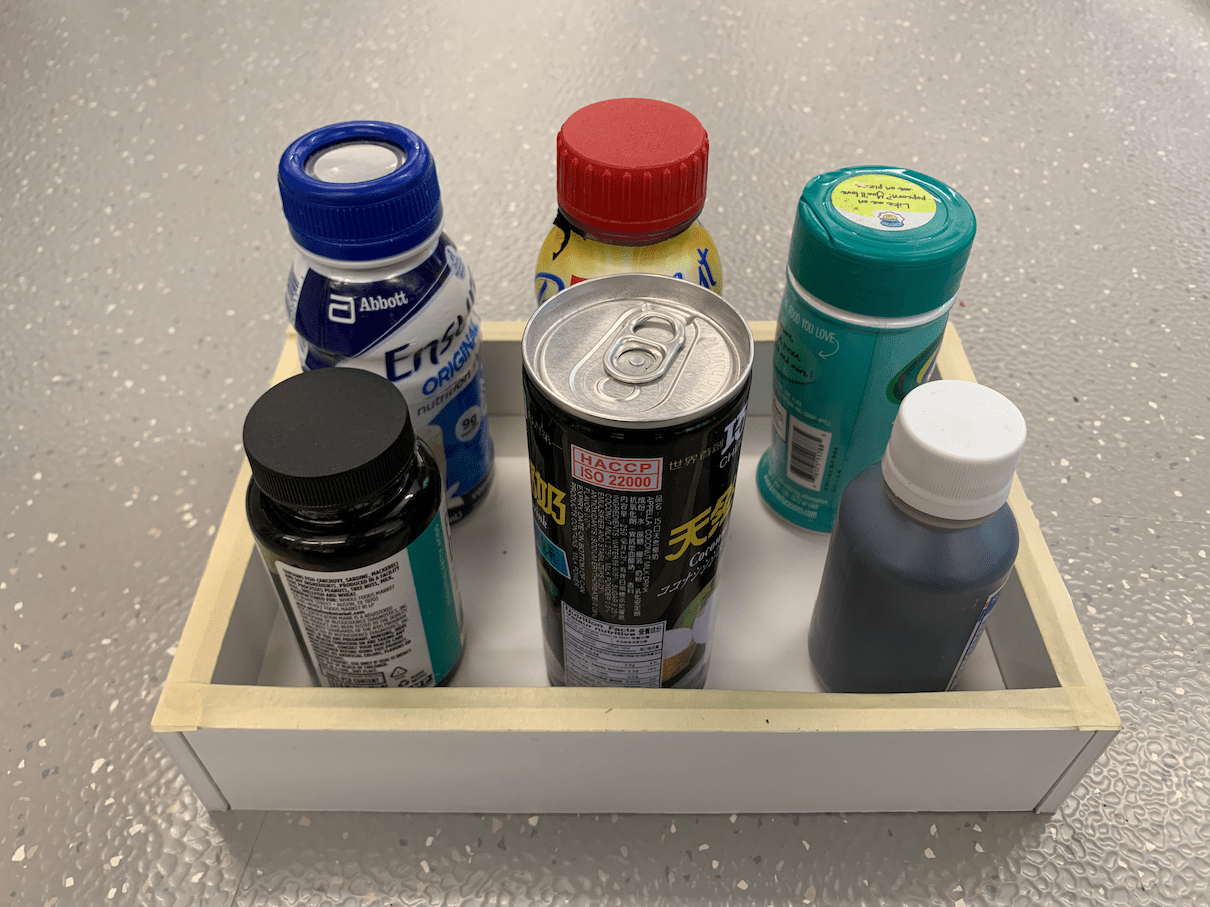}
\caption{The bottles used in the Bottle Arrangement robot experiment.}
\label{fig:robot_bottle}
\end{figure}

To ensure better sim to real transfer, we train our model used in the real world with a Perlin noise~\cite{perlin_noise} (with a maximum magnitude of 7mm) applied to the observations.

\underline{Bottle Arrangement:} In the Bottle Arrangement task, we use the bottles and tray shown in Fig~\ref{fig:robot_bottle} for testing.

\underline{House Building:} In the House Building task, we train the model with object size randomization within $\pm 8.3\%$. A Gaussian filter is applied after the Perlin noise during training to make the observation noisier. The model is trained for 20k episodes instead of 10k as in the simulation experiment.

\underline{Box Palletizing:} In the Box Palletizing task, we add an object size randomization within $\pm 3.75\%$ and increase the size of $H$ and $P$ from $24\times 24$ to $40\times 40$.

Fig~\ref{fig:robot_bottle_full} shows an example episode of the robot finishing the Bottle Arrangement task. Fig~\ref{fig:robot_box18_full} shows an example episode of the robot finishing the Box Palletizing task.

\begin{figure}[t]
\centering
\subfloat{
\includegraphics[width=0.15\textwidth]{image/exp/robot_bottle/1.png}
}
\subfloat{
\includegraphics[width=0.15\textwidth]{image/exp/robot_bottle/2.png}
}
\subfloat{
\includegraphics[width=0.15\textwidth]{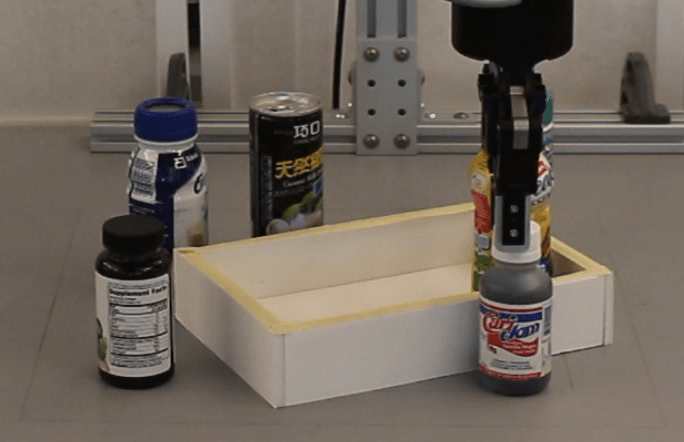}
}
\subfloat{
\includegraphics[width=0.15\textwidth]{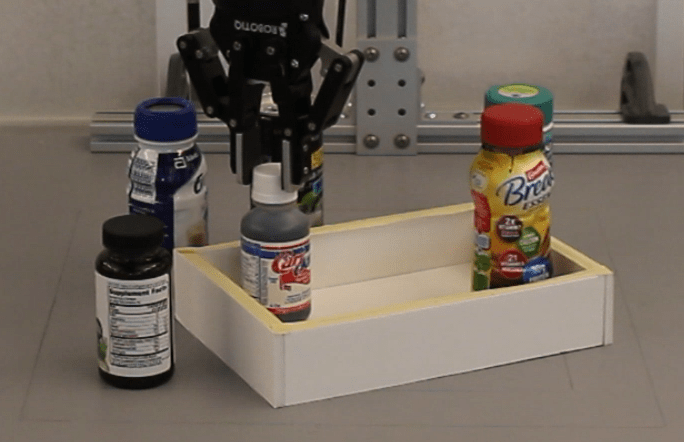}
}
\subfloat{
\includegraphics[width=0.15\textwidth]{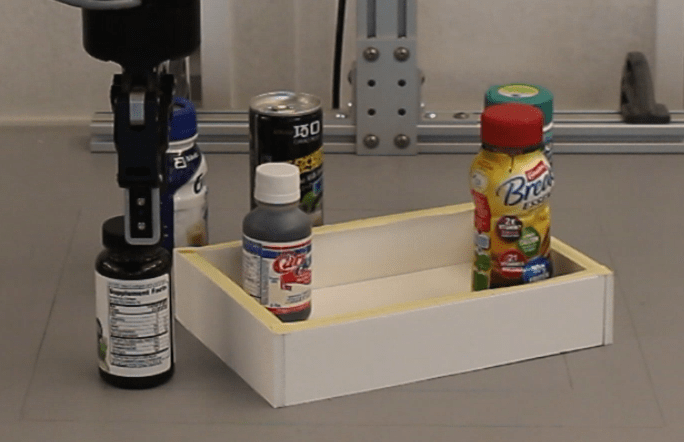}
}
\subfloat{
\includegraphics[width=0.15\textwidth]{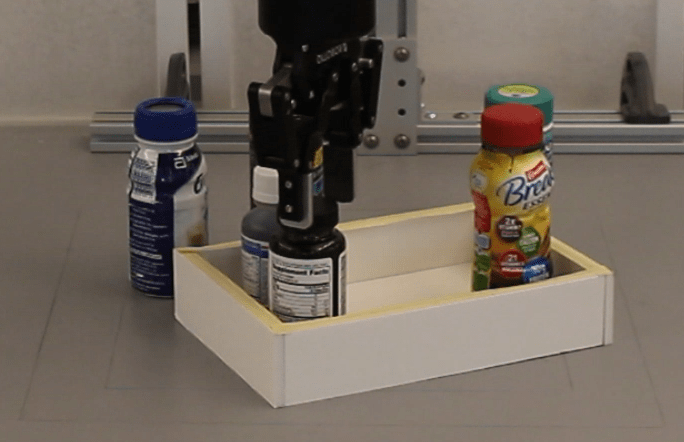}
}\\
\subfloat{
\includegraphics[width=0.15\textwidth]{image/exp/robot_bottle/7.png}
}
\subfloat{
\includegraphics[width=0.15\textwidth]{image/exp/robot_bottle/8.png}
}
\subfloat{
\includegraphics[width=0.15\textwidth]{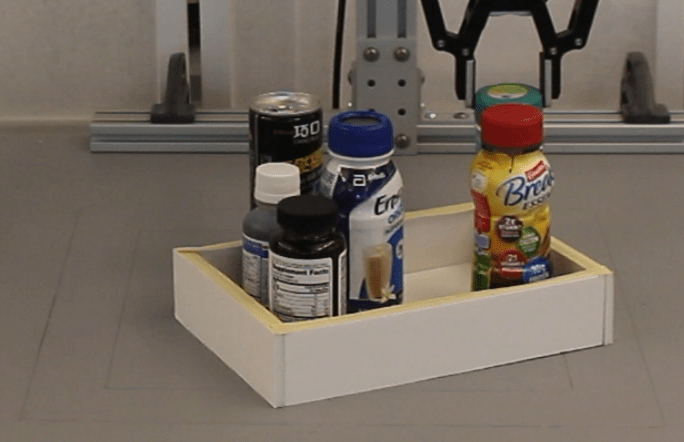}
}
\subfloat{
\includegraphics[width=0.15\textwidth]{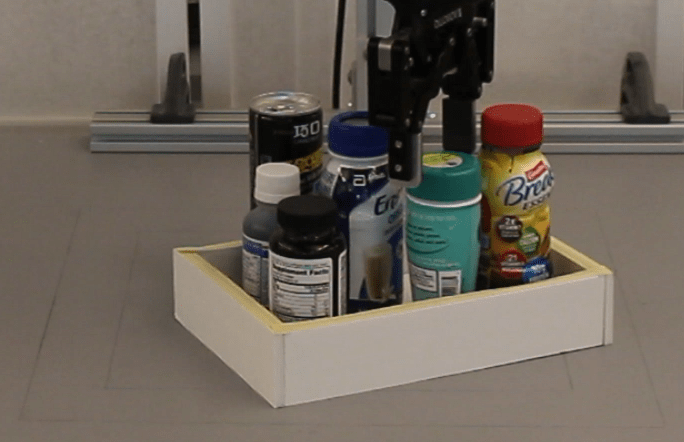}
}
\subfloat{
\includegraphics[width=0.15\textwidth]{image/exp/robot_bottle/11.png}
}
\subfloat{
\includegraphics[width=0.15\textwidth]{image/exp/robot_bottle/12.png}
}
\caption{An example episode of the Bottle Arrangement in the real world.}
\label{fig:robot_bottle_full}
\end{figure}

\begin{figure}[t]
\centering
\subfloat{
\includegraphics[width=0.15\textwidth]{image/exp/robot_box/01.png}
}
\subfloat{
\includegraphics[width=0.15\textwidth]{image/exp/robot_box/02.png}
}
\subfloat{
\includegraphics[width=0.15\textwidth]{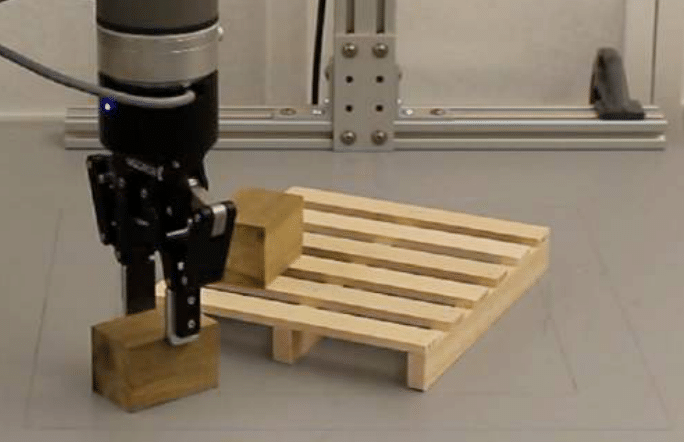}
}
\subfloat{
\includegraphics[width=0.15\textwidth]{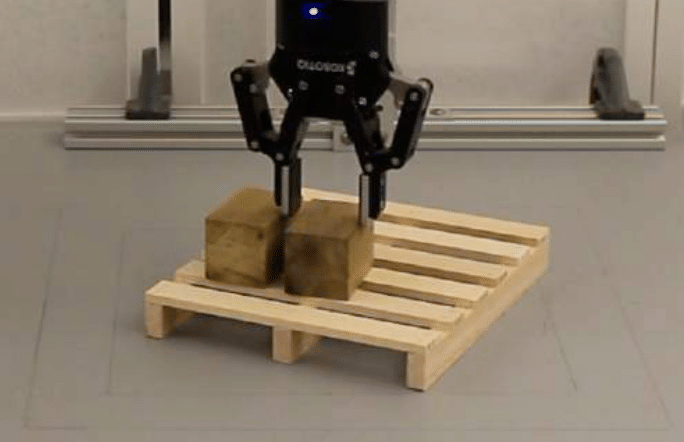}
}
\subfloat{
\includegraphics[width=0.15\textwidth]{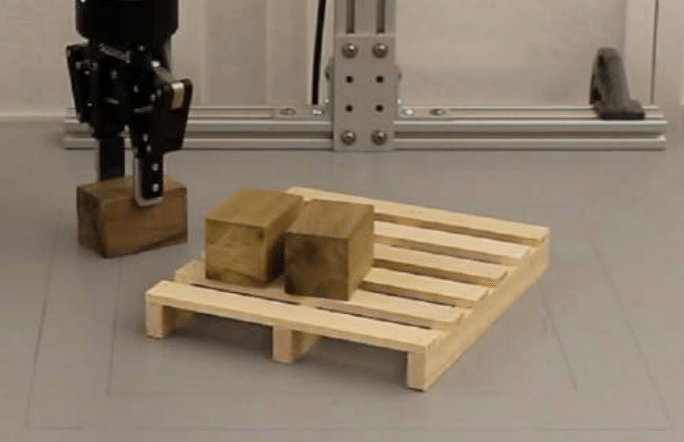}
}
\subfloat{
\includegraphics[width=0.15\textwidth]{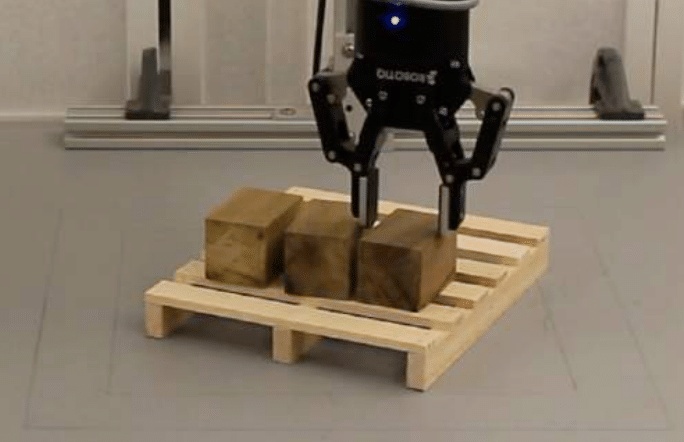}
}\\
\subfloat{
\includegraphics[width=0.15\textwidth]{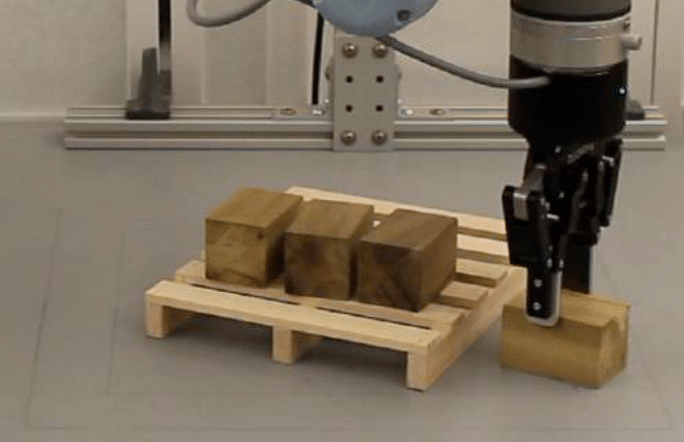}
}
\subfloat{
\includegraphics[width=0.15\textwidth]{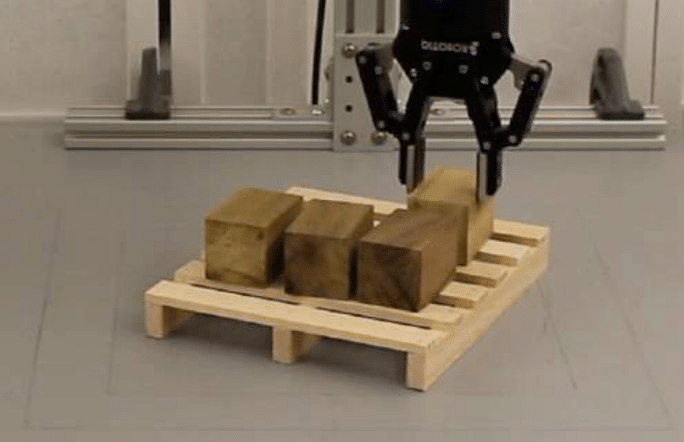}
}
\subfloat{
\includegraphics[width=0.15\textwidth]{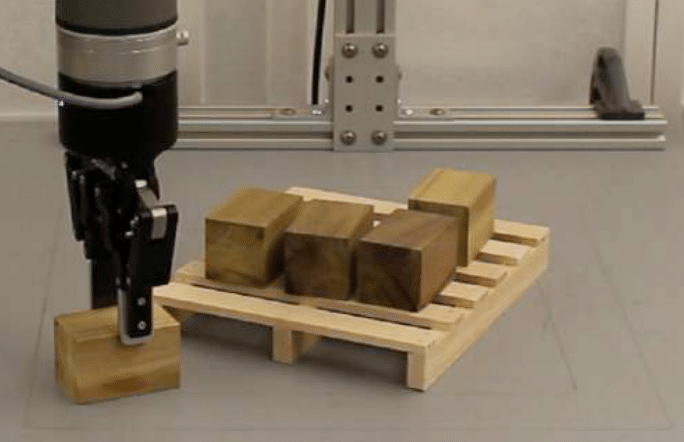}
}
\subfloat{
\includegraphics[width=0.15\textwidth]{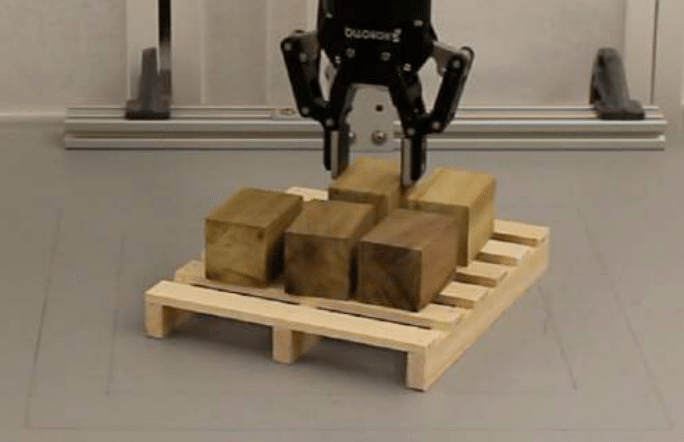}
}
\subfloat{
\includegraphics[width=0.15\textwidth]{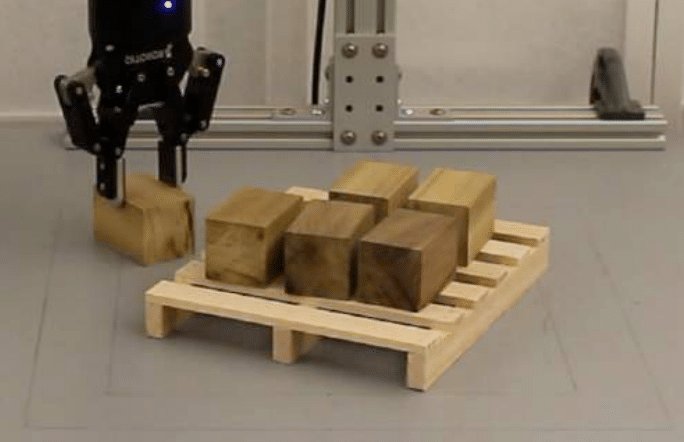}
}
\subfloat{
\includegraphics[width=0.15\textwidth]{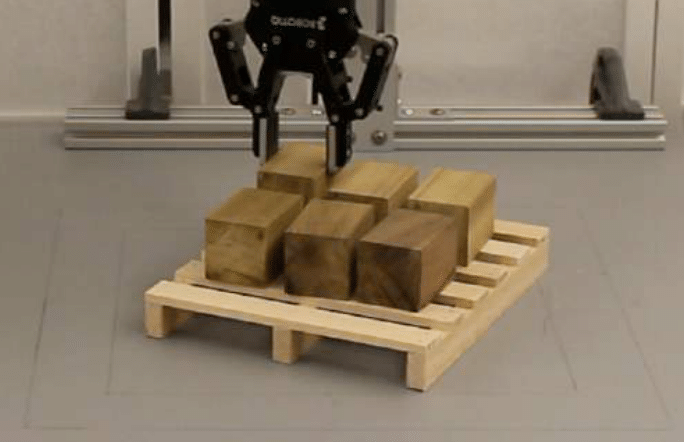}
}\\
\subfloat{
\includegraphics[width=0.15\textwidth]{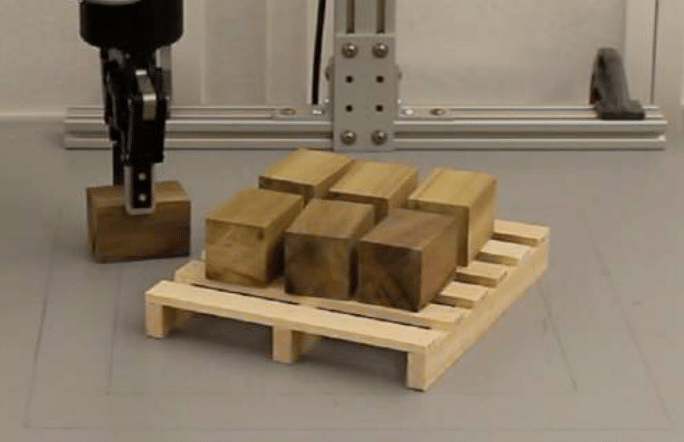}
}
\subfloat{
\includegraphics[width=0.15\textwidth]{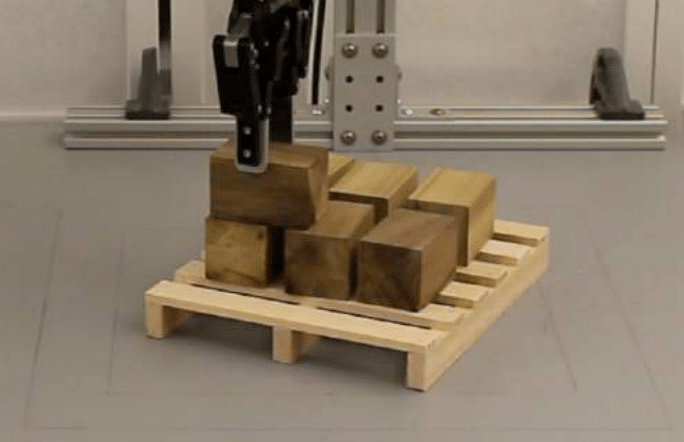}
}
\subfloat{
\includegraphics[width=0.15\textwidth]{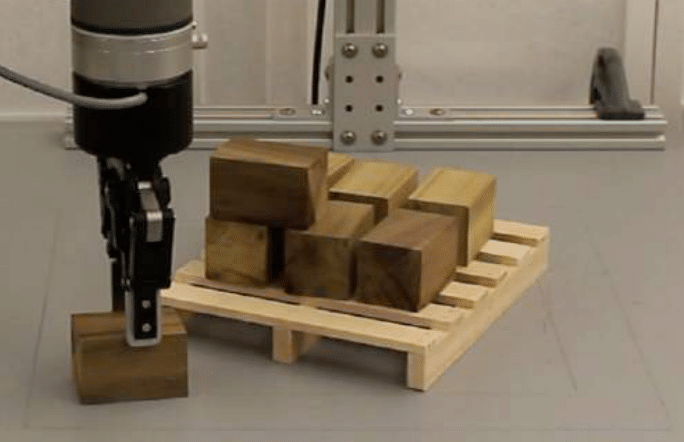}
}
\subfloat{
\includegraphics[width=0.15\textwidth]{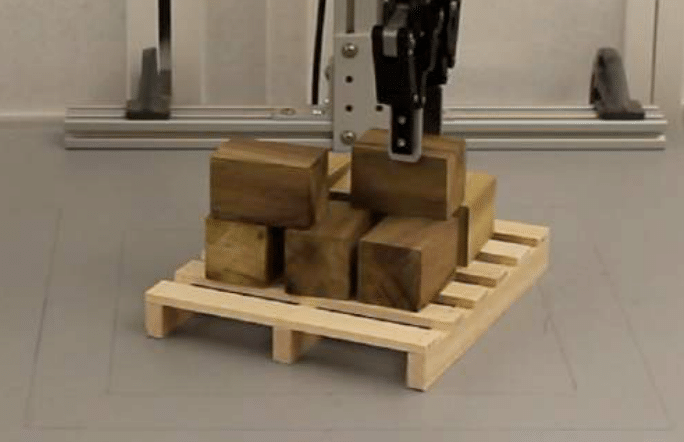}
}
\subfloat{
\includegraphics[width=0.15\textwidth]{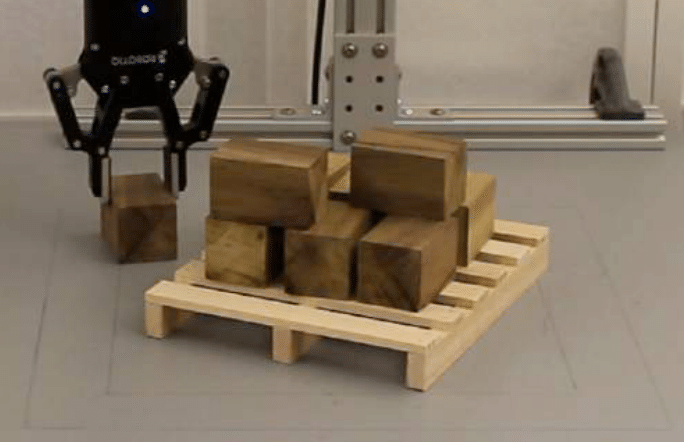}
}
\subfloat{
\includegraphics[width=0.15\textwidth]{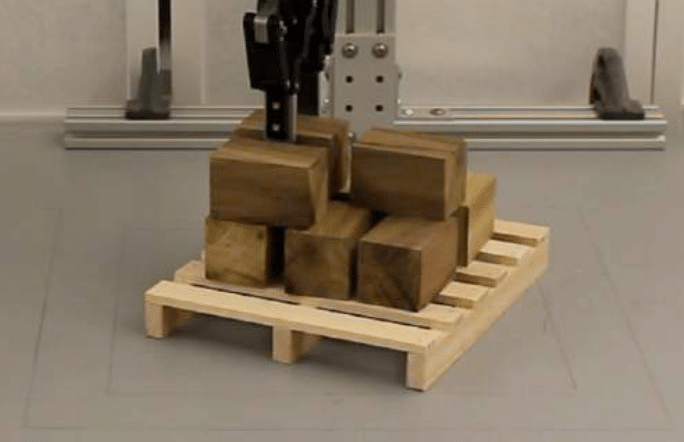}
}\\
\subfloat{
\includegraphics[width=0.15\textwidth]{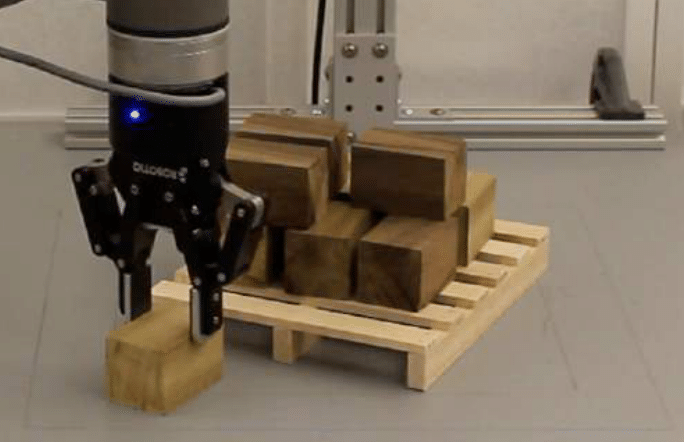}
}
\subfloat{
\includegraphics[width=0.15\textwidth]{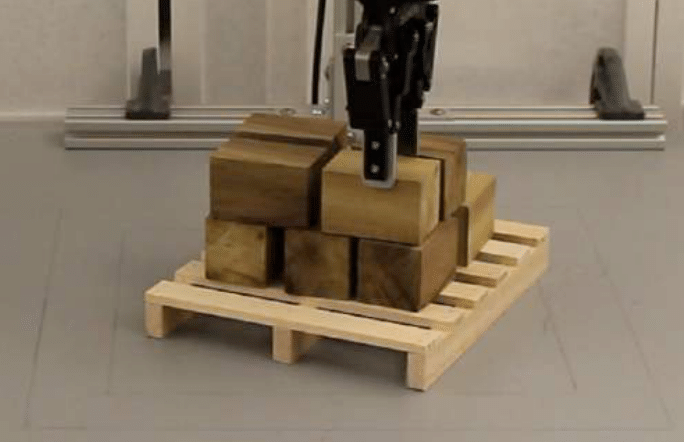}
}
\subfloat{
\includegraphics[width=0.15\textwidth]{image/exp/robot_box/21.png}
}
\subfloat{
\includegraphics[width=0.15\textwidth]{image/exp/robot_box/22.png}
}
\subfloat{
\includegraphics[width=0.15\textwidth]{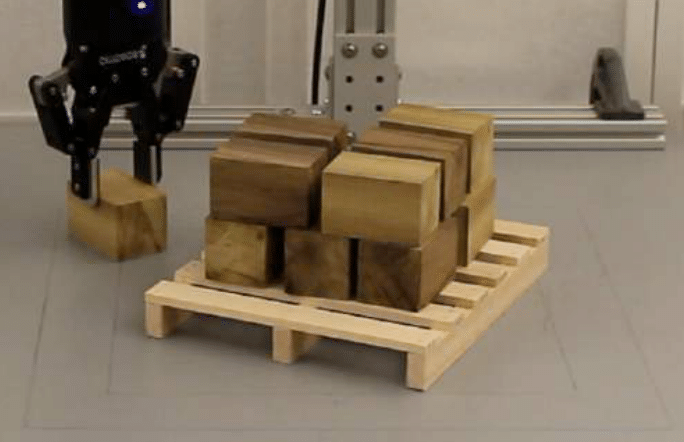}
}
\subfloat{
\includegraphics[width=0.15\textwidth]{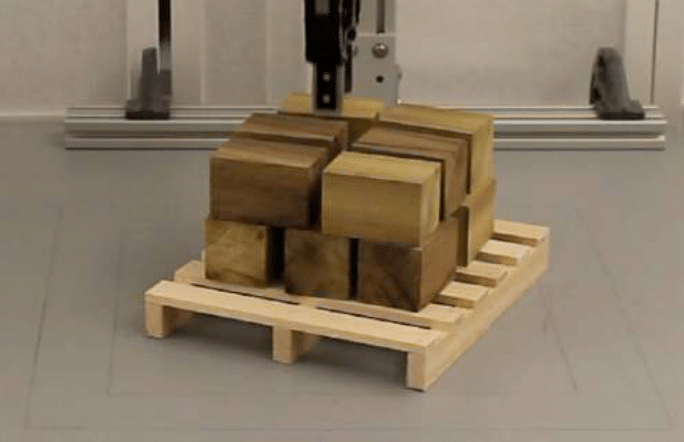}
}\\
\subfloat{
\includegraphics[width=0.15\textwidth]{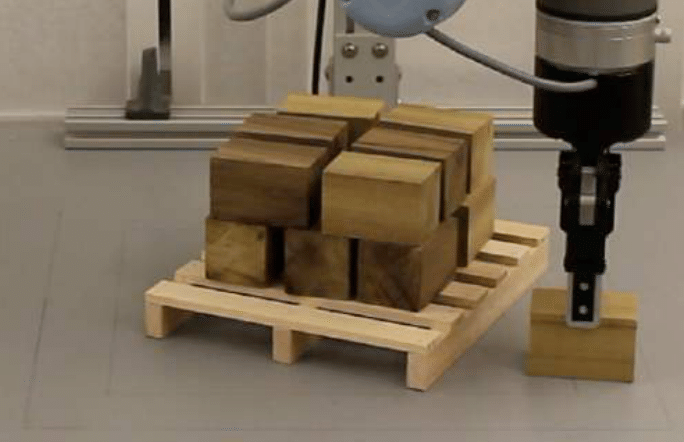}
}
\subfloat{
\includegraphics[width=0.15\textwidth]{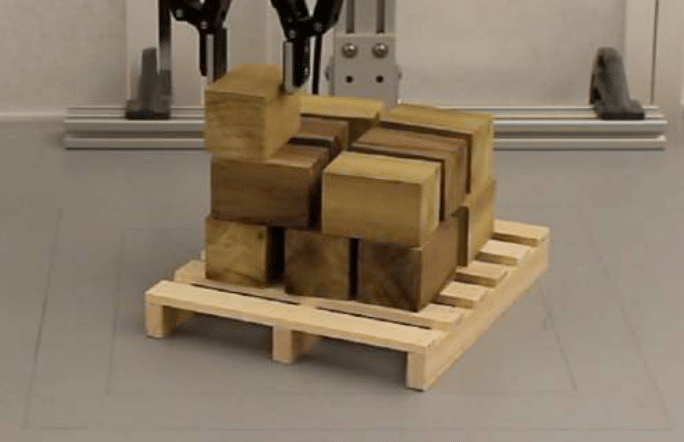}
}
\subfloat{
\includegraphics[width=0.15\textwidth]{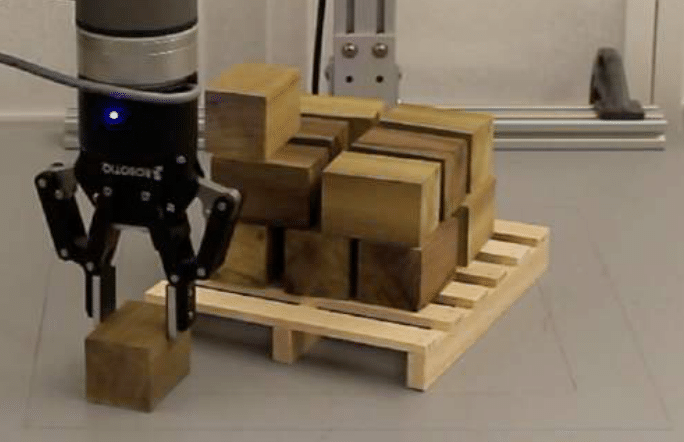}
}
\subfloat{
\includegraphics[width=0.15\textwidth]{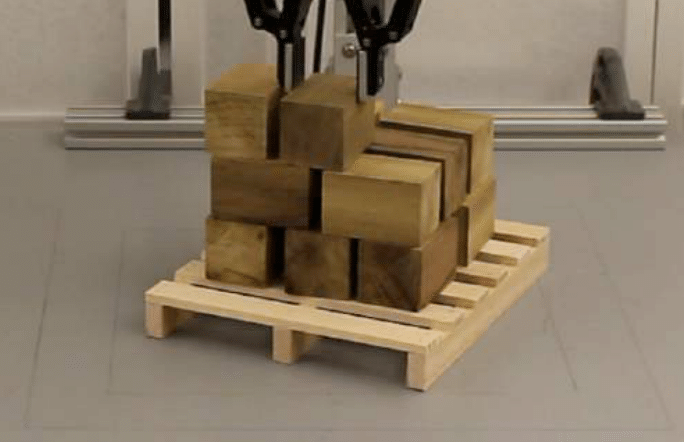}
}
\subfloat{
\includegraphics[width=0.15\textwidth]{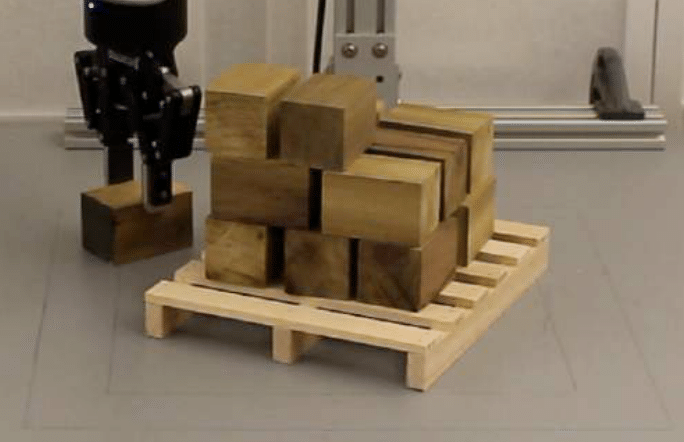}
}
\subfloat{
\includegraphics[width=0.15\textwidth]{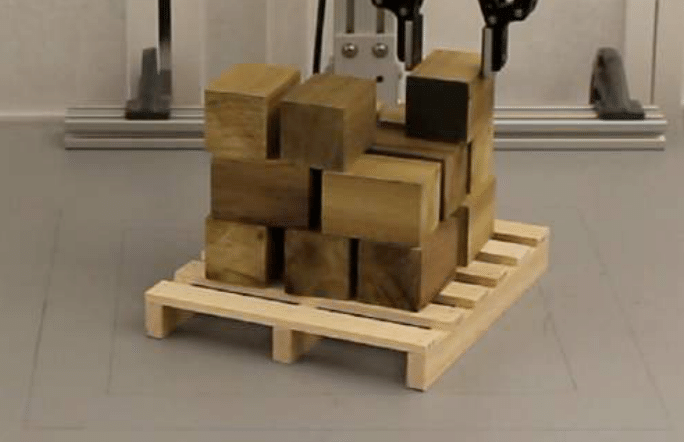}
}\\
\subfloat{
\includegraphics[width=0.15\textwidth]{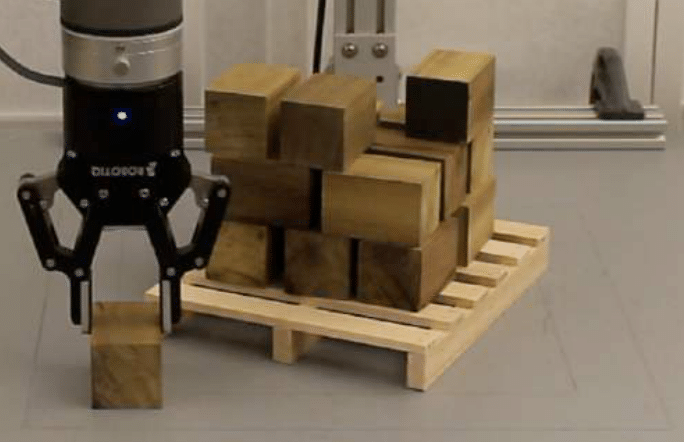}
}
\subfloat{
\includegraphics[width=0.15\textwidth]{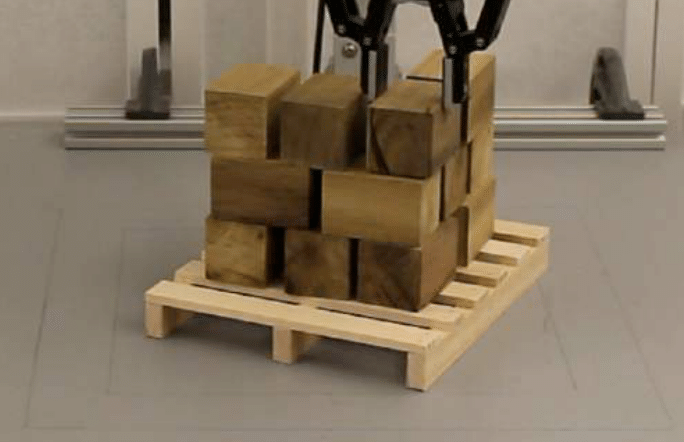}
}
\subfloat{
\includegraphics[width=0.15\textwidth]{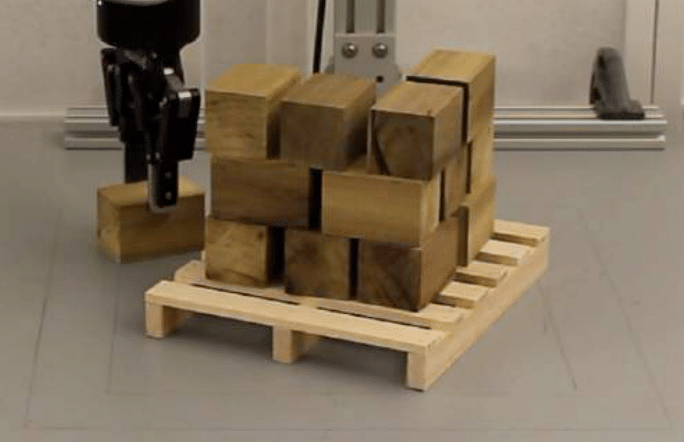}
}
\subfloat{
\includegraphics[width=0.15\textwidth]{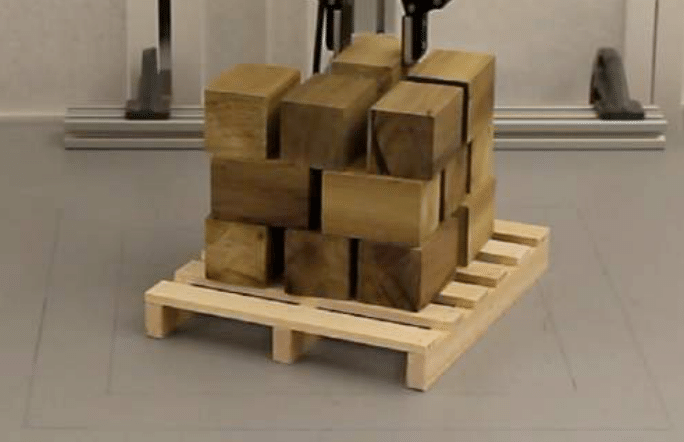}
}
\subfloat{
\includegraphics[width=0.15\textwidth]{image/exp/robot_box/35.png}
}
\subfloat{
\includegraphics[width=0.15\textwidth]{image/exp/robot_box/36.png}
}
\caption{An example episode of the Box Palletizing in the real world.}
\label{fig:robot_box18_full}
\end{figure}

\end{document}